\documentclass[final]{amsart}
\usepackage{Aaron_Pim_Style}
\usepackage{booktabs, fullpage, algorithm, algpseudocode, pgfplots, pgfplotstable, subcaption}
\usepackage{tikz}
\usetikzlibrary{positioning, arrows.meta, calc}
\usepackage{caption}
\usepackage[giveninits=true,eprint=false,url=false,style=alphabetic]{biblatex}
\usepackage[dvipsnames]{xcolor}
\usepackage[colorlinks=true, linkcolor=Purple, citecolor=Orange, urlcolor=Purple]{hyperref}

\title{Multi-level Monte Carlo Dropout for Efficient Uncertainty Quantification}

\author[2]{Aaron Pim}

\author[1,2]{Tristan Pryer}

\address{$^1$ Institute for Mathematical Innovation\\ University of
  Bath, Bath, UK. $^2$ Department of Mathematical Sciences
  \\ University of Bath, Bath, UK.}

\pgfplotsset{compat=1.18}
\begin{document}

\begin{abstract}
  We develop a multilevel Monte Carlo (MLMC) framework for uncertainty
  quantification with Monte Carlo dropout. Treating dropout masks as a
  source of epistemic randomness, we define a fidelity hierarchy by
  the number of stochastic forward passes used to estimate predictive
  moments. We construct coupled coarse--fine estimators by reusing
  dropout masks across fidelities, yielding telescoping MLMC
  estimators for both predictive means and predictive variances that
  remain unbiased for the corresponding dropout-induced quantities
  while reducing sampling variance at fixed evaluation budget. We
  derive explicit bias, variance and effective cost expressions,
  together with sample-allocation rules across levels. Numerical experiments on forward and
  inverse PINNs--Uzawa benchmarks confirm the predicted variance rates
  and demonstrate efficiency gains over single-level MC-dropout at
  matched cost.
\end{abstract}

\maketitle

\section{Introduction}


In many scientific computing and machine learning applications, neural
networks are increasingly used as fast surrogate models for complex
simulations \cite{raissi2019physics}. Unlike traditional mechanistic
models, these neural surrogates do not by default quantify prediction
uncertainty. Their outputs are deterministic. In practice this is
problematic: deep networks often interpolate training data well but
can make overconfident, erroneous predictions when extrapolating
\cite{gawlikowski2023survey,hullermeier2021aleatoric}. In applied
surrogate pipelines this lack of uncertainty can be a practical barrier
to reliable deployment, including settings where the surrogate is used
for repeated evaluation and decision making \cite{pim2025surrogate}.

Consequently, there is a pressing need to accompany neural surrogates
with uncertainty quantification (UQ). Bayesian deep learning methods
address this by modelling a distribution over network outputs. For
example, Monte Carlo dropout (MC-dropout) interprets dropout at test
time as approximate Bayesian inference, yielding epistemic uncertainty
estimates without changing network architecture \cite{gal2016dropout}.
Similarly, methods like deep ensembles train multiple networks and use
the ensemble variance for UQ \cite{lakshminarayanan2017simple}. In
scientific machine learning, Physics-informed Neural Networks (PINNs)
incorporate prior knowledge by training on physics constraints, but
they too generally lack calibrated uncertainty outputs unless extended
to Bayesian formulations. For inverse problems and constrained training
setups this can be especially acute, since inference is driven by a
combination of data mismatch and modelling assumptions
\cite{cox2024bayesian,makridakis2024deep}. In such settings one may
face both aleatoric uncertainty (inherent noise in data) and epistemic
uncertainty (model uncertainty) \cite{kendall2017uncertainties}. In
this work we focus on reducing the cost of estimating epistemic
uncertainty induced by MC-dropout.

A standard approach (as in MC-dropout or ensemble sampling) is to run
many stochastic forward passes to estimate predictive statistics.
However, the computational cost grows linearly with the number of
samples. Reducing estimator variance by increasing the number of
samples can be prohibitive in large-scale settings. For example,
MC-dropout often uses dozens of passes per input to estimate
uncertainty, each pass incurring a full forward evaluation of the
network. This renders uncertainty estimation expensive for large or
real-time models and it is already a practical bottleneck in surrogate
pipelines that require repeated uncertainty evaluation
\cite{pim2025surrogate}. More generally, naive Monte Carlo estimators
are often too costly when each sample is expensive
\cite{giles2015multilevel}.

Multilevel Monte Carlo (MLMC) addresses precisely this issue by
exploiting a hierarchy of fidelities and a telescoping decomposition
\cite{heinrich2001multilevel,giles2008multilevel}. Most samples are
taken at low-fidelity (cheap) levels and few at high-fidelity
(expensive) levels, reducing total cost for a given error
\cite{giles2008multilevel,giles2015multilevel}. In stochastic
simulation (e.g.\ SDEs and PDEs) MLMC has proven very effective
\cite{cliffe2011multilevel}. Recent work has begun to bring MLMC ideas
into machine learning; for instance, \cite{fujisawa2021multilevel}
applies MLMC to variational inference in Bayesian models and
\cite{gerstner2021multilevel} proposes multilevel training strategies
for deep networks. More closely related to our setting,
\cite{blanchet2023dropout} develops an unbiased MLMC scheme to speed
up dropout training of classifiers. However, to our knowledge MLMC has
not been explored for inference-time uncertainty quantification in
neural surrogates.


The goal of this work is to bridge practical dropout-based uncertainty
estimation with variance-reduction tools from multilevel Monte Carlo
(MLMC). We formulate an MLMC framework in which Monte Carlo dropout
provides the randomness, a single forward evaluation with dropout is
treated as a randomised model evaluation and predictive statistics are
estimated by sampling over dropout masks. Our primary notion of
fidelity is the number $T$ of stochastic forward passes used inside the
predictive estimators. Higher fidelity corresponds to more dropout
masks and hence higher computational cost. Building on this hierarchy,
we construct multilevel estimators by coupling coarse and fine
evaluations through shared dropout masks, so that level corrections
have reduced variance. This yields multilevel estimators for both
dropout-induced predictive expectations and dropout-induced predictive
variances, while retaining unbiasedness for these dropout-induced
targets. Other notions of fidelity, such as varying dropout probability
or using a hierarchy of network resolutions, can also be viewed through
a multilevel lens, but they require additional modelling choices and
are not essential for the core development presented here.

We provide a theoretical analysis of the resulting estimators. We
derive explicit bias, variance and cost expressions, identify the
quantities that control the variance of level corrections and use these
to formulate sample allocation strategies across levels under a coupled
cost model with mask reuse. The resulting allocations recover the
classical MLMC structure $M_\ell\propto \sqrt{V_\ell/C_\ell}$, with
problem-specific expressions for the level variances $V_\ell$ induced by
the dropout coupling. We also discuss practical ladder design and
stopping criteria based on empirical level-variance estimates.

We demonstrate the approach on benchmark problems drawn from scientific
machine learning. First, we consider a forward PINNs--Uzawa problem for
a boundary-layer ODE with a known closed-form solution and study the
behaviour of single-fidelity and multilevel dropout estimators,
including empirical verification of the predicted sampling-variance
rates and a fixed-cost allocation test. Second, we consider an inverse
PDE-constrained optimisation benchmark with a stochastic target
generated by a scalar noise variable, for which the ground-truth target
moments are available in closed form. In both settings we report how
the estimator variances depend on the fidelity ladder and on the sample
allocation, and we compare empirical best allocations with the
continuous allocations predicted by the theory.


Research on uncertainty for neural surrogates spans several strands
that intersect in this work. Dropout, originally introduced for
regularisation \cite{srivastava2014dropout}, can be kept active at
test time and interpreted as approximate Bayesian inference, providing
a practical route to epistemic uncertainty without changing
architecture or training \cite{gal2016dropout}. A related theme is the
distinction between aleatoric uncertainty (stemming from data noise)
and epistemic uncertainty (stemming from model uncertainty), which is
now standard in uncertainty-aware deep learning
\cite{kendall2017uncertainties}. Complementary strategies include
Bayesian neural networks and deep ensembles, where independent
predictors provide variance-based estimates that are often strong in
practice \cite{lakshminarayanan2017simple}. In scientific surrogate
pipelines, MC-dropout has also been used directly as a practical UQ
tool when repeated uncertainty evaluation is required
\cite{pim2025surrogate}.

Multilevel Monte Carlo provides a cost--accuracy tradeoff that is well
suited to expensive stochastic evaluation. Foundational results
establish telescoping estimators over coupled fidelity levels and show
that, in canonical discretisation settings, complexity improvements
such as $O(\varepsilon^{-3})$ to $O(\varepsilon^{-2})$ can be achieved
under standard assumptions on bias, variance decay and per-sample cost
\cite{heinrich2001multilevel,giles2008multilevel,giles2015multilevel}.
The methodology has been widely validated in stochastic simulation for
SDEs and PDEs, including elliptic problems with random coefficients
\cite{cliffe2011multilevel}. In this paper we exploit the same
principles at inference time for neural predictors by introducing a
hierarchy of estimators for MC-dropout and coupling them across
fidelities so that level differences have low variance.

Recent work begins to explore multilevel ideas within machine learning
itself. Examples include multilevel variance reduction for stochastic
variational inference \cite{fujisawa2021multilevel}, hierarchical
training strategies for networks \cite{gerstner2021multilevel} and
unbiased multilevel constructions that accelerate dropout training
\cite{blanchet2023dropout}. There is also growing interest in multilevel
and multi-fidelity approaches for simulation-based inference when
multiple simulators are available, demonstrating improved performance
at fixed computational budgets with appropriate couplings
\cite{hikida2025multilevel}. The present contribution is distinct in
its focus on inference-time uncertainty quantification for neural
surrogates, where multilevel sampling targets the cost of estimating
MC-dropout moments.

Within scientific machine learning, physics-informed neural networks
provide a flexible framework for solving forward and inverse problems
but generally require dedicated mechanisms for calibrated uncertainty
\cite{raissi2019physics}. Bayesian variants and multi-fidelity
uncertainty-aware surrogates have been developed to address this
\cite{yang2021b,meng2021multi}. In our setting, PINN surrogates benefit
from treating MC-dropout estimators as levelled evaluations, allowing
dropout-induced epistemic uncertainty to be computed with
substantially fewer expensive forward passes.

Overall, this paper synthesises MC-dropout for epistemic uncertainty
\cite{gal2016dropout}, uncertainty-aware surrogate modelling in
scientific applications \cite{raissi2019physics,meng2021multi} and
multilevel Monte Carlo for computational efficiency
\cite{giles2015multilevel,fujisawa2021multilevel}. The result is a
framework that targets the computational cost of uncertainty
estimation and demonstrates that multilevel sampling can deliver
accurate dropout-based UQ for neural surrogates at reduced budgets.

The rest of the paper is set out as follows. In
\S\ref{sec:mlmc_dropout_estimators} we define the dropout-induced
targets of interest and introduce multilevel Monte Carlo estimators
for the predictive mean and variance based on coupled inner
fidelities. In \S\ref{sec:mlmc_dropout_computation} we give a
practical coupled sampling algorithm, define the coupled cost model
under mask reuse and derive optimal sample allocations across levels
for a fixed computational budget. In \S\ref{sec:numerical_experiments}
we validate the theoretical predictions on forward and inverse
PINNs--Uzawa benchmarks, including empirical variance-rate studies and
fixed-cost allocation tests. In \S\ref{sec:code_discussion} we
summarise implementation details relevant to reproducibility and
computational complexity. We conclude in \S\ref{sec:conclusion}. An
appendix collects supporting proofs of the technical lemmata used in the analysis,
including covariance identities for overlapping sample-variance
estimators and proofs of the allocation formulae.

\section{Multilevel estimators for MC-dropout moments}\label{sec:mlmc_dropout_estimators}
In this section we formulate MC-dropout inference as the numerical
estimation of moments of a randomised network output. This viewpoint
makes the computational bottleneck explicit, uncertainty estimates
require many stochastic forward passes. We then develop multilevel
Monte Carlo estimators that reduce this cost while targeting the same
dropout-induced quantities.

\subsection{Dropout-induced random predictor}\label{subsec:mlmc:random_predictor}
We begin by fixing a deterministic surrogate model represented by a
trained feedforward neural network. Concretely, for an input
$x \in \mathbb{R}^{N_{\mathrm{in}}}$ the network produces an output in
$\mathbb{R}^{N_{\mathrm{out}}}$ via a composition of affine maps and a
nonlinearity. For example, writing $z_0 := x$ and, for $k=0,\dots,K-1$,
\begin{equation}\label{eq:mlmc:nn_forward}
    z_{k+1} := \sigma\!\big(W_k z_k + b_k\big),
    \qquad
    \mathscr{D}(x;\hat\theta) := W_K z_K + b_K,
\end{equation}
where $\sigma$ is an appropriate activator, $\hat\theta :=
\{(W_k,b_k)\}_{k=0}^{K} \in \Theta$ denotes the trained weights and biases, held
fixed throughout this section. A schematic of this deterministic
evaluation is shown in Figure~\ref{fig:mlmc:deterministic_net}.

In MC-dropout, dropout is kept active at evaluation. Each forward pass
therefore draws a random collection of dropout masks, which are
applied to selected layer activations. For a given $\hat\theta \in
\Theta$, we collect all such sources of test-time randomness into a
random element $\theta \sim \pi_{\hat\theta}$, where
$\pi_{\hat\theta}$ is the associated probability law induced by the
chosen dropout rates and parameters $\hat\theta \in \Theta$.

In the specific case of MC-dropout, and for fixed learned parameters $\hat\theta$, the random element $\theta$ may be identified with the
collection of Bernoulli masks applied across all dropout layers during
a single forward pass. Concretely, for a dropout layer of width
$N_{\mathrm{width}}$ we draw a mask
\begin{equation}\label{eq:mlmc:dropout_mask}
    \vec B = (B_1,\dots,B_{N_{\mathrm{width}}})^{\top},
    \qquad
    B_j \overset{\mathrm{iid}}{\sim} \mathrm{Bernoulli}(1-p_{\mathrm{drop}}),
\end{equation}
where $p_{\mathrm{drop}}\in(0,1)$ denotes the dropout probability. A
typical dropout layer then acts as

\begin{equation}\label{eq:mlmc:dropout_layer}
    z_{k+1} :=\vec B^{(k)} \odot \sigma \big( W_k z_k + b_k\big),
\end{equation}
where $\odot$ denotes componentwise multiplication. (Any standard
dropout scaling convention, such as inverted dropout, is absorbed into
the definition of $\mathscr{D}(x;\theta)$.)

Collecting the instances of the dropout masks $\{\vec B^{(k)}\}_{k=0}^{L-1}$ and pairing them with the trained
weights $\hat\theta \in \Theta$, yields a single realisation of $\theta \sim \pi_{\hat\theta}$. Repeated stochastic forward passes correspond to independent draws
$\theta_1,\theta_2,\dots \overset{\mathrm{iid}}{\sim}\pi_{\hat\theta}$ with $\hat\theta$ held fixed. For notational simplicity we write the
resulting randomised predictor as $\mathscr{D}(x;\theta)$, with the
understanding that $\hat\theta$ is fixed and $\theta$ encodes only
test-time randomness. This stochastic evaluation is illustrated in
Figure~\ref{fig:mlmc:dropout_net}.

\begin{figure}[h!]
  \centering
  \begin{tikzpicture}[
    neuron/.style={circle, draw, minimum size=8mm, fill=blue!10},
    dropout/.style={circle, draw, dashed, minimum size=8mm, fill=red!10},
    layer/.style={rectangle, draw=none, minimum height=1cm},
    arrow/.style={-{Latex}, thick},
    every label/.append style={font=\footnotesize}
]

\node[neuron, label=above:Input] (x) at (0,0) {$x$};

\node[neuron] (h0) at (2,2) {};
\node[neuron] (h1) at (2,1) {};
\node[neuron] (h2) at (2,0) {};
\node[neuron] (h3) at (2,-1) {};
\node[neuron] (h4) at (2,-2) {};

\node[neuron] (h5) at (4,2) {};
 \node[neuron] (h6) at (4,1) {};
\node[neuron] (h7) at (4,0) {};
\node[neuron] (h8) at (4,-1) {};
\node[neuron] (h9) at (4,-2) {};

\node[neuron] (d0) at (6,2) {};
\node[neuron] (d1) at (6,1) {};
\node[neuron] (d2) at (6,0) {};
\node[neuron] (d3) at (6,-1) {};
\node[neuron] (d4) at (6,-2) {};
\node[neuron] (d5) at (8,2) {};
\node[neuron] (d6) at (8,1) {};
\node[neuron] (d7) at (8,0) {};
\node[neuron] (d8) at (8,-1) {};
\node[neuron] (d9) at (8,-2) {};
\node[neuron, label=above:Output] (y) at (10,0) {$\mathscr{D}(x; \hat{\theta})$};

\draw[arrow] (x) -- (h0);
\draw[arrow] (x) -- (h1);
\draw[arrow] (x) -- (h2);
\draw[arrow] (x) -- (h3);
\draw[arrow] (x) -- (h4);
\draw[arrow] (h0) -- (h5);
\draw[arrow] (h0) -- (h6);
\draw[arrow] (h0) -- (h7);
\draw[arrow] (h0) -- (h8);
\draw[arrow] (h0) -- (h9);
\draw[arrow] (h1) -- (h5);
\draw[arrow] (h1) -- (h6);
\draw[arrow] (h1) -- (h7);
\draw[arrow] (h1) -- (h8);
\draw[arrow] (h1) -- (h9);
\draw[arrow] (h2) -- (h5);
\draw[arrow] (h2) -- (h6);
\draw[arrow] (h2) -- (h7);
\draw[arrow] (h2) -- (h8);
\draw[arrow] (h2) -- (h9);
\draw[arrow] (h3) -- (h5);
\draw[arrow] (h3) -- (h6);
\draw[arrow] (h3) -- (h7);
\draw[arrow] (h3) -- (h8);
\draw[arrow] (h3) -- (h9);
\draw[arrow] (h4) -- (h5);
\draw[arrow] (h4) -- (h6);
\draw[arrow] (h4) -- (h7);
\draw[arrow] (h4) -- (h8);
\draw[arrow] (h4) -- (h9);
\draw[arrow] (h5) -- (d0);
\draw[arrow] (h5) -- (d1);
\draw[arrow] (h5) -- (d2);
\draw[arrow] (h5) -- (d3);
\draw[arrow] (h5) -- (d4);
\draw[arrow] (h6) -- (d0);
\draw[arrow] (h6) -- (d1);
\draw[arrow] (h6) -- (d2);
\draw[arrow] (h6) -- (d3);
\draw[arrow] (h6) -- (d4);
\draw[arrow] (h7) -- (d0);
\draw[arrow] (h7) -- (d1);
\draw[arrow] (h7) -- (d2);
\draw[arrow] (h7) -- (d3);
\draw[arrow] (h7) -- (d4);
\draw[arrow] (h8) -- (d0);
\draw[arrow] (h8) -- (d1);
\draw[arrow] (h8) -- (d2);
\draw[arrow] (h8) -- (d3);
\draw[arrow] (h8) -- (d4);
\draw[arrow] (h9) -- (d0);
\draw[arrow] (h9) -- (d1);
\draw[arrow] (h9) -- (d2);
\draw[arrow] (h9) -- (d3);
\draw[arrow] (h9) -- (d4);
\draw[arrow] (d0) -- (d5);
\draw[arrow] (d0) -- (d6);
\draw[arrow] (d0) -- (d7);
\draw[arrow] (d0) -- (d8);
\draw[arrow] (d0) -- (d9);
\draw[arrow] (d1) -- (d5);
\draw[arrow] (d1) -- (d6);
\draw[arrow] (d1) -- (d7);
\draw[arrow] (d1) -- (d8);
\draw[arrow] (d1) -- (d9);
\draw[arrow] (d2) -- (d5);
\draw[arrow] (d2) -- (d6);
\draw[arrow] (d2) -- (d7);
\draw[arrow] (d2) -- (d8);
\draw[arrow] (d2) -- (d9);
\draw[arrow] (d3) -- (d5);
\draw[arrow] (d3) -- (d6);
\draw[arrow] (d3) -- (d7);
\draw[arrow] (d3) -- (d8);
\draw[arrow] (d3) -- (d9);
\draw[arrow] (d4) -- (d5);
\draw[arrow] (d4) -- (d6);
\draw[arrow] (d4) -- (d7);
\draw[arrow] (d4) -- (d8);
\draw[arrow] (d4) -- (d9);
\draw[arrow] (d5) -- (y);
\draw[arrow] (d6) -- (y);
\draw[arrow] (d7) -- (y);
\draw[arrow] (d8) -- (y);
\draw[arrow] (d9) -- (y);

\end{tikzpicture}
  \caption{Schematic of the surrogate network architecture in its
  deterministic form. The trained parameters $\hat\theta$ are fixed and
  repeated evaluations produce identical outputs $\mathscr{D}(x;\hat\theta)$.}
  \label{fig:mlmc:deterministic_net}
\end{figure}
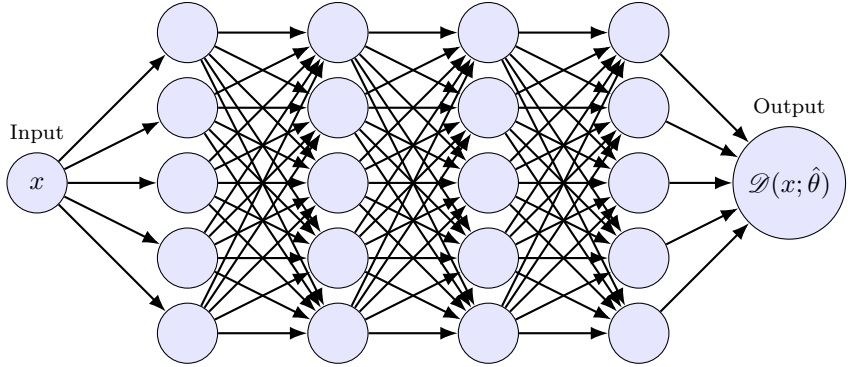

\begin{figure}[h!]
  \centering
  \begin{tikzpicture}[
    neuron/.style={circle, draw, minimum size=8mm, fill=blue!10},
    dropout/.style={circle, draw, dashed, minimum size=8mm, fill=red!10},
    layer/.style={rectangle, draw=none, minimum height=1cm},
    arrow/.style={-{Latex}, thick},
    every label/.append style={font=\footnotesize}
]

\node[neuron, label=above:Input] (x) at (0,0) {$x$};

\node[dropout] (h0) at (2,2) {};
\node[dropout] (h1) at (2,1) {};
\node[neuron] (h2) at (2,0) {};
\node[neuron] (h3) at (2,-1) {};
\node[dropout] (h4) at (2,-2) {};

\node[neuron] (h5) at (4,2) {};
\node[dropout] (h6) at (4,1) {};
\node[dropout] (h7) at (4,0) {};
\node[neuron] (h8) at (4,-1) {};
\node[dropout] (h9) at (4,-2) {};

\node[dropout] (d0) at (6,2) {};
\node[neuron] (d1) at (6,1) {};
\node[dropout] (d2) at (6,0) {};
\node[dropout] (d3) at (6,-1) {};
\node[neuron] (d4) at (6,-2) {};
\node[neuron, label=above:Hidden layer] (d5) at (8,2) {};
\node[neuron] (d6) at (8,1) {};
\node[neuron] (d7) at (8,0) {};
\node[neuron] (d8) at (8,-1) {};
\node[neuron] (d9) at (8,-2) {};
\node[neuron, label=above:Output] (y) at (10,0) {$\mathscr{D}(x; \theta)$};

\draw[arrow] (x) -- (h2);
\draw[arrow] (x) -- (h3);
\draw[arrow] (h2) -- (h5);
\draw[arrow] (h2) -- (h8);
\draw[arrow] (h3) -- (h5);
\draw[arrow] (h3) -- (h8);
\draw[arrow] (h5) -- (d1);
\draw[arrow] (h5) -- (d4);
\draw[arrow] (h8) -- (d1);
\draw[arrow] (h8) -- (d4);
\draw[arrow] (d1) -- (d5);
\draw[arrow] (d1) -- (d6);
\draw[arrow] (d1) -- (d7);
\draw[arrow] (d1) -- (d8);
\draw[arrow] (d1) -- (d9);
\draw[arrow] (d4) -- (d5);
\draw[arrow] (d4) -- (d6);
\draw[arrow] (d4) -- (d7);
\draw[arrow] (d4) -- (d8);
\draw[arrow] (d4) -- (d9);
\draw[arrow] (d5) -- (y);
\draw[arrow] (d6) -- (y);
\draw[arrow] (d7) -- (y);
\draw[arrow] (d8) -- (y);
\draw[arrow] (d9) -- (y);

\node at ($(h5)!0.5!(d4) + (-1,2.75)$) {$\overbrace{\hspace{4.5cm}}^{\text{Dropout layers}}$};
\end{tikzpicture}
  \caption{The same network evaluated with dropout active. Each forward
  pass draws dropout masks (encoded in $\theta \sim \pi_{\hat\theta}$), yielding a
  randomised output $\mathscr{D}(x;\theta)$. Repeated evaluations
  resample $\theta$ and can therefore produce different outputs.}
  \label{fig:mlmc:dropout_net}
\end{figure}
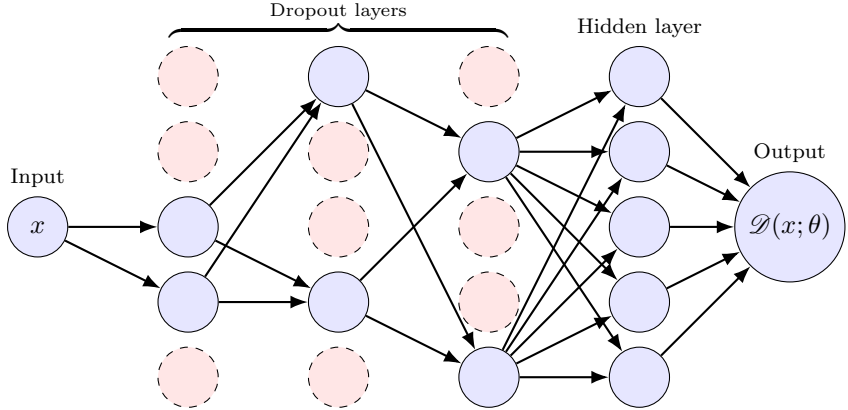

Throughout this section we present the moment identities in the scalar-output
case $N_{\mathrm{out}}=1$, so that $\mathscr{D}(x;\theta)\in\mathbb{R}$ and
$\mu_2(x)$ denotes the predictive variance. For vector-valued outputs
$\mathscr{D}(x;\theta)\in\mathbb{R}^{N_{\mathrm{out}}}$, the same constructions
and results apply componentwise by applying the definitions below to each
coordinate of $\mathscr{D}(x;\theta)$.

We are interested in the first two central moments of this randomised
predictor with respect to the epistemic randomness induced by dropout.
We denote these by
\begin{equation}\label{eq:mlmc:target_moments}
    \mu(x)
    :=
    \mathbb{E}_{\theta}\big[\mathscr{D}(x;\theta)\big],
    \qquad
    \mu_2(x)
    :=
    \mathrm{Var}_{\theta}\big[\mathscr{D}(x;\theta)\big]
    =
    \mathbb{E}_{\theta}\Big[\big(\mathscr{D}(x;\theta)-\mu(x)\big)^2\Big].
\end{equation}
We use the notation $\mu, \mu_2$ to emphasise that these are moments of
the dropout-induced predictive distribution at fixed $x$. This choice
is convenient later when we introduce sample estimators and their
sampling variances, where using $\sigma^2$ for multiple distinct
variances can become ambiguous.

In the analysis of variance estimators we will also require the fourth
central moment, which we denote by
\begin{equation}\label{eq:mlmc:fourth_central_moment}
    \mu_4(x)
    :=
    \mathbb{E}_{\theta}\Big[\big(\mathscr{D}(x;\theta)-\mu(x)\big)^4\Big].
\end{equation}
This moment controls the sampling variability of empirical variance
estimators and will therefore enter the allocation and complexity
results.

The randomness in $\mathscr{D}(x;\theta)$ captures epistemic
uncertainty arising from limited data and model uncertainty, and should
be distinguished from aleatoric uncertainty, which represents
irreducible variability or systematic effects in the data-generating
process. The dropout-induced predictive distribution is not the exact
Bayesian predictive distribution, since dropout provides only a
variational approximation to Bayesian inference \cite{gal2016dropout}.
A known limitation is that the resulting uncertainty estimates can be
miscalibrated, for example due to architecture choice, dropout rate or
unmodelled aleatoric noise \cite{guo2017calibration}. Nonetheless,
MC-dropout is computationally lightweight and is widely used as a
practical route to epistemic uncertainty estimation
\cite{gal2016uncertaintyPHD,kendall2017uncertainties}.

\subsection{Single-fidelity Monte Carlo estimators}\label{subsec:mlmc:single_fidelity}
To approximate the dropout-induced moments
\eqref{eq:mlmc:target_moments} at a fixed input $x$, a natural
approach is Monte Carlo sampling over independent realisations of the
test-time randomness. Let $\{\theta_t\}_{t=1}^T$ be i.i.d.\ draws from $\pi_{\hat\theta}$, the law induced by dropout for fixed weights $\hat\theta$, written $\theta_t
\overset{\mathrm{iid}}{\sim} \pi_{\hat\theta}$. For notational simplicity, we omit the explicit dependence of the associated probability law $\pi_{\hat\theta}$ on the trained weights $\hat\theta \in \Theta$, and denote the law by $\pi$. We then define the
single-fidelity (single-$T$) estimators of the predictive mean and
predictive variance by
\begin{equation}\label{eq:mlmc:single_fidelity_estimators}
    Y(x;T)
    :=
    \frac{1}{T}\sum_{t=1}^T \mathscr{D}(x;\theta_t),
    \qquad
    V(x;T)
    :=
    \frac{1}{T-1}\sum_{t=1}^T \big(\mathscr{D}(x;\theta_t)-Y(x,T)\big)^2.
\end{equation}
Here $T$ is the fidelity parameter, larger $T$ yields lower sampling
variability but requires more stochastic forward passes.

\subsubsection{Moments of single-fidelity estimators}\label{subsubsec:mlmc:single_fidelity_moments}
The estimators $Y(x;T)$ and $V(x;T)$ are random variables, since they
depend on the random draws $\{\theta_t\}_{t=1}^T$. Their expectations
recover the target moments \eqref{eq:mlmc:target_moments}, and their
sampling variances quantify the accuracy of a single Monte Carlo
estimate at fidelity $T$.

\begin{lemma}[Moments of single-fidelity estimators]\label{lem:mlmc:single_fidelity_moments}
If $\theta_t \overset{\mathrm{iid}}{\sim}\pi$ and $T\geq 2$, then
\begin{equation}\label{eq:mlmc:single_fidelity_moments}
    \begin{aligned}
        \mathbb{E}\big[Y(x;T)\big]
        &=
        \mu(x),
        &
        \mathrm{Var}\big[Y(x;T)\big]
        &=
        \frac{1}{T}\mu_2(x),
        \\
        \mathbb{E}\big[V(x;T)\big]
        &=
        \mu_2(x),
        &
        \mathrm{Var}\big[V(x;T)\big]
        &=
        \frac{1}{T}\left(\mu_4(x)-\frac{T-3}{T-1}\mu_2^2(x)\right),
    \end{aligned}
\end{equation}
where $\mu_4(x)$ is defined in \eqref{eq:mlmc:fourth_central_moment}.
\end{lemma}

\subsubsection{Outer sampling and empirical estimators}\label{subsubsec:mlmc:outer_sampling}
In applications we often wish to quantify not only the estimators
$Y(x;T)$ and $V(x;T)$ themselves, but also their sampling variability
$\mathrm{Var}[Y(x;T)]$ and $\mathrm{Var}[V(x;T)]$. A direct empirical
approach is to draw $M$ independent \emph{replicates} of the estimators
\eqref{eq:mlmc:single_fidelity_estimators}. That is, for
$m=1,\dots,M$, let $\{\theta^{(m)}_t\}_{t=1}^T$ be i.i.d.\ draws from
$\pi$, and define
\[
Y^{(m)}(x;T) := \frac{1}{T}\sum_{t=1}^T \mathscr{D}(x;\theta^{(m)}_t),
\qquad
V^{(m)}(x;T) := \frac{1}{T-1}\sum_{t=1}^T \big(\mathscr{D}(x;\theta^{(m)}_t)-Y^{(m)}(x;T)\big)^2.
\]

We then form the empirical estimators
\begin{equation}\label{eq:mlmc:outer_empirical_estimators}
    \begin{aligned}
      \overline{Y}(x;M,T)
        &:=
        \frac{1}{M}\sum_{m=1}^M Y^{(m)}(x,T),
        &
        \mathcal{S}^2_{Y}(x;M,T)
        &:=
        \frac{1}{M-1}\sum_{m=1}^M \big(Y^{(m)}(x,T)-\overline{Y}(x;M,T)\big)^2,
        \\
        \overline{V}(x;M,T)
        &:=
        \frac{1}{M}\sum_{m=1}^M V^{(m)}(x,T),
        &
        \mathcal{S}^2_{V}(x;M,T)
        &:=
        \frac{1}{M-1}\sum_{m=1}^M \big(V^{(m)}(x,T)-\overline{V}(x;M,T)\big)^2.
    \end{aligned}
\end{equation}
The single-fidelity construction requires $MT$ evaluations of
$\mathscr{D}(x;\theta)$ and therefore has cost scaling linearly in both
$M$ and $T$. Achieving small empirical variances
$\mathcal{S}^2_{Y}(x;M,T)$ and $\mathcal{S}^2_{V}(x;M,T)$ can therefore
become computationally expensive. This motivates the multilevel
constructions introduced in the next subsection, which aim to reduce
cost for a prescribed accuracy by combining several fidelities.

\subsection{Multilevel estimator for the predictive mean}\label{subsec:mlmc:mean}
We now construct a multilevel Monte Carlo (MLMC) estimator for the
dropout-induced predictive mean $\mu(x)$. The key idea is to combine
estimators of increasing fidelities $T_0<\cdots<T_L$, using many cheap
low-fidelity samples and few expensive high-fidelity samples, while
coupling neighbouring fidelities to reduce the variance of level
differences.

\begin{remark}[Estimator hierarchy rather than discretisation hierarchy]\label{rem:mlmc:estimator_hierarchy}
In contrast to classical MLMC for biased discretisations, the estimators
$Y(x,T_\ell)$ defined in \eqref{eq:mlmc:single_fidelity_estimators} are
unbiased for $\mu(x)$ for every $\ell$. The fidelity parameter $T_\ell$
controls sampling variance and cost, not modelling bias. The multilevel
construction therefore acts as a structured control-variate scheme built
from coupled estimators of the same target moment.
\end{remark}

\subsubsection{Coupling across fidelities}\label{subsubsec:mlmc:mean_coupling}
Fix a strictly increasing sequence of fidelities
\begin{equation}\label{eq:mlmc:mean_fidelity_ladder}
    1 \leq T_0 < T_1 < \cdots < T_L,
\end{equation}
and let $\vec T := (T_0,\dots,T_L)$ and $\vec M := (M_0,\dots,M_L)$
denote the fidelity ladder and the corresponding numbers of samples
(per level). The MLMC construction couples \emph{neighbouring}
fidelities within each level difference: for a fixed level $\ell\ge 1$
and outer index $m \in [1, M_{\ell}]$, we form a coupled pair
$\big(Y^{(m)}(x,T_{\ell-1}),Y^{(m)}(x,T_\ell)\big)$ by reusing the
first $T_{\ell-1}$ dropout realisations inside the $T_\ell$-sample
mean, and generating only $T_\ell-T_{\ell-1}$ additional
draws. Dropout realisations are sampled independently across outer replications $m$, so that for any fixed $\ell$ the estimators $Y^{(m_1)}(x;T_\ell)$ and $Y^{(m_2)}(x;T_\ell)$ are independent whenever $m_1\neq m_2$, while for a given $m$ the samples used at fidelity $T_{\ell-1}$ are reused as a prefix of those used at $T_\ell$. Equivalently, if $\theta^{(m)}_t
\overset{\mathrm{iid}}{\sim}\pi$, then
\begin{equation}\label{eq:mlmc:mean_recurrence}
    T_{\ell} Y^{(m)}(x;T_{\ell})
    =
    T_{\ell-1}Y^{(m)}(x;T_{\ell-1})
    +
    \sum_{t = T_{\ell-1}+1}^{T_{\ell}} \mathscr{D}\big(x;\theta^{(m)}_{t}\big),
    \qquad m=1,\dots,M_{\ell}, \quad \ell=1,\dots,L.
\end{equation}
We define the coupled level increments by
\begin{equation}\label{eq:mlmc:mean_increment}
    \Delta Y^{(m)}(x;T_{\ell})
    :=
    Y^{(m)}(x;T_{\ell})-Y^{(m)}(x;T_{\ell-1}),
    \qquad m=1,\dots,M_{\ell}, \quad \ell=1,\dots,L.
\end{equation}
The coupling ensures that $Y^{(m)}(x;T_{\ell})$ and $Y^{(m)}(x;T_{\ell-1})$ are
strongly correlated, so that $\Delta Y^{(m)}(x;T_{\ell})$ typically has much
smaller variance than either term alone. In particular, the telescoping
identity holds \cite{giles2008multilevel},
\begin{equation}\label{eq:mlmc:mean_telescoping}
    \mathbb{E}\big[Y^{(m)}(x;T_L)\big]
    =
    \mathbb{E}\big[Y^{(m)}(x;T_0)\big]
    +
    \sum_{\ell=1}^{L}\mathbb{E}\big[\Delta Y^{(m)}(x;T_\ell)\big] \qquad m=1,\dots,M_{L}.
\end{equation}

\subsubsection{MLMC mean estimator and its variance estimator}\label{subsubsec:mlmc:mean_estimators}
Motivated by \eqref{eq:mlmc:mean_telescoping}, we define the MLMC
estimator of $\mu(x)$ by
\begin{equation}\label{eq:mlmc:mean_mlmc_estimator}
    \mathcal{Y}(x;\vec M,\vec T)
    :=
    \overline{Y}(x;M_0,T_0)
    +
    \sum_{\ell = 1}^{L}\overline{\Delta Y}(x;M_\ell,T_\ell), \qquad
    \overline{\Delta Y}(x;M_\ell,T_\ell)
    :=
    \frac{1}{M_\ell}\sum_{m=1}^{M_\ell}\Delta Y^{(m)}(x;T_\ell).
\end{equation}
where for each $\ell$ the samples $\{\Delta Y^{(m)}(x;T_\ell)\}_{m=1}^{M_\ell}$
are i.i.d.\ across $m$. We quantify the
sampling variance of \eqref{eq:mlmc:mean_mlmc_estimator} using the
empirical variance estimator
\begin{equation}\label{eq:mlmc:mean_mlmc_variance_estimator}
    S^2_Y(x;\vec M,\vec T)
    :=
    \frac{1}{M_0} \mathcal{S}^2_{Y}(x;M_0,T_0)
    +
    \sum_{\ell = 1}^{L}\frac{1}{M_\ell} \mathcal{S}^2_{\Delta Y}(x;M_\ell,T_\ell),
\end{equation}
where $\mathcal{S}^2_{Y}(x;M_0,T_0)$ is defined in
\eqref{eq:mlmc:outer_empirical_estimators} and
\begin{equation}\label{eq:mlmc:mean_increment_sample_variance}
    \mathcal{S}^2_{\Delta Y}(x;M_\ell,T_\ell)
    :=
    \frac{1}{M_\ell-1}\sum_{m=1}^{M_\ell}
    \big(\Delta Y^{(m)}(x;T_\ell)-\overline{\Delta Y}(x;T_\ell)\big)^2,
\end{equation}

\subsubsection{Moment identities}\label{subsubsec:mlmc:mean_moments}
The MLMC estimator \eqref{eq:mlmc:mean_mlmc_estimator} is unbiased for
$\mu(x)$, and the expected value of the empirical variance estimator
\eqref{eq:mlmc:mean_mlmc_variance_estimator} is given by the variance $\mu_2$ multiplied by function depending on $\vec T$ and $\vec M$, respectively.

\begin{lemma}[Moments of the MLMC mean estimator]\label{lem:mlmc:mean_moments}
Let $\vec T=(T_0,\dots,T_L)$ satisfy \eqref{eq:mlmc:mean_fidelity_ladder}
and suppose that, for each level $\ell$, the outer samples
$\{\Delta Y^{(m)}(x;T_\ell)\}_{m=1}^{M_\ell}$ are i.i.d. Then
\begin{equation}\label{eq:mlmc:mean_moments_mean_mean}
    \mathbb{E}\big[\mathcal{Y}(x;\vec M,\vec T)\big] = \mu(x),
\end{equation}
and 
\begin{equation}\label{eq:mlmc:mean_moments_var_mean}
    \mathbb{E}\big[S^2_Y(x;\vec M,\vec T)\big]
    =
    \mu_2(x)\left(
        \frac{1}{M_0T_0}
        +
        \sum_{\ell=1}^{L}\frac{1}{M_\ell}\left(\frac{1}{T_{\ell-1}}-\frac{1}{T_\ell}\right)
    \right).
\end{equation}
\end{lemma}

\subsection{Multilevel estimator for the predictive variance}\label{subsec:mlmc:variance}
We next develop an MLMC estimator for the dropout-induced predictive
variance $\mu_2(x)$. As in the single-fidelity setting, we estimate
$\mu_2(x)$ using the unbiased sample variance $V(x,T)$ based on $T$
stochastic forward passes. The multilevel construction mirrors that for
the mean, but requires a careful coupling of variance estimators across
fidelities.

\subsubsection{Coupling via an online variance update}\label{subsubsec:mlmc:variance_coupling}
Fix $2\leq T_0 < T_1 < \cdots < T_L$ and let
$\{\theta_t\}_{t=1}^{T_L}$ be i.i.d.\ draws from $\pi$. Assume in
addition that $T_\ell-T_{\ell-1}\ge 2$ for $\ell=1,\dots,L$, so that
the unbiased sample variance on each new block is well-defined. For
notational convenience write
\[
D_t(x) := \mathscr{D}(x;\theta_t), \qquad t=1,\dots,T_L,
\]
and define, for integers $1\leq a < b$,
\[
\overline{D}_{a:b}(x) := \frac{1}{b-a+1}\sum_{t=a}^{b} D_t(x),
\qquad
V_{a:b}(x) := \frac{1}{b-a}\sum_{t=a}^{b}\big(D_t(x)-\overline{D}_{a:b}(x)\big)^2.
\]
With this notation the single-fidelity variance estimator in
\eqref{eq:mlmc:single_fidelity_estimators} is $V(x,T)=V_{1:T}(x)$.

To couple $V(x,T_{\ell-1})$ and $V(x,T_\ell)$, we reuse the first
$T_{\ell-1}$ dropout draws in the $T_\ell$-sample estimate and add only
$T_\ell-T_{\ell-1}$ fresh draws. The resulting coupled variances satisfy
the pooled-variance identity (equivalently, an online variance update)
\cite{ChanGolubLeveque1979}:
\begin{equation}\label{eq:mlmc:variance_update}
\begin{aligned}
    (T_\ell-1) V(x,T_\ell)
    ={}&(T_{\ell-1}-1) V(x,T_{\ell-1})
    + (T_\ell-T_{\ell-1}-1) V_{T_{\ell-1}+1:T_\ell}(x)
    \\
    &\quad + \frac{T_{\ell-1}(T_\ell-T_{\ell-1})}{T_\ell}
    \Big(\overline{D}_{1:T_{\ell-1}}(x)-\overline{D}_{T_{\ell-1}+1:T_\ell}(x)\Big)^2.
\end{aligned}
\end{equation}
We define the coupled level increments by
\begin{equation}\label{eq:mlmc:variance_increment}
    \Delta V(x;T_{\ell})
    :=
    V(x;T_{\ell})-V(x;T_{\ell-1}),
    \qquad \ell=1,\dots,L.
\end{equation}
Since both $V(x;T_{\ell-1})$ and $V(x;T_\ell)$ are unbiased for $\mu_2(x)$
(Lemma~\ref{lem:mlmc:single_fidelity_moments}), we have
$\mathbb{E}[\Delta V(x;T_\ell)]=0$, and the MLMC construction again
relies on the reduced variance of the coupled differences.

\subsubsection{MLMC variance estimator and its variance estimator}\label{subsubsec:mlmc:variance_estimators}
Let $\vec T=(T_0,\dots,T_L)$ and $\vec M=(M_0,\dots,M_L)$. For each level
$\ell$, we generate $M_\ell$ i.i.d.\ outer replicates of the coupled
increment $\Delta V(x;T_\ell)$, and define
the MLMC estimator of $\mu_2(x)$ by
\begin{equation}\label{eq:mlmc:variance_mlmc_estimator}
    \mathcal{V}(x;\vec M,\vec T)
    :=
    \overline{V}(x; M_0, T_0)
    +
    \sum_{\ell = 1}^{L}\overline{\Delta V}(x; M_\ell, T_\ell), \quad \overline{\Delta V}(x; M_\ell, T_\ell):=\frac{1}{M_\ell}\sum_{m=1}^{M_\ell}\Delta V^{(m)}(x;T_\ell).
\end{equation}
We quantify the sampling variance of \eqref{eq:mlmc:variance_mlmc_estimator}
using the empirical variance estimator
\begin{equation}\label{eq:mlmc:variance_mlmc_variance_estimator}
    S^2_V(x;\vec M,\vec T)
    :=
    \frac{1}{M_0} \mathcal{S}^2_{V}(x;M_0,T_0)
    +
    \sum_{\ell = 1}^{L}\frac{1}{M_\ell} \mathcal{S}^2_{\Delta V}(x;M_\ell,T_\ell),
\end{equation}
where $\mathcal{S}^2_{V}(x;M_0,T_0)$ is defined in
\eqref{eq:mlmc:outer_empirical_estimators} and
\begin{equation}\label{eq:mlmc:variance_increment_sample_variance}
    \mathcal{S}^2_{\Delta V}(x;M_\ell,T_\ell)
    :=
    \frac{1}{M_\ell-1}\sum_{m=1}^{M_\ell}
    \big(\Delta V^{(m)}(x;T_\ell)-\overline{\Delta V}(x; M_\ell, T_\ell)\big)^2.
\end{equation}

\subsubsection{Moment identities}\label{subsubsec:mlmc:variance_moments}
The MLMC estimator \eqref{eq:mlmc:variance_mlmc_estimator} is unbiased
for $\mu_2(x)$, and the expected value of the empirical variance
estimator \eqref{eq:mlmc:variance_mlmc_variance_estimator} is soley dependent on the $2^{\mathrm{nd}}$ and $4^{\mathrm{th}}$-order moments, $\vec M$, and $\vec T$.

\begin{lemma}[Covariance of unbiased sample-variance estimators with overlapping samples]
\label{lem:cov_overlap_var}
Fix $x$ and let $\{\theta_i\}_{i=1}^{T_\ell}$ be i.i.d.\ draws from $\pi$.
Assume $\mathscr{D}(x;\theta)$ has finite fourth central moment and write
$\mu_2(x)$ and $\mu_4(x)$ for the second and fourth central moments defined in
\eqref{eq:mlmc:target_moments} and \eqref{eq:mlmc:fourth_central_moment}.
For $\ell\ge 1$, form $V(x;T_{\ell-1})$ from the first $T_{\ell-1}$ samples and
$V(x;T_\ell)$ from the first $T_\ell$ samples. Then
\begin{equation}\label{eq:cov_overlap:cov_formula}
\begin{aligned}
    \mathrm{Cov}\bigl[V(x;T_{\ell}),V(x;T_{\ell-1})\bigr]
    &=
    \frac{T_{\ell-1}-1}{T_{\ell}-1}  \mathrm{Var}[V(x;T_{\ell-1})]
    +\frac{T_{\ell}-T_{\ell-1}}{T_{\ell-1}T_{\ell}(T_{\ell}-1)}\bigl(\mu_4(x)-3\mu_2^2(x)\bigr).
\end{aligned}
\end{equation}
\end{lemma}

\begin{theorem}[Moments of the MLMC variance estimator]\label{thm:mlmc:variance_moments}
Let $\vec T=(T_0,\dots,T_L)$ with $T_0\ge 2$ and suppose that, for each
level $\ell$, the outer samples
$\{\Delta V^{(m)}(x;T_\ell)\}_{m=1}^{M_\ell}$ are i.i.d. Then
\begin{equation}\label{eq:mlmc:variance_unbiasedness}
    \mathbb{E}\big[\mathcal{V}(x;\vec M,\vec T)\big] = \mu_2(x),
\end{equation}
and
\begin{equation}\label{eq:mlmc:variance_sampling_variance}
\begin{split}
    \mathbb{E}\big[S_V^2(x;\vec M,\vec T)\big]
    ={}&
    \frac{1}{M_0T_0}\left(\mu_4(x) - \frac{T_0-3}{T_0-1}\mu_2^2(x)\right)
    +
    \sum_{\ell=1}^L \frac{1}{M_{\ell}}
    \left(\frac{1}{T_{\ell-1}}-\frac{1}{T_{\ell}}\right)\big(\mu_4(x)-3 \mu_2^2(x)\big)
    \\
    &+
    \sum_{\ell=1}^L \frac{2}{M_{\ell}}
    \left(\frac{1}{T_{\ell-1}-1}-\frac{1}{T_{\ell}-1}\right)\mu_2^2(x),
\end{split}
\end{equation}
where $\mu_4(x)$ is defined in \eqref{eq:mlmc:fourth_central_moment}.
\end{theorem}

\begin{proof}
The unbiasedness \eqref{eq:mlmc:variance_unbiasedness} follows from
linearity of expectation and Lemma~\ref{lem:mlmc:single_fidelity_moments}:
$\mathbb{E}[V(x;T_0)]=\mu_2(x)$ and $\mathbb{E}[\Delta V(x;T_\ell)]=0$. 
For the sampling variance, we have by the linearity of expectation that
\begin{equation}
    \mathbb{E}[S^2_V(x;\vec M,\vec T)]
    =
    \frac{1}{M_0} \mathbb{E}[\mathcal{S}^2_{V}(x;M_0,T_0)]
    +
    \sum_{\ell = 1}^{L}\frac{1}{M_\ell} \mathbb{E}[\mathcal{S}^2_{\Delta V}(x;M_\ell,T_\ell)],
\end{equation}
The definitions of $\mathcal{S}^2_{V}$ and $\mathcal{S}^2_{\Delta V}$ are that they are the unbiased variance estimators of $V$ and $\Delta V$ respectively. Under the assumption that the outer samples are i.i.d., we have that expectation of the unbiased variance estimator is the variance, $\mathbb{E}[\mathcal{S}^2_{Z}] = \mathrm{Var}[Z]$, and thus
\begin{equation}
    \mathbb{E}[S^2_V(x;\vec M,\vec T)]
    =
    \frac{1}{M_0}\mathrm{Var}\big[V(x;T_0)\big]
    +
    \sum_{\ell=1}^L \frac{1}{M_\ell}\mathrm{Var}\big[\Delta V(x;T_\ell)\big].
\end{equation}
Expanding the summand, we have the following expression
\[
\mathrm{Var}\big[\Delta V(x;T_\ell)\big]
=
\mathrm{Var}\big[V(x;T_{\ell-1})\big]-2\mathrm{Cov}\big[V(x;T_{\ell-1}),V(x;T_{\ell})\big]+\mathrm{Var}\big[V(x;T_\ell)\big].
\]
The first and last terms are given by Lemma~\ref{lem:mlmc:single_fidelity_moments}.
For $\ell\ge 1$, the coupling reuses the first $T_{\ell-1}$ stochastic
forward passes inside the $T_\ell$-sample variance. Hence, applying Lemma~\ref{lem:cov_overlap_var} to the nested pair $\bigl(V(x;T_{\ell-1}),V(x;T_\ell)\bigr)$ gives a closed-form expression,
\[
\mathrm{Var}\big[\Delta V(x;T_\ell)\big]
=
\left(\frac{1}{T_{\ell-1}}-\frac{1}{T_{\ell}}\right)\big(\mu_4(x)-3 \mu_2^2(x)\big)
+
2\left(\frac{1}{T_{\ell-1}-1}-\frac{1}{T_{\ell}-1}\right)\mu_2^2(x),
\]
which, upon substitution into the variance decomposition above, gives
\eqref{eq:mlmc:variance_sampling_variance}.
\end{proof}

\begin{remark}\label{rem:mlmc:variance_nonneg}
The right-hand side of \eqref{eq:mlmc:variance_sampling_variance} is
non-negative since it equals $\mathrm{Var}\big[\mathcal{V}(x;\vec M,\vec T)\big]$
by construction. Moreover, under the assumption of zero excess kurtosis,
$\mu_4(x)=3\mu_2^2(x)$, the expression \eqref{eq:mlmc:variance_sampling_variance}
reduces to
\begin{equation}\label{eq:mlmc:variance_level_vars_kappa0}
\begin{split}
    \mathbb{E}\big[S_V^2(x;\vec M,\vec T)\big]
    =
    2\mu_2^2(x)\bra{\frac{1}{M_0}\frac{1}{T_0-1}
    +
    \sum_{\ell=1}^L \frac{1}{M_{\ell}}
    \bra{\frac{1}{T_{\ell-1}-1}-\frac{1}{T_{\ell}-1}}}.
\end{split}
\end{equation}
In particular, for a geometric ladder (for example $T_\ell=2^\ell T_0$) the
single-level correction variances decay like $O(T_{\ell-1}^{-2})$, and the
optimal allocation $\vec M$ depends only on $\vec T$ because the dependence on
the underlying distribution enters through the common factor $2\mu_2^2(x)$.
\end{remark}

\section{Computation and allocation}\label{sec:mlmc_dropout_computation}
In this section we translate the estimators of
Section~\ref{sec:mlmc_dropout_estimators} into a practical coupled
sampling routine and then discuss how to allocate samples across
fidelity levels to minimise sampling variance for a given computational
budget.

\subsection{Coupled sampling algorithm}\label{subsec:mlmc:algorithm}

The multilevel estimators for the predictive mean and variance rely on
\emph{coupled} evaluations within each level increment. For level $\ell\ge 1$
we compute $Y(x;T_{\ell-1})$ and $Y(x;T_\ell)$ (and similarly $V(x;T_{\ell-1})$
and $V(x;T_\ell)$) from the same underlying dropout realisations, reusing the
first $T_{\ell-1}$ stochastic forward passes inside the $T_\ell$-sample
estimators. Algorithm~\ref{alg:mlmc:coupled_sampling} implements this coupling
and returns the MLMC estimators $\mathcal{Y}(x;\vec M,\vec T)$ and
$\mathcal{V}(x;\vec M,\vec T)$, along with the empirical variance estimators
$S_Y^2(x;\vec M,\vec T)$ and $S_V^2(x;\vec M,\vec T)$ defined in
\eqref{eq:mlmc:mean_mlmc_estimator}--\eqref{eq:mlmc:variance_mlmc_variance_estimator}.

\begin{algorithm}[h!]
\caption{Coupled sampling for MLMC mean and variance estimators}\label{alg:mlmc:coupled_sampling}
\begin{algorithmic}[1]
\Require input $x$; fidelity ladder $\vec T=(T_0,\dots,T_L)$ with $2\le T_0<\cdots<T_L$;
sample counts $\vec M=(M_0,\dots,M_L)$ with $M_\ell\ge 1$ and $M_\ell\ge 2$ whenever a sample variance is formed
\Ensure samples $\{Y^{(m)}(x;T_0),V^{(m)}(x;T_0)\}$ and increments
$\{\Delta Y^{(m)}(x;T_\ell),\Delta V^{(m)}(x;T_\ell)\}$; estimators $\mathcal{Y},\mathcal{V}$ and empirical variances $S_Y^2,S_V^2$

\State Initialise arrays $\mathtt{Y}_0 \gets [ ]$, $\mathtt{V}_0 \gets [ ]$
\For{$m=1$ to $M_0$}
    \State Draw $T_0$ i.i.d.\ $\theta^{(m)} \sim \pi$ and store evaluations $\mathtt{D}^{(m)}=[D_1^{(m)},\dots,D_{T_0}^{(m)}]$ with $D_t^{(m)}=\mathscr{D}(x;\theta_t^{(m)})$
    \State Compute $Y^{(m)} \gets \text{mean}(\mathtt{D}^{(m)})$
    \State Compute $V^{(m)} \gets \text{sample\_var}(\mathtt{D}^{(m)})$ 
    \State Append $Y^{(m)}$ to $\mathtt{Y}_0$ and append $V^{(m)}$ to $\mathtt{V}_0$
\EndFor

\State Initialise arrays $\mathtt{dY}_\ell \gets [ ]$, $\mathtt{dV}_\ell \gets [ ]$ for $\ell=1,\dots,L$
\For{$\ell = 1$ to $L$}
    \For{$m=1$ to $M_\ell$}
        \State Draw $T_\ell-T_{\ell-1}$ i.i.d.\ $\theta^{(m)} \sim \pi$ and append evaluations $[D_{T_{\ell-1}+1}^{(m)},\dots,D_{T_\ell}^{(m)}]$ to $\mathtt{D}^{(m)}$
        \State Compute $Y_{\mathrm{c}}^{(m)} \gets \text{mean}([D_1^{(m)},\dots,D_{T_{\ell-1}}^{(m)}])$ and $Y_{\mathrm{f}}^{(m)} \gets \text{mean}([D_1^{(m)},\dots,D_{T_{\ell}}^{(m)}])$
        \State Compute $V_{\mathrm{c}}^{(m)} \gets \text{sample\_var}([D_1^{(m)},\dots,D_{T_{\ell-1}}^{(m)}])$ and $V_{\mathrm{f}}^{(m)} \gets \text{sample\_var}([D_1^{(m)},\dots,D_{T_{\ell}}^{(m)}])$
        \State Append $Y_{\mathrm{f}}^{(m)}-Y_{\mathrm{c}}^{(m)}$ to $\mathtt{dY}_\ell$
        \State Append $V_{\mathrm{f}}^{(m)}-V_{\mathrm{c}}^{(m)}$ to $\mathtt{dV}_\ell$
    \EndFor
\EndFor

\State Compute $\mathcal{Y} \gets \text{mean}(\mathtt{Y}_0) + \sum_{\ell=1}^L \text{mean}(\mathtt{dY}_\ell)$
\State Compute $\mathcal{V} \gets \text{mean}(\mathtt{V}_0) + \sum_{\ell=1}^L \text{mean}(\mathtt{dV}_\ell)$

\State Compute $S_Y^2 \gets \text{sample\_var}(\mathtt{Y}_0)/M_0 + \sum_{\ell=1}^L \text{sample\_var}(\mathtt{dY}_\ell)/M_\ell$
\State Compute $S_V^2 \gets \text{sample\_var}(\mathtt{V}_0)/M_0 + \sum_{\ell=1}^L \text{sample\_var}(\mathtt{dV}_\ell)/M_\ell$
\end{algorithmic}
\end{algorithm}

\subsection{Cost model}\label{subsec:mlmc:cost_model}
To compare single-fidelity and multilevel strategies we require a cost
model. Throughout we measure cost by the number of stochastic forward
passes of the trained network, that is, by the total number of calls to
$\mathscr{D}(x;\theta)$ with independently drawn $\theta \sim \pi$.
This proxy is appropriate when the dominant expense is the network
evaluation itself and when overhead (such as reductions for computing
means and variances) is negligible in comparison.

\subsubsection{Uncoupled cost baseline}\label{subsubsec:mlmc:cost_uncoupled}
A useful baseline is the cost incurred if coarse and fine estimators were
generated independently when forming each increment. In that uncoupled
setting, level $0$ requires $M_0T_0$ evaluations, while for each $\ell\ge 1$
one would generate $T_{\ell}$ evaluations for the fine estimator, giving $M_\ell T_\ell$
evaluations at level $\ell$. Hence the total uncoupled cost is
\begin{equation}\label{eq:mlmc:cost_uncoupled_additive}
    c_{\mathrm{unc}}
    :=
    M_0 T_0 + \sum_{\ell=1}^{L} M_\ell T_\ell,
\end{equation}

\subsubsection{Effective cost under mask reuse}\label{subsubsec:mlmc:cost_coupled}
In the common regime $M_0\ge \cdots \ge M_L$, a coupled single $\ell$-level estimator draws $T_\ell$
dropout realisations and reuses the first $T_{\ell-1}$ of them to compute the
coarse estimator. Therefore a coupled level-$\ell$ replicate costs $T_\ell-T_{\ell-1}$
evaluations rather than $T_\ell$. The effective coupled cost is
\begin{equation}\label{eq:mlmc:cost_coupled_additive}
    c_{\mathrm{cpl}}
    =
    T_0 M_0 + \sum_{\ell=1}^{L} (T_\ell-T_{\ell-1}) M_\ell
    =
    \sum_{\ell=0}^{L-1} T_\ell (M_\ell-M_{\ell+1}) + T_L M_L
    .
\end{equation}
So that $c_{\mathrm{cpl}}\le c_{\mathrm{unc}}$.
For allocation we work with the coupled budget $c:=c_{\mathrm{cpl}}$.
\subsection{Optimal allocation across levels}\label{subsec:mlmc:allocation}
The MLMC estimators in Section~\ref{sec:mlmc_dropout_estimators} depend
on the choice of sample counts $\vec M=(M_0,\dots,M_L)$. In this
subsection we pose and solve the allocation problem: for a fixed
computational budget, choose $\vec M$ to minimise the sampling variance
of the target estimator. We work with the coupled cost $c:=c_{\mathrm{cpl}}$ in the form
\eqref{eq:mlmc:cost_coupled_additive}. The optimisation is carried out over
continuous $M_\ell$ and then rounded to integers in implementation.

\subsubsection{Allocation for the mean estimator}\label{subsubsec:mlmc:allocation_mean}
The expectation of the empirical sampling variance of the MLMC mean estimator is given explicitly
by Lemma~\ref{lem:mlmc:mean_moments}. Writing
\begin{equation}\label{eq:mlmc:allocation_mean_level_vars}
    v_0(x) := \mathrm{Var}\big[Y(x;T_0)\big] = \frac{\mu_2(x)}{T_0},
    \qquad
    v_\ell(x) := \mathrm{Var}\big[\Delta Y(x;T_\ell)\big]
    = \mu_2(x)\left(\frac{1}{T_{\ell-1}}-\frac{1}{T_\ell}\right),\ \ell\ge 1,
\end{equation}
we have that the expectation of the empirical sampling
variance of the expectation estimator is given by
\begin{equation}\label{eq:mlmc:allocation_mean_objective}
    \mathbb{E}[S^2_Y(x;\vec M,\vec T)]
    =
    \frac{v_0(x)}{M_0} + \sum_{\ell=1}^{L}\frac{v_\ell(x)}{M_\ell}.
\end{equation}
Minimising \eqref{eq:mlmc:allocation_mean_objective} subject to the
budget constraint $c=T_0M_0+\sum_{\ell=1}^L (T_\ell-T_{\ell-1})M_\ell$
yields the allocation in Lemma~\ref{lem:mlmc:allocation_mean}.

\begin{lemma}[Optimal allocation for the mean estimator]\label{lem:mlmc:allocation_mean}
Fix $x$ and a strictly increasing fidelity ladder $2\le T_0<\cdots<T_L$.
Let $a_0:=T_0$ and $a_\ell:=T_\ell-T_{\ell-1}$ for $\ell\ge 1$, and let $c>0$ be the
coupled-cost budget $c=\sum_{\ell=0}^L a_\ell M_\ell$.
Among continuous choices $M_\ell>0$ satisfying this budget constraint,
the sampling variance $\mathbb{E}[S^2_Y(x;\vec M,\vec T)]$ is minimised by
\begin{equation}\label{eq:mlmc:allocation_mean_solution}
    M_\ell^{\star}
    =
    c \frac{\sqrt{v_\ell(x)/a_\ell}}{\sum_{k=0}^{L}\sqrt{v_k(x)a_k}},
    \qquad \ell=0,\dots,L.
\end{equation}
Equivalently, substituting the explicit $v_\ell$ and $a_\ell$ yields
\begin{equation}
    M_0^\star=\frac{c}{T_0}\left(
        1+\sum_{\ell=1}^{L}\frac{T_\ell-T_{\ell-1}}{\sqrt{T_\ell T_{\ell-1}}}
    \right)^{-1},
    \qquad
    M_i^\star=\frac{c}{\sqrt{T_i T_{i-1}}}\left(
        1+\sum_{\ell=1}^{L}\frac{T_\ell-T_{\ell-1}}{\sqrt{T_\ell T_{\ell-1}}}
    \right)^{-1},
    \quad i=1,\dots,L.
\end{equation}
\end{lemma}

\subsubsection{Allocation for the variance estimator}\label{subsubsec:mlmc:allocation_variance}
The sampling variance of the MLMC variance estimator is given by
Theorem~\ref{thm:mlmc:variance_moments}. Writing
\begin{equation}\label{eq:mlmc:allocation_variance_level_vars_general}
    w_0(x) := \mathrm{Var}\big[V(x;T_0)\big],
    \qquad
    w_\ell(x) := \mathrm{Var}\big[\Delta V(x;T_\ell)\big],\ \ell\ge 1,
\end{equation}
we may express the objective as
\begin{equation}\label{eq:mlmc:allocation_variance_objective}
    \mathbb{E}[S^2_V(x;\vec M,\vec T)]
    =
    \frac{w_0(x)}{M_0} + \sum_{\ell=1}^{L}\frac{w_\ell(x)}{M_\ell},
\end{equation}
under the same budget constraint. In general $w_\ell(x)$ depends on
$\mu_2(x)$ and $\mu_4(x)$, so in practice one either estimates
$\{w_\ell(x)\}$ from pilot runs or uses a moment closure such as
$\mu_4(x)=3\mu_2^2(x)$, equation \eqref{eq:mlmc:variance_level_vars_kappa0}. The corresponding optimal
allocations is given in Lemma~\ref{lem:mlmc:allocation_variance_general}.

\begin{lemma}[Optimal allocation for the variance estimator]\label{lem:mlmc:allocation_variance_general}
Fix $x$ and a fidelity ladder $2\le T_0<\cdots<T_L$. Let $a_0:=T_0$ and
$a_\ell:=T_\ell-T_{\ell-1}$ for $\ell\ge 1$, and let $c>0$ be the coupled-cost
budget $c=\sum_{\ell=0}^L a_\ell M_\ell$. Among continuous choices $M_\ell>0$ satisfying the budget constraint, the expectation of the empirical sampling
variance of the variance estimator,
\begin{equation}
    \mathbb{E}\big[S^2_V(x;\vec M,\vec T)\big]
    =
    \frac{w_0(x)}{M_0}+\sum_{\ell=1}^L\frac{w_\ell(x)}{M_\ell}
\end{equation}
is minimised by
\begin{equation}\label{eq:mlmc:allocation_variance_general_solution}
    M_\ell^{\star}
    =
    c \frac{\sqrt{w_\ell(x)/a_\ell}}{\sum_{k=0}^{L}\sqrt{w_k(x)a_k}},
    \qquad \ell=0,\dots,L.
\end{equation}
\end{lemma}

\subsection{Choosing the fidelity ladder}\label{subsec:mlmc:fidelity_ladder}
The multilevel constructions require a strictly increasing sequence of
fidelities $\vec T=(T_0,\dots,T_L)$. In our setting $T_\ell$ is the
number of stochastic forward passes used inside the single-replicate
estimators $Y(x,T_\ell)$ and $V(x,T_\ell)$, so the ladder controls the
variance--cost tradeoff through the coupled increment variances in
Lemmas~\ref{lem:mlmc:mean_moments} and Theorem~\ref{thm:mlmc:variance_moments}.
We discuss practical choices of $\vec T$ that are consistent with the
couplings in Sections~\ref{subsubsec:mlmc:mean_coupling} and
\ref{subsubsec:mlmc:variance_coupling}.

A convenient default is a geometric ladder
\begin{equation}\label{eq:mlmc:fidelity_geometric}
    T_\ell := \left\lceil T_0 r^\ell \right\rceil,
    \qquad \ell=0,\dots,L,
\end{equation}
with a fixed ratio $r>1$ and smallest fidelity $T_0$. This choice keeps
neighbouring fidelities well separated while limiting the number of
levels. In the present work we estimate both the predictive mean and
the predictive variance, so we require $T_0\ge 2$ in order for the
unbiased sample variance $V(x,T_0)$ to be well-defined. Moreover, the
variance coupling \eqref{eq:mlmc:variance_update} involves the sample
variance of the \emph{new} block $V_{T_{\ell-1}+1:T_\ell}(x)$, hence we
also impose
\begin{equation}\label{eq:mlmc:fidelity_blocksize_constraint}
    T_\ell-T_{\ell-1} \ge 2,
    \qquad \ell=1,\dots,L,
\end{equation}
so that each refinement step introduces at least two new dropout draws.
A simple and robust special case is the dyadic ladder
\begin{equation}\label{eq:mlmc:fidelity_dyadic}
    T_0 = 2,
    \qquad
    T_\ell = 2^{\ell+1}\quad(\ell=0,\dots,L),
\end{equation}
which satisfies \eqref{eq:mlmc:fidelity_blocksize_constraint} exactly.

Given $T_0$ and $r$, the final level $L$ is selected to reach a desired
maximum inner fidelity $T_{\max}$, for example the number of dropout
passes one would use in a conventional MC-dropout baseline. Concretely,
we take
\begin{equation}\label{eq:mlmc:fidelity_choose_L}
    L := \max\{\ell \in \mathbb{N}_0 : T_\ell \le T_{\max}\},
\end{equation}
so that $T_L \le T_{\max} < T_{L+1}$. In applications $T_{\max}$ is set
by an available evaluation budget per input or by a target error level
for the inner estimators $Y(x,T)$ and $V(x,T)$.

The choice of the ratio $r$ is primarily a practical modelling choice.
If $r$ is too close to $1$ then $L$ becomes large, which increases
bookkeeping and may lead to levels with very small $M_\ell$ after
allocation (in particular, if one wishes to estimate the sampling
variances via $S_Y^2$ and $S_V^2$, one typically requires $M_\ell\ge 2$
to form a non-degenerate sample variance). If $r$ is too large then
refinements add many new samples at once and the ladder behaves more
like a small number of coarse-to-fine jumps. In the numerical results
below we therefore adopt either the dyadic ladder
\eqref{eq:mlmc:fidelity_dyadic} or the geometric ladder
\eqref{eq:mlmc:fidelity_geometric} with a moderate ratio (for example,
$r=2$), which respects \eqref{eq:mlmc:fidelity_blocksize_constraint} and
keeps $L$ small.

Finally, the ladder may also be selected adaptively using pilot samples.
A standard stopping criterion is to increase $L$ until the empirical
contribution of the finest correction is negligible compared to the
current estimator variance, for example until both
\begin{equation}\label{eq:mlmc:fidelity_stopping_criterion}
    \frac{\mathcal{S}^2_{\Delta Y}(x;M_L,T_L)}{M_L}
    \ll
    S_Y^2(x;\vec M,\vec T),
    \qquad
    \frac{\mathcal{S}^2_{\Delta V}(x;M_L,T_L)}{M_L}
    \ll
    S_V^2(x;\vec M,\vec T),
\end{equation}
so that further refinement would not materially reduce the sampling
variability at the current budget. This criterion is consistent with
the variance decompositions in Lemma~\ref{lem:mlmc:mean_moments} and
Theorem~\ref{thm:mlmc:variance_moments}, and it naturally couples the
choice of $L$ with the allocation procedure in
Section~\ref{subsec:mlmc:allocation}.

\section{Numerical experiments}\label{sec:numerical_experiments}
This section evaluates the proposed multilevel estimators on two
benchmark PINNs--Uzawa problems. The aim is to (i) verify the sampling
theory of Sections~\ref{sec:mlmc_dropout_estimators}--\ref{sec:mlmc_dropout_computation},
(ii) quantify the variance--cost tradeoff of MLMC relative to
single-fidelity MC-dropout, and (iii) test the allocation formulas in
Section~\ref{subsec:mlmc:allocation} against brute-force enumeration
under fixed cost.

\subsection{Experimental setup and common definitions}\label{subsec:numerics:setup}
All numerical results in this section are produced by repeated
test-time evaluation of a fixed, trained surrogate network with dropout
kept active. A single \emph{network evaluation} is one stochastic
forward pass $x \mapsto \mathscr{D}(x;\theta)$ with an independently
resampled realisation $\theta\sim\pi$ (Section~\ref{subsec:mlmc:random_predictor}).
All reported estimators are computed pointwise in $x$, and any norms or
regression fits are subsequently applied to these pointwise quantities.
Unless stated otherwise, the same implementation is used for both
single-fidelity Monte Carlo and multilevel Monte Carlo, with coupling
implemented within each level increment as in
Algorithm~\ref{alg:mlmc:coupled_sampling}.

\subsubsection{Estimators and norms reported}\label{subsubsec:numerics:estimators_norms}
Fix an evaluation grid $\{x_i\}_{i=1}^{N_x}\subset(0,1)$. For a chosen
inner fidelity $T\ge 2$ and outer sample count $M\ge 1$, we compute the
single-fidelity estimators
\[
Y^{(m)}(x_i,T)
=
\frac{1}{T}\sum_{t=1}^{T}\mathscr{D}(x_i;\theta_{t}^{(m)}),
\qquad
V^{(m)}(x_i,T)
=
\frac{1}{T-1}\sum_{t=1}^{T}\big(\mathscr{D}(x_i;\theta_{t}^{(m)})-Y^{(m)}(x_i,T)\big)^2,
\]
with $\theta_{t}^{(m)}\overset{\mathrm{iid}}{\sim}\pi$, and then form the
outer averages
\[
\overline{Y}(x_i;M,T)=\frac{1}{M}\sum_{m=1}^{M}Y^{(m)}(x_i,T),
\qquad
\overline{V}(x_i;M,T)=\frac{1}{M}\sum_{m=1}^{M}V^{(m)}(x_i,T),
\]
together with their empirical outer-variance estimators
\begin{equation}
    \begin{split}
        \mathcal{S}^2_Y(x_i;M,T)
=&
\frac{1}{M-1}\sum_{m=1}^{M}\big(Y^{(m)}(x_i,T)-\overline{Y}(x_i;M,T)\big)^2,
\\
\mathcal{S}^2_V(x_i;M,T)
=&
\frac{1}{M-1}\sum_{m=1}^{M}\big(V^{(m)}(x_i,T)-\overline{V}(x_i;M,T)\big)^2.
    \end{split}
\end{equation}
For multilevel experiments we use a fidelity ladder $\vec T=(T_0,\dots,T_L)$
and sample counts $\vec M=(M_0,\dots,M_L)$, and compute the coupled MLMC
estimators $\mathcal{Y}(x_i;\vec M,\vec T)$ and $\mathcal{V}(x_i;\vec M,\vec T)$
and the associated empirical variance estimators
$S_Y^2(x_i;\vec M,\vec T)$ and $S_V^2(x_i;\vec M,\vec T)$ as defined in
Section~\ref{sec:mlmc_dropout_estimators} and implemented in
Algorithm~\ref{alg:mlmc:coupled_sampling}.

When a scalar summary over $x$ is required (for example in variance--fidelity
plots or fixed-cost surfaces), we report the discrete $L^1$-proxy
\begin{equation}\label{eq:numerics:L1_proxy}
\|g\|_{L^1}
\approx
\sum_{i=1}^{N_x} |g(x_i)| \Delta x,
\qquad
\Delta x := x_{i+1}-x_i,
\end{equation}
with $\Delta x$ constant for a uniform grid. For log--log rate plots, we
fit a linear model to $\log \|g\|_{L^1}$ against $\log T$ by least squares
and report the estimated slope together with its confidence interval as
indicated in the corresponding figure captions.

\subsubsection{Cost accounting}\label{subsubsec:numerics:cost_accounting}
We measure computational cost by the total number of stochastic forward
passes of the trained network, that is, by the total number of calls to
$\mathscr{D}(x;\theta)$ with independently drawn $\theta\sim \pi$
(Section~\ref{subsec:mlmc:cost_model}). For a single-fidelity experiment
with parameters $(M,T)$ the cost is $MT$ evaluations per input point.

For multilevel experiments we distinguish an uncoupled baseline cost,
in which the coarse and fine estimators in each increment are generated
independently, denoted $c_{\mathrm{unc}}$, from the effective coupled cost in which the first
$T_{\ell-1}$ samples are reused inside the $T_\ell$-sample estimator, denoted $c_{\mathrm{cpl}}$. 

In the fixed-cost allocation studies, we enumerate integer tuples
$\vec M$ satisfying the chosen cost constraint (and any additional
feasibility conditions such as $M_\ell\ge 2$ when a sample variance at
level $\ell$ is required) and then compare empirical minima of the
resulting estimator-variance surfaces with the theoretical allocations
from Section~\ref{subsec:mlmc:allocation}. Unless explicitly stated
otherwise, all ``fixed cost'' multilevel experiments hold for
$c_{\mathrm{cpl}}$, equation \eqref{eq:mlmc:cost_coupled_additive}, fixed.

\subsection{Forward-PINNs benchmark: boundary layer problem}\label{subsec:numerics:forward_pinn}
We begin with a forward model problem for which an explicit reference
solution is available. This allows us to (i) visualise MC-dropout
confidence bands, (ii) verify the single-fidelity sampling rates from
Section~\ref{sec:mlmc_dropout_estimators}, and (iii) test the fixed-cost
allocation predictions from Section~\ref{sec:mlmc_dropout_computation}.

\subsubsection{Problem definition and reference solution}\label{subsubsec:numerics:forward_problem}
For a parameter $\epsilon>0$, we consider the boundary-layer problem
\begin{equation}\label{eq:numerics:boundary_layer_problem}
    u(x) - \epsilon^2 u''(x) = 1 \quad \text{for } x\in(0,1),
    \qquad
    u(0)=u(1)=0.
\end{equation}
The exact solution is
\begin{equation}\label{eq:numerics:boundary_layer_exact_solution}
    u(x)
    =
    1 - \dfrac{e^{x/\epsilon} + e^{(1-x)/\epsilon}}{1 + e^{1/\epsilon}}.
\end{equation}

\subsubsection{Training and architecture}\label{subsubsec:numerics:forward_training}
We train a surrogate network $\mathscr{D}$ using a PINNs--Uzawa setup for the
boundary-layer problem \eqref{eq:numerics:boundary_layer_problem}, with dropout
active in $L_d$ layers. At test time, the trained deterministic
weights are held fixed and the only source of randomness is from $\theta\sim \pi$, so each evaluation $\mathscr{D}(x;\theta)$
corresponds to one stochastic forward pass as in
Section~\ref{subsec:mlmc:random_predictor}. The training hyperparameters
used in this experiment are summarised in Table~\ref{tab:archtecture_1}.

\begin{table}[hbtp]
    \centering
    \begin{tabular}{|c|c|c|c|c|c|c|c|c|c|c|c|c|}
    \hline\hline
        $N_{\mathrm{in}}$
        & $N_{\mathrm{out}}$
        & $N_{\mathrm{width}}$
        & $L_d$
        & $L_h$
        & $p_{\mathrm{drop}}$
        & $N_{\mathrm{SGD}}$
        & $N_{\mathrm{Uz}}$
        & $N_{\mathrm{repeats}}$
        & $\gamma$
        & $\rho$
        & $\eta$
        & $\epsilon$
        \\
        \hline
        1
        & 1
        & 64
        & 3
        & 0
        & 0.1
        & 50,000
        & 50
        & 5
        & 100
        & 0.01
        & 0.5
        & 1.0
        \\
        \hline
    \end{tabular}
    \caption{Hyperparameters used to train the forward (boundary-layer) surrogate.
    A single stochastic forward pass corresponds to evaluating
    $\mathscr{D}(x;\theta)$ with a freshly drawn random variable $\theta \sim \pi$.
    The parameters $N_{\mathrm{in}}, N_{\mathrm{out}}, N_{\mathrm{width}}, L_d, L_h$, and
    $p_{\mathrm{drop}}$ are defined in the network architecture described in
    Section~\ref{subsec:mlmc:random_predictor}.
    The parameters $N_{\mathrm{SGD}}, N_{\mathrm{Uz}}$, and $N_{\mathrm{repeats}}$
    denote the total number of training epochs, the number of epochs per Uzawa update,
    and the number of repeats used to compute the PINNs term each epoch, respectively.
    The parameter $\gamma$ weights the boundary-condition term in the loss functional,
    $\rho$ is the Uzawa update parameter, and $\eta$ is the optimizer learning rate.}
    \label{tab:archtecture_1}
\end{table}
\subsubsection{Single-fidelity behaviour as inner fidelity increases}\label{subsubsec:numerics:forward_single_fidelity_T}
We first illustrate how the inner fidelity $T$ affects the single-fidelity
Monte Carlo estimators. For each $x$ we compute $Y(x,T)$ and $V(x,T)$ from
\eqref{eq:mlmc:single_fidelity_estimators} with $M=1$ outer replicate,
and plot the mean estimator together with the pointwise bands
$Y(x,T)\pm \sqrt{V(x,T)}$ and $Y(x,T)\pm 2\sqrt{V(x,T)}$. The results in
Figure~\ref{fig:MLMC_confidence_intervals} show that increasing $T$
reduces visible sampling noise and yields smoother uncertainty bands.

\begin{figure}[htbp]
\centering
\begin{subfigure}{0.4\textwidth}
    \includegraphics[width=\textwidth]{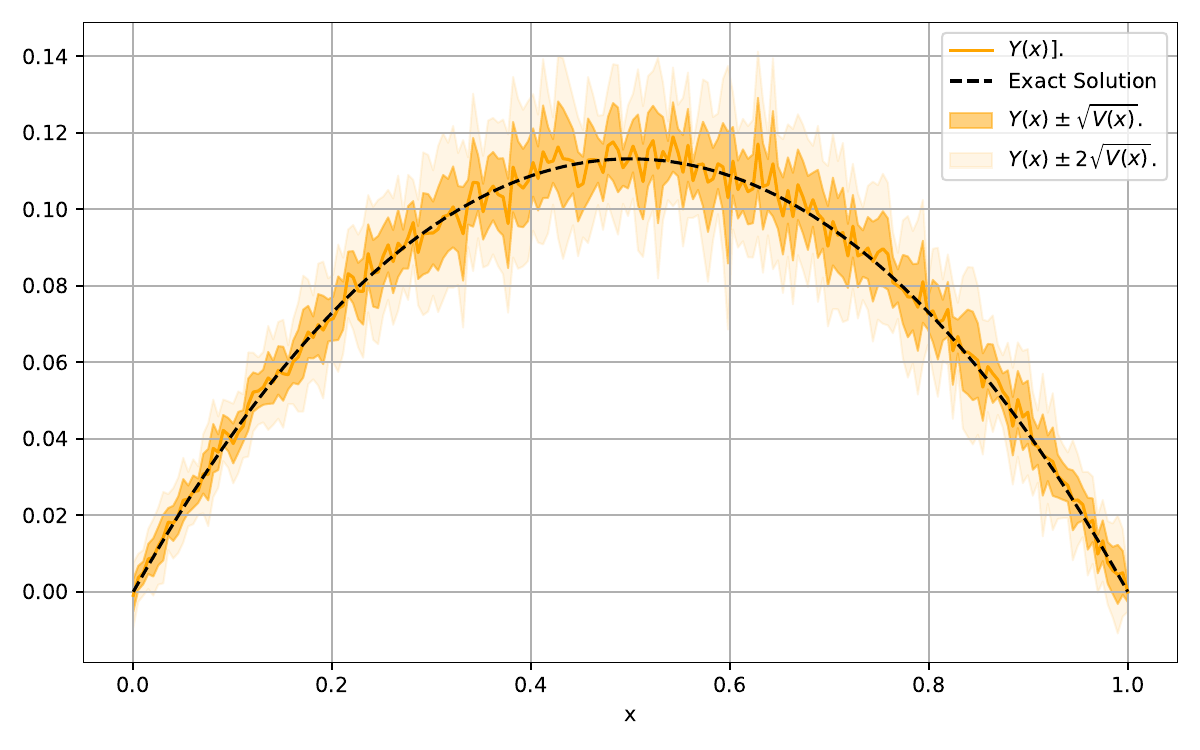}
    \caption{$T=10$.}
    \label{fig:confidence_intervals_10}
\end{subfigure}
\hfill
\begin{subfigure}{0.4\textwidth}
    \includegraphics[width=\textwidth]{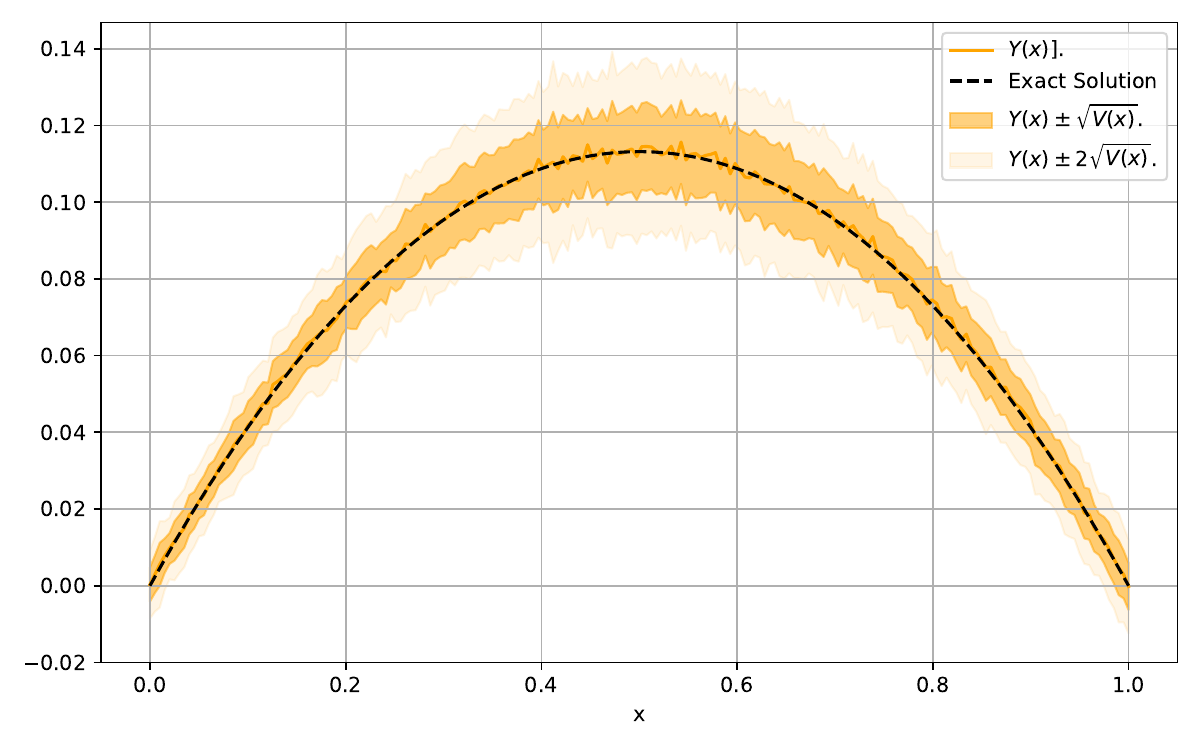}
    \caption{$T=100$.}
    \label{fig:confidence_intervals_100}
\end{subfigure}
\hfill
\begin{subfigure}{0.4\textwidth}
    \includegraphics[width=\textwidth]{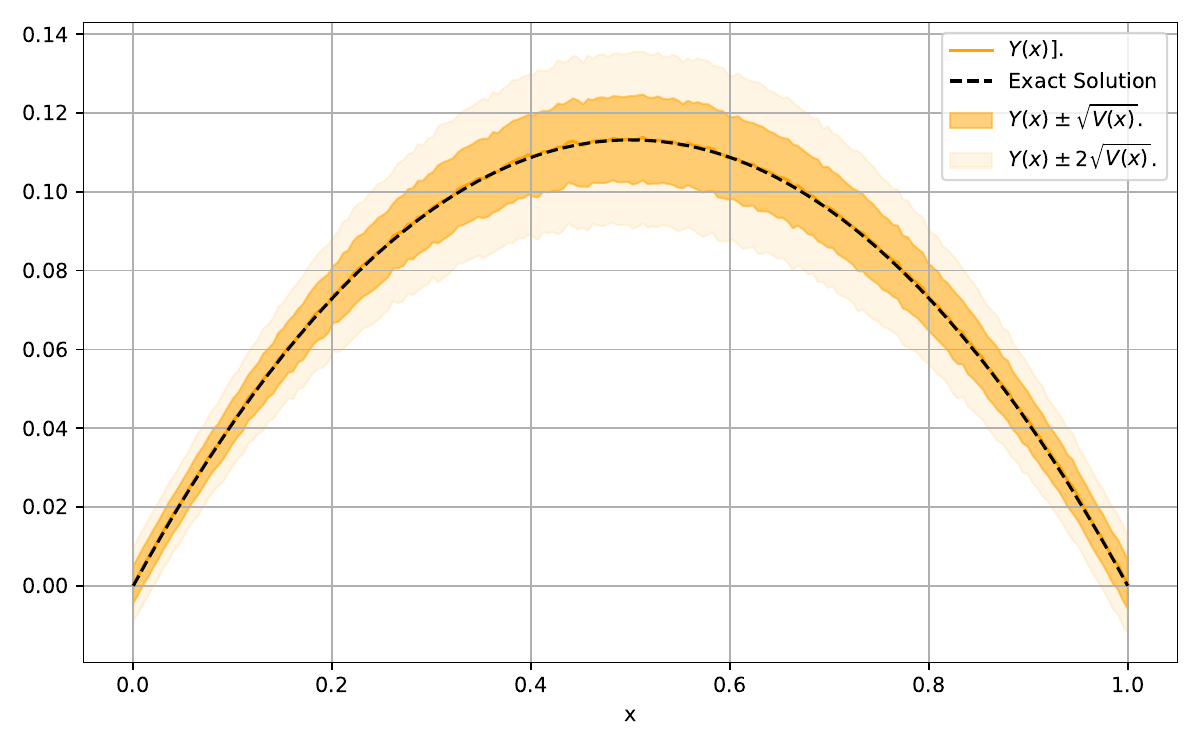}
    \caption{$T=1000$.}
    \label{fig:confidence_intervals_1000}
\end{subfigure}
\hfill
\begin{subfigure}{0.4\textwidth}
    \includegraphics[width=\textwidth]{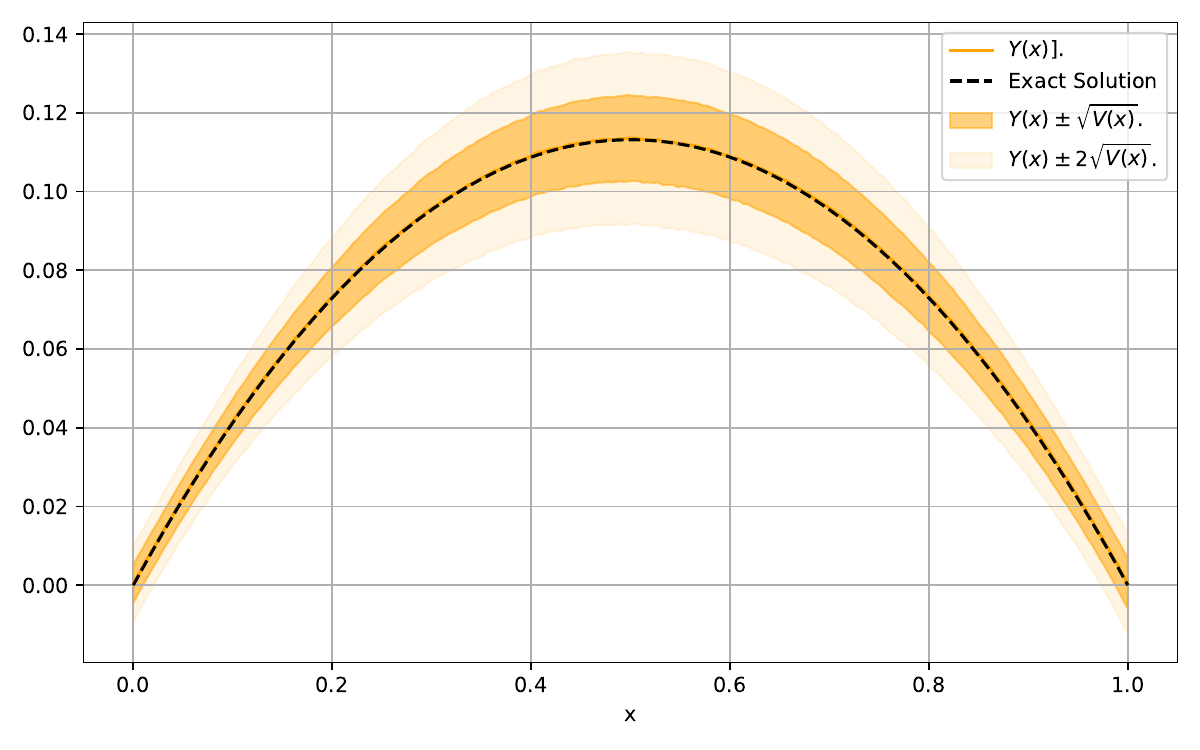}
    \caption{$T=10000$.}
    \label{fig:confidence_intervals_10000}
\end{subfigure}
\caption{Single-fidelity single-level MC-dropout estimators at increasing inner fidelity $T$.
Each subplot shows the estimated mean $Y(x,T)$ and the confidence bands
$Y(x,T)\pm \sqrt{V(x,T)}$ and $Y(x,T)\pm 2\sqrt{V(x,T)}$. The exact
solution \eqref{eq:numerics:boundary_layer_exact_solution} is shown as
a dashed curve.}
\label{fig:MLMC_confidence_intervals}
\end{figure}

\subsubsection{Empirical verification of sampling-variance rates}\label{subsubsec:numerics:forward_variance_rates}
We next verify the predicted $T^{-1}$ scaling of the sampling variances.
For each $T$ we generate $M_0=10$ independent outer replicates of the
single-fidelity estimators and compute the empirical estimator variances
$\mathcal{S}^2_Y(\cdot;M_0,T)$ and $\mathcal{S}^2_V(\cdot;M_0,T)$ from
\eqref{eq:mlmc:outer_empirical_estimators}. We report their discrete
$L^1$ norms over the evaluation grid as described in
Section~\ref{subsec:numerics:setup}.
Figure~\ref{fig:estimator_var_v_fid} shows approximate log--log slopes
close to $-1$, consistent with Lemma~\ref{lem:mlmc:single_fidelity_moments}.

\begin{figure}[htbp]
\centering
\begin{subfigure}{0.4\textwidth}
    \includegraphics[width=\textwidth]{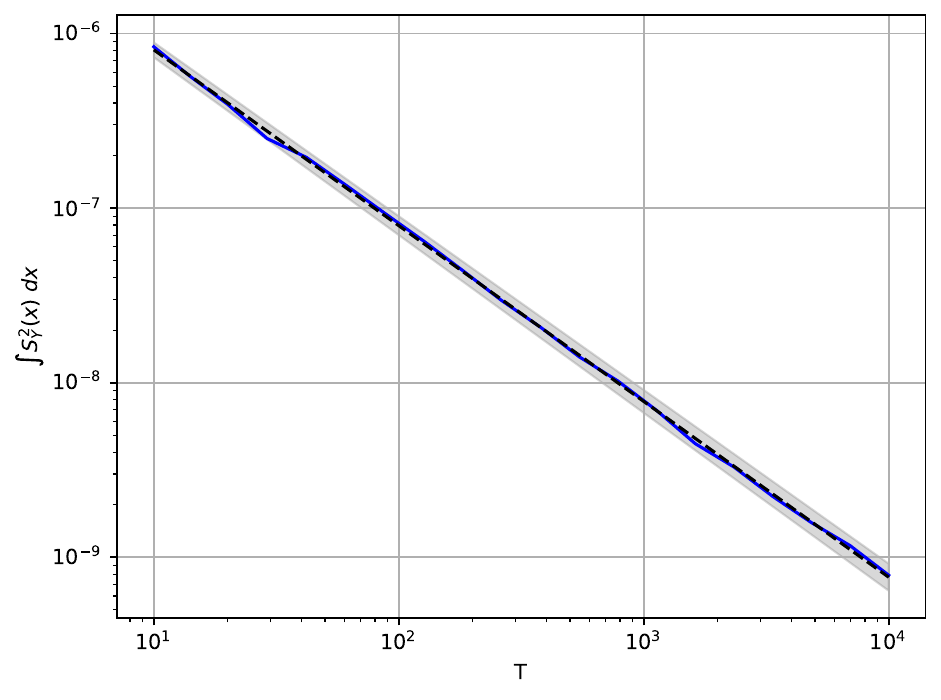}
    \caption{$\|\mathcal{S}^2_{Y}\|_{L^1}$ versus $T$.}
    \label{fig:exp_estimator_var_v_fid}
\end{subfigure}
\hfill
\begin{subfigure}{0.4\textwidth}
    \includegraphics[width=\textwidth]{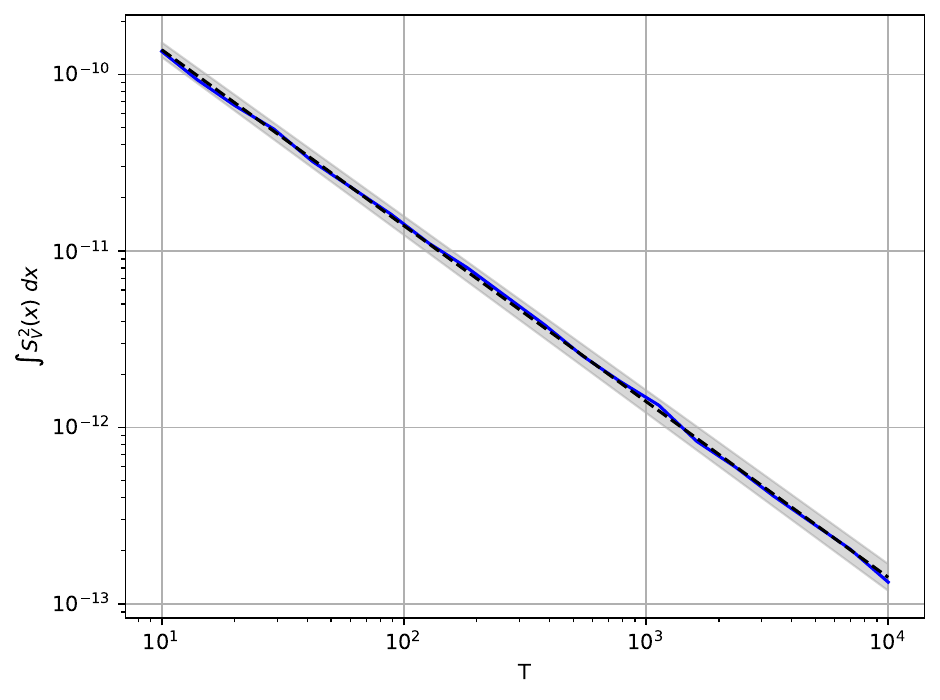}
    \caption{$\|\mathcal{S}^2_{V}\|_{L^1}$ versus $T$.}
    \label{fig:var_estimator_var_v_fid}
\end{subfigure}
\caption{Empirical estimator variances versus inner fidelity $T$ in the forward problem.
Each subplot reports the discrete $L^1$ norm of the empirical outer-sample variance
for the corresponding estimator, together with a log--log linear fit and a 99\% confidence
interval for the fitted slope.}
\label{fig:estimator_var_v_fid}
\end{figure}

\subsubsection{Sensitivity to dropout probability}\label{subsubsec:numerics:forward_pdrop_sweep}
Fixing $T_0=10^4$ and $M_0=10$, we study how the empirical estimator
variances depend on the dropout probability $p_{\mathrm{drop}}$.
Figure~\ref{fig:estimator_var_v_pdrop} shows that the variance decreases
towards zero as $p_{\mathrm{drop}}\to 0$ (deterministic evaluation) and
increases as $p_{\mathrm{drop}}\to 1$.
To parameterise the trend we fit the generalised logit model
\begin{equation}\label{eq:numerics:logit_model}
    \log \|S_{\cdot}^2\|_{L^1}
    \approx
    \alpha_1 + \alpha_2 \log p_{\mathrm{drop}}^{\alpha_4}
    - \alpha_3 \log \big(1- p_{\mathrm{drop}}^{\alpha_4}\big),
\end{equation}
and report fitted parameters and confidence intervals in
Table~\ref{tab:logit_vals}. This experiment is included as a secondary
diagnostic and is not used in the MLMC constructions, where fidelity is
controlled by $T$.

\begin{figure}[htbp]
\centering
\begin{subfigure}{0.4\textwidth}
    \includegraphics[width=\textwidth]{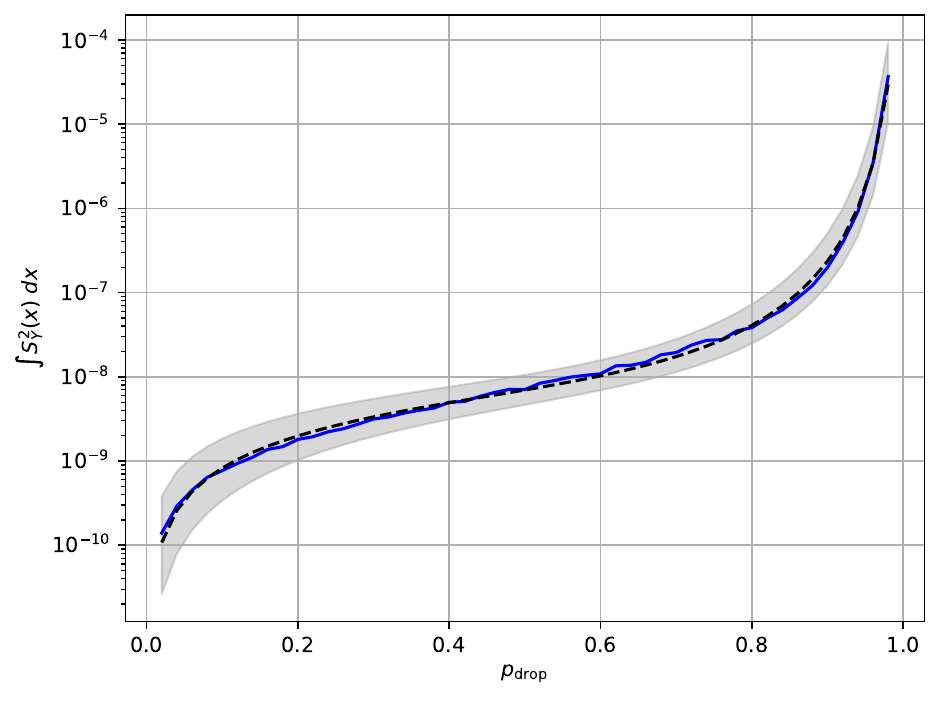}
    \caption{$\|\mathcal{S}^2_{Y}\|_{L^1}$ versus $p_{\mathrm{drop}}$.}
    \label{fig:exp_estimator_var_v_pdrop}
\end{subfigure}
\hfill
\begin{subfigure}{0.4\textwidth}
    \includegraphics[width=\textwidth]{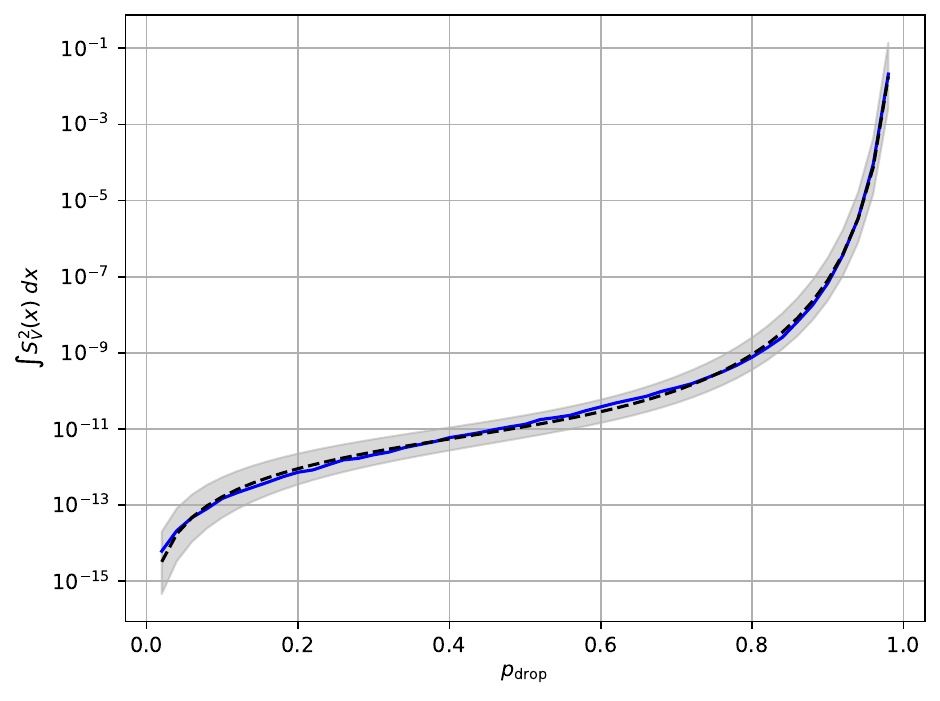}
    \caption{$\|\mathcal{S}^2_{V}\|_{L^1}$ versus $p_{\mathrm{drop}}$.}
    \label{fig:var_estimator_var_v_pdrop}
\end{subfigure}
\caption{Dependence of empirical estimator variances on dropout probability $p_{\mathrm{drop}}$
for the forward problem at fixed $T_0=10^4$ and $M_0=10$. The dashed curves show
the fitted model \eqref{eq:numerics:logit_model} with 99\% confidence bands.}
\label{fig:estimator_var_v_pdrop}
\end{figure}

\begin{table}[htbp]
    \centering
    \begin{tabular}{|c|c|c|c|c|}
    \hline\hline
    & $\alpha_1$ & $\alpha_2$ & $\alpha_3$ & $\alpha_4$
    \\
    \hline
    $Y$
    & -18.002 (-18.143, -17.862)
    & 0.250 (0.223, 0.277)
    & 3.259 (3.037, 3.481)
    & 5.059 (4.373, 5.744)
    \\
    $V$
    & -23.797 (-24.064, -23.531)
    & 0.513 (0.476, 0.550)
    & 8.307 (7.917, 8.697)
    & 4.766 (4.306, 5.226)
    \\
    \hline
    \end{tabular}
    \caption{Fitted parameters for the model \eqref{eq:numerics:logit_model} in the forward problem.
    Parentheses denote 99\% confidence intervals.}
    \label{tab:logit_vals}
\end{table}

\subsubsection{Fixed-cost multilevel allocation test}\label{subsubsec:numerics:forward_fixed_cost_allocation}
We now test the multilevel allocation theory under a fixed coupled-cost
budget. We consider the three-level ladder
\begin{equation}\label{eq:numerics:forward_fixed_cost_ladder}
    L=2,
    \qquad
    \vec T = (T_0,T_1,T_2) = (4,8,16),
    \qquad
    c_{\mathrm{cpl}} = 1000.
\end{equation}
For each integer allocation $\vec M=(M_0,M_1,M_2)$ with $M_\ell\ge 2$ that
satisfies the coupled-cost constraint (Subsubsection~\ref{subsubsec:numerics:cost_accounting})
\begin{equation}\label{eq:numerics:forward_fixed_cost_constraint}
    T_0 M_0 + (T_1-T_0) M_1 + (T_2-T_1) M_2 = c_{\mathrm{cpl}},
\end{equation}
we compute the empirical MLMC variance estimators $S_Y^2(\cdot;\vec
M,\vec T)$ and $S_V^2(\cdot;\vec M,\vec T)$ from
\eqref{eq:mlmc:mean_mlmc_variance_estimator} and
\eqref{eq:mlmc:variance_mlmc_variance_estimator} using
Algorithm~\ref{alg:mlmc:coupled_sampling}. Since
\eqref{eq:numerics:forward_fixed_cost_constraint} leaves two degrees
of freedom, we parameterise feasible allocations by $(M_1,M_2)$, with
$M_0$ determined by the budget, and visualise the reciprocal surfaces
$1/\|S_Y^2\|_{L^1}$ and $1/\|S_V^2\|_{L^1}$ to highlight
minima. Figure~\ref{fig:var_estimator_fixed_cost} compares empirically
best allocations with the continuous optimal allocations derived in
Section~\ref{subsec:mlmc:allocation}, namely
Lemma~\ref{lem:mlmc:allocation_mean} for the mean estimator and
Lemma~\ref{lem:mlmc:allocation_variance_general} for the variance
estimator (the latter under the zero-excess-kurtosis closure, equation
\eqref{eq:mlmc:variance_level_vars_kappa0}, used in its derivation).

For the ladder $\vec T=(4,8,16)$ and cost $c_{\mathrm{cpl}}=1000$, the continuous
optimal allocations predicted by the theory are independent of $x$ for the mean
estimator, and (under the zero-excess-kurtosis closure) independent of the
underlying distribution up to an overall factor for the variance estimator. In
particular, the mean-allocation formula yields
$(M_0^\star,M_1^\star,M_2^\star)=(1000/12, 1000/24, 1000/48)$, while the
variance-allocation formula under the closure yields
$(M_0^\star,M_1^\star,M_2^\star)\approx(82.64, 44.17, 19.76)$.

\begin{figure}[htbp]
\centering
\begin{subfigure}{0.4\textwidth}
    \includegraphics[width=\textwidth]{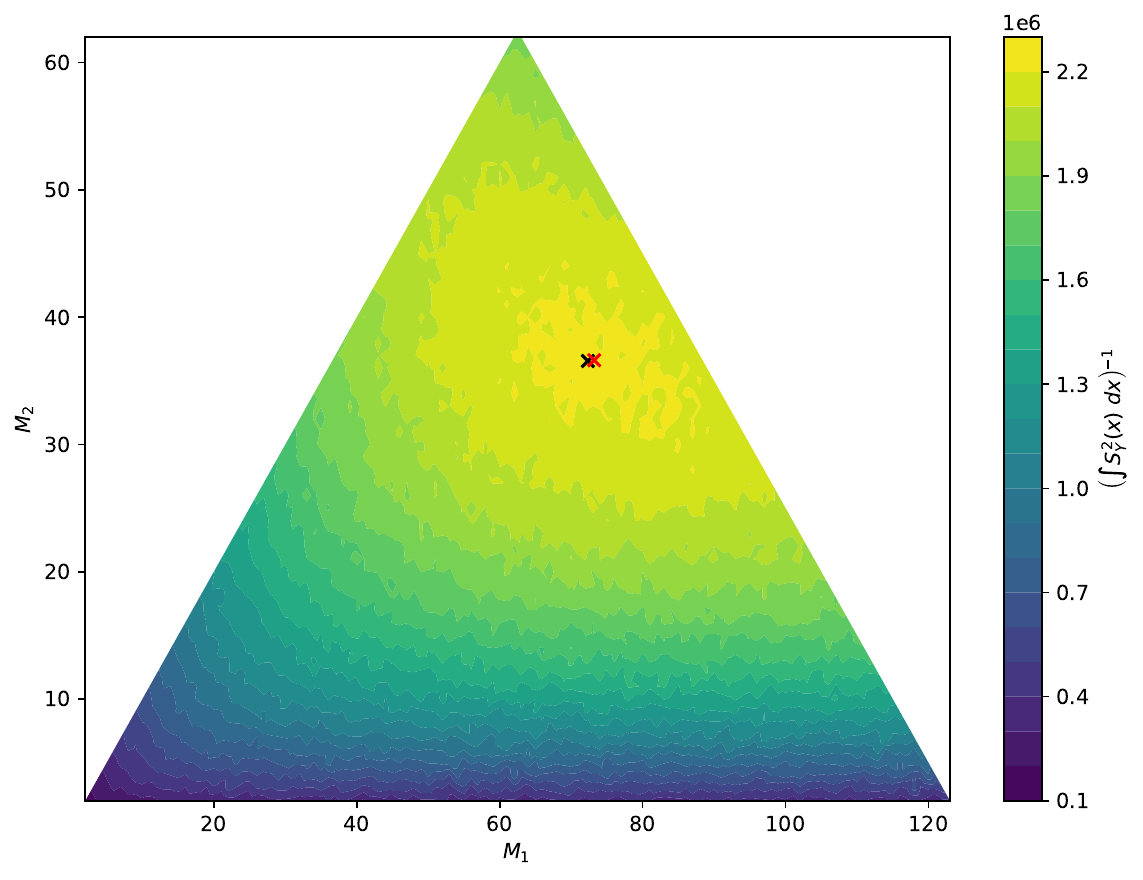}
    \caption{$1/\|S_Y^2\|_{L^1}$ as a function of $(M_1,M_2)$.}
    \label{fig:var_exp_estimator_fixed_cost}
\end{subfigure}
\hfill
\begin{subfigure}{0.4\textwidth}
    \includegraphics[width=\textwidth]{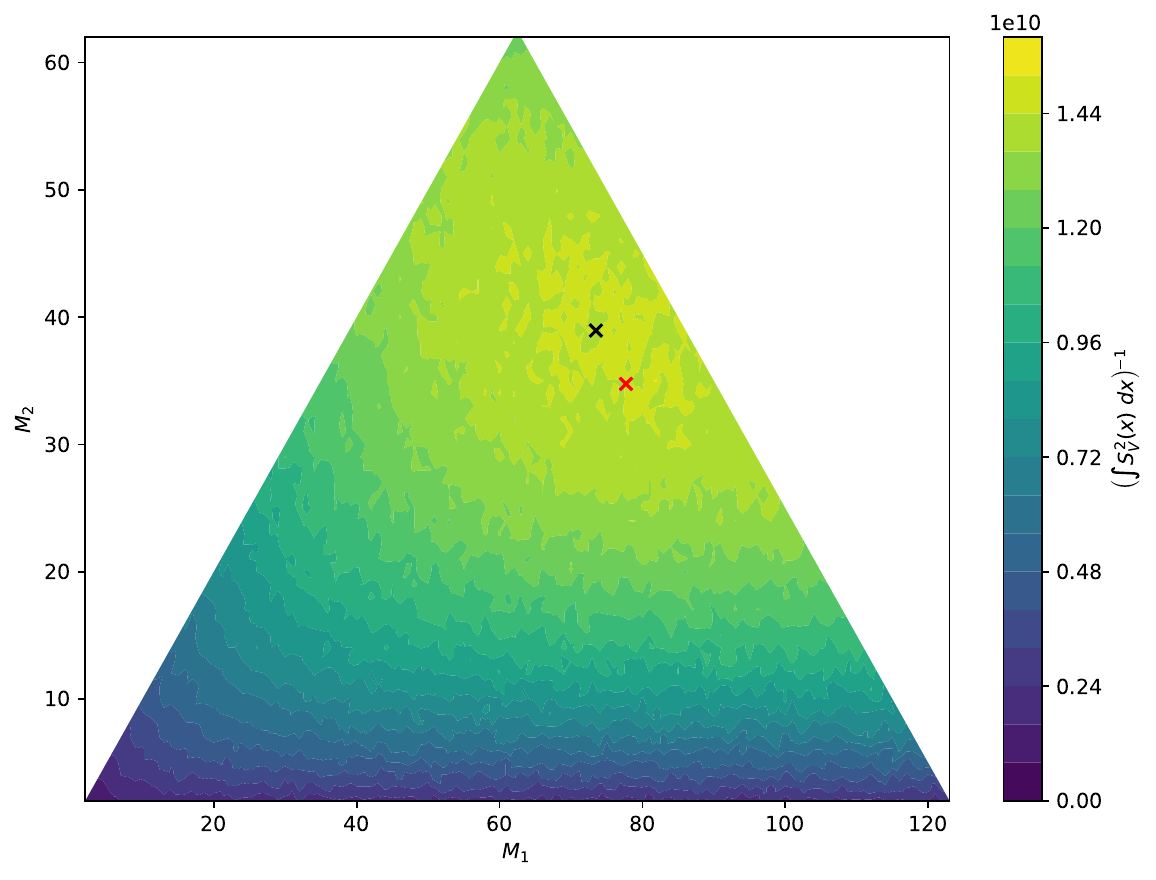}
    \caption{$1/\|S_V^2\|_{L^1}$ as a function of $(M_1,M_2)$.}
    \label{fig:var_var_estimator_fixed_cost}
\end{subfigure}
\caption{Fixed-cost allocation test for the forward problem with
$\vec T=(4,8,16)$ and $c_{\mathrm{cpl}}=1000$.
The black marker indicates an empirical minimiser over feasible integer allocations
and the red marker indicates the continuous optimum from
Lemma~\ref{lem:mlmc:allocation_mean} (left) or
Lemma~\ref{lem:mlmc:allocation_variance_general} (right), noting that the
variance-allocation formula is derived under the zero-excess-kurtosis closure.}
\label{fig:var_estimator_fixed_cost}
\end{figure}

\subsection{Inverse-PINNs benchmark: stochastic target source inference}\label{subsec:numerics:inverse_pinn}
In this subsection we test the multilevel estimators on an inverse PINNs--Uzawa
problem in which the training target is stochastic. This setting provides a
closed-form reference for the target moments, and it lets us separate (i) Monte
Carlo sampling error in estimating dropout-induced moments from (ii) surrogate
mismatch due to imperfect training.

\subsubsection{Problem definition and ground-truth moments}\label{subsubsec:numerics:inverse_problem}
We consider the following constrained optimisation problem: given $\alpha>0$ and
a (possibly random) target $u_{\rm target}\in L^2(0,1)$, minimise
\begin{equation}\label{eq:numerics:inverse_objective}
    \frac{1}{2} \|u-u_{\rm target}\|_{L^2(0,1)}^2 + \frac{\alpha}{2} \|f\|_{L^2(0,1)}^2
\end{equation}
subject to the PDE constraint
\begin{equation}\label{eq:numerics:inverse_constraint}
    -u''(x)=f(x)\ \text{ on }(0,1),\qquad u(0)=u(1)=f(0)=f(1)=0.
\end{equation}
For a synthetic benchmark with known moments, we generate stochasticity through
a scalar noise variable
\begin{equation}\label{eq:numerics:inverse_target_family}
    u_{\rm target}(x) = (1+\omega)(1 + \alpha \pi^4)\sin(\pi x), \quad \omega \sim \mathrm{Unif}(-\delta/2,\delta/2),
\end{equation}
and we use the closed-form family
\begin{equation}\label{eq:numerics:inverse_closed_form}
    u(x;\omega)=(1+\omega)\sin(\pi x),
    \qquad
    f(x;\omega)=(1+\omega)\pi^2\sin(\pi x).
\end{equation}
The corresponding ground-truth moments are
\begin{equation}\label{eq:numerics:inverse_ground_truth_moments}
    \mathbb{E}[u(x)]=\sin(\pi x),
    \qquad
    \mathbb{E}[f(x)]=\pi^2\sin(\pi x),
\end{equation}
and, using $\mathrm{Var}(\omega)=\delta^2/12$,
\begin{equation}\label{eq:numerics:inverse_ground_truth_vars}
    \mathrm{Var}[u(x)]=\frac{\delta^2}{12}\sin^2(\pi x),
    \qquad
    \mathrm{Var}[f(x)]=\frac{\delta^2}{12}\pi^4\sin^2(\pi x).
\end{equation}
These expressions provide a reference for assessing the empirical estimators of
the dropout-induced moments computed as described in
Subsection~\ref{subsec:numerics:setup}.

\subsubsection{Surrogate model, architecture and hyperparameters}\label{subsubsec:numerics:inverse_training}
We train a two-output dropout network $\mathscr{D}=(\mathscr{U},\mathscr{F})$ to
approximate the solution pair $(u,f)$ associated with
\eqref{eq:numerics:inverse_objective}--\eqref{eq:numerics:inverse_constraint}
under the stochastic target \eqref{eq:numerics:inverse_target_family}.
The architecture uses a shared trunk followed by branched heads for
$\mathscr{U}$ and $\mathscr{F}$, and we enforce the homogeneous Dirichlet boundary
conditions by an output factor $x(1-x)$.
A schematic is shown in Figure~\ref{fig:tikz_shared_net}, and the training
hyperparameters are listed in Table~\ref{tab:archtecture_2}.

At evaluation we hold the trained weights fixed and keep dropout active, so that
each forward pass draws a fresh realisation $\theta\sim \pi$ and
produces a stochastic prediction $\mathscr{D}(x;\theta)$. The estimators
$\mathcal{Y}$, $\mathcal{V}$, $S_Y^2$ and $S_V^2$ are then computed from repeated
stochastic forward passes exactly as described in
Section~\ref{sec:mlmc_dropout_estimators} and implemented in
Algorithm~\ref{alg:mlmc:coupled_sampling}.

\begin{figure}[hbtp]
  \centering
  \begin{tikzpicture}[
    neuron/.style={circle, draw, minimum size=8mm, fill=blue!10},
    dropout/.style={circle, draw, dashed, minimum size=8mm, fill=red!10},
    arrow/.style={-{Latex}, thick},
    every label/.append style={font=\footnotesize}
]

\node[neuron, label=above:Input] (x) at (0,0) {$x$};

\node[neuron] (s1) at (2,0.5) {};
\node[neuron] (s2) at (2,-0.5) {};

\node[neuron] (s5) at (4,0.5) {};
\node[neuron] (s6) at (4,-0.5) {};

\foreach \a in {s1,s2} {
    \draw[arrow] (x) -- (\a);
}

\foreach \a in {s1,s2} {
    \foreach \b in {s5,s6} {
        \draw[arrow] (\a) -- (\b);
    }
}


\node[neuron] (u0) at (6,1.5) {};
\node[neuron] (u1) at (6,0.5) {};

\node[neuron] (u2) at (8,1.5) {};
\node[neuron] (u3) at (8,0.5) {};

\node[neuron] (uout) at (10,1) {$\mathscr{U}(x)$};

\foreach \a in {s5,s6} {
    \foreach \b in {u0,u1} {
        \draw[arrow] (\a) -- (\b);
    }
}

\foreach \a in {u0,u1} {
    \foreach \b in {u2,u3} {
        \draw[arrow] (\a) -- (\b);
    }
}

\foreach \a in {u2,u3} {
    \draw[arrow] (\a) -- (uout);
}


\node[neuron] (f0) at (6,-0.5) {};
\node[neuron] (f1) at (6,-1.5) {};

\node[neuron] (f2) at (8,-0.5) {};
\node[neuron] (f3) at (8,-1.5) {};

\node[neuron] (fout) at (10,-1) {$\mathscr{F}(x)$};

\foreach \a in {s5,s6} {
    \foreach \b in {f0,f1} {
        \draw[arrow] (\a) -- (\b);
    }
}

\foreach \a in {f0,f1} {
    \foreach \b in {f2,f3} {
        \draw[arrow] (\a) -- (\b);
    }
}

\foreach \a in {f2,f3} {
    \draw[arrow] (\a) -- (fout);
}

\node at ($(s1)!0.5!(s6) + (-0,1.25)$) {$\overbrace{\hspace{3.0cm}}^{\text{Shared layers}}$};

\node at ($(u0)!0.5!(f3) + (-0,2.25)$) {$\overbrace{\hspace{3.0cm}}^{\text{Individual layers}}$};

\end{tikzpicture}
  \caption{Schematic of the surrogate network architecture for the inverse
  benchmark. The network has a shared trunk of dropout layers followed by
  branched heads for $\mathscr{U}$ and $\mathscr{F}$. A boundary-enforcing output
  factor $x(1-x)$ is applied to each head.}
  \label{fig:tikz_shared_net}
\end{figure}
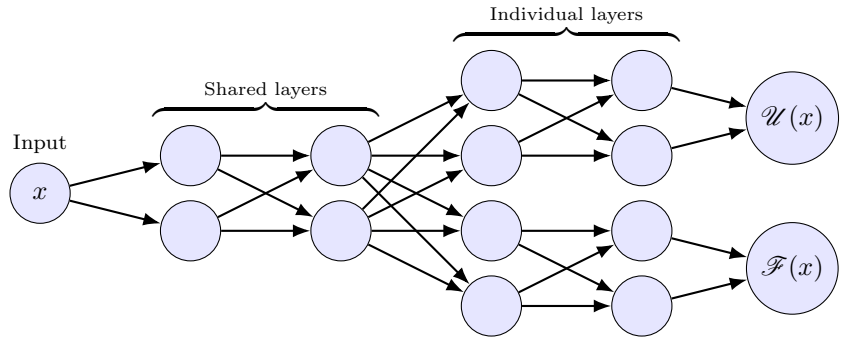

\begin{table}[htbp]
    \centering
    \begin{tabular}{|c|c|c|c|c|c|c|c|}
    \hline\hline
    $N_{\mathrm{in}}$
    & $N_{\mathrm{out}}$
    & $N_{\mathrm{width}}$
    & $L_d$
    & $L_{q,u}$
    & $L_{q,f}$
    & $p_{\mathrm{drop}}$
    & $N_{\mathrm{SGD}}$
    \\ \hline
    1
    & 2
    & 128
    & 4
    & 0
    & 1
    & 0.2
    & 200000
    \\ \hline \hline
    $N_{\mathrm{Uz}}$
    & $N_{\mathrm{repeats}}$
    & $N_{\mathrm{lag}}$
    & $\alpha$
    & $\beta$
    & $\eta$
    & $\rho$
    & $\delta$
    \\ \hline
    50
    & 5
    & 20
    & $10^{-4}$
    & $10^{-4}$
    & $2.5 \times 10^{-5}$
    & $10^{-3}$
    & $2.5 \times 10^{-2}$
    \\ \hline
    \end{tabular}
    \caption{Hyperparameters used to train the inverse benchmark surrogate
    $\mathscr{D}=(\mathscr{U},\mathscr{F})$. The architecture is the shared-trunk /
    branched-head dropout design and
    Figure~\ref{fig:tikz_shared_net}. Many of the hyperparamters share she same definition as in Table \ref{tab:archtecture_1}. The hyperparameter $\beta$ corresponds to the PINNs weight in the Loss functional, and $N_{\mathrm{lag}}$ is the number of evaluations per Uzawa update.}
    \label{tab:archtecture_2}
\end{table}

\subsubsection{Single-fidelity behaviour: predictive means}\label{subsubsec:numerics:inverse_single_fidelity_mean}
We first assess single-fidelity Monte Carlo estimation of the
dropout-induced predictive means. For each output we compute
$\overline{Y}(x;M,T)$ at a low inner fidelity $T=10$ and a high inner
fidelity $T=10^4$ and compare to the ground-truth expectations
\eqref{eq:numerics:inverse_ground_truth_moments}. Figure~\ref{fig:Inverse_exp_est}
shows the estimated expectations and the corresponding pointwise
errors.

\begin{figure}[bhtp]
    \centering
    \begin{subfigure}[t]{0.39\textwidth}
        \centering
        \includegraphics[width=\textwidth]{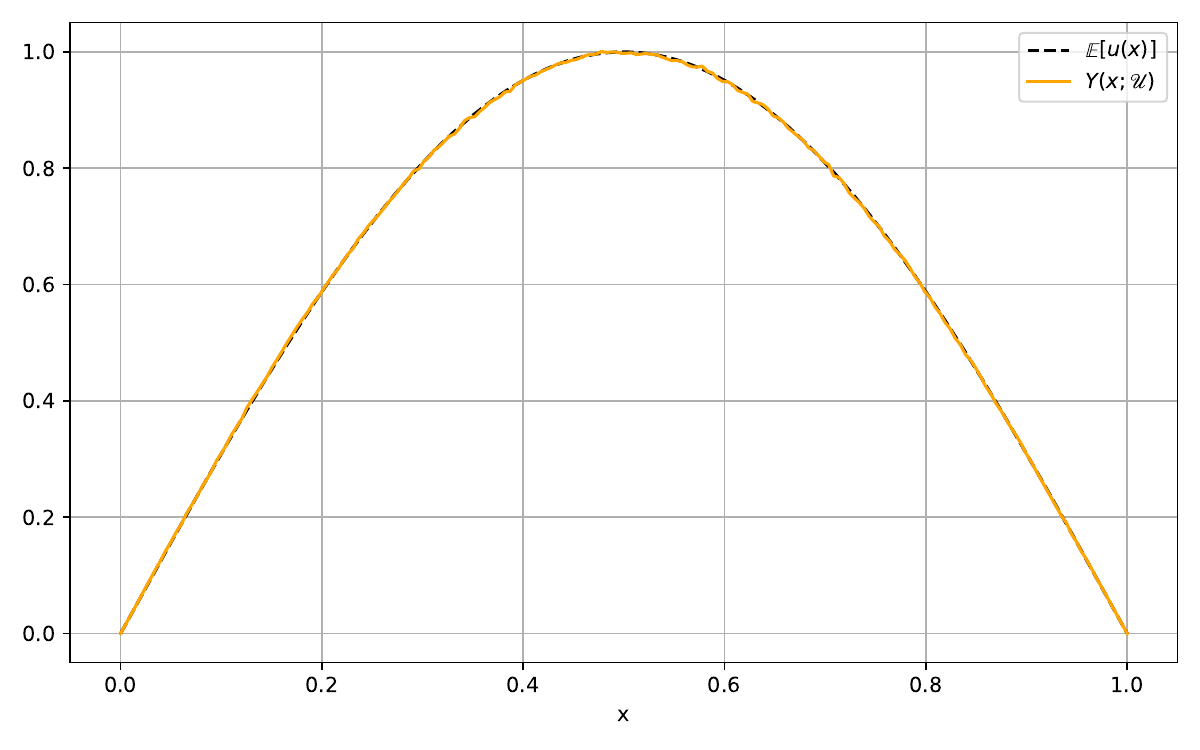}
        \caption{$Y(\mathscr{U})$ vs.\ $\mathbb{E}[u]$, $T=10$.}
        \label{fig:Inverse_exp_est_comp_u_10}
    \end{subfigure}
    \hfill
    \begin{subfigure}[t]{0.39\textwidth}
        \centering
        \includegraphics[width=\textwidth]{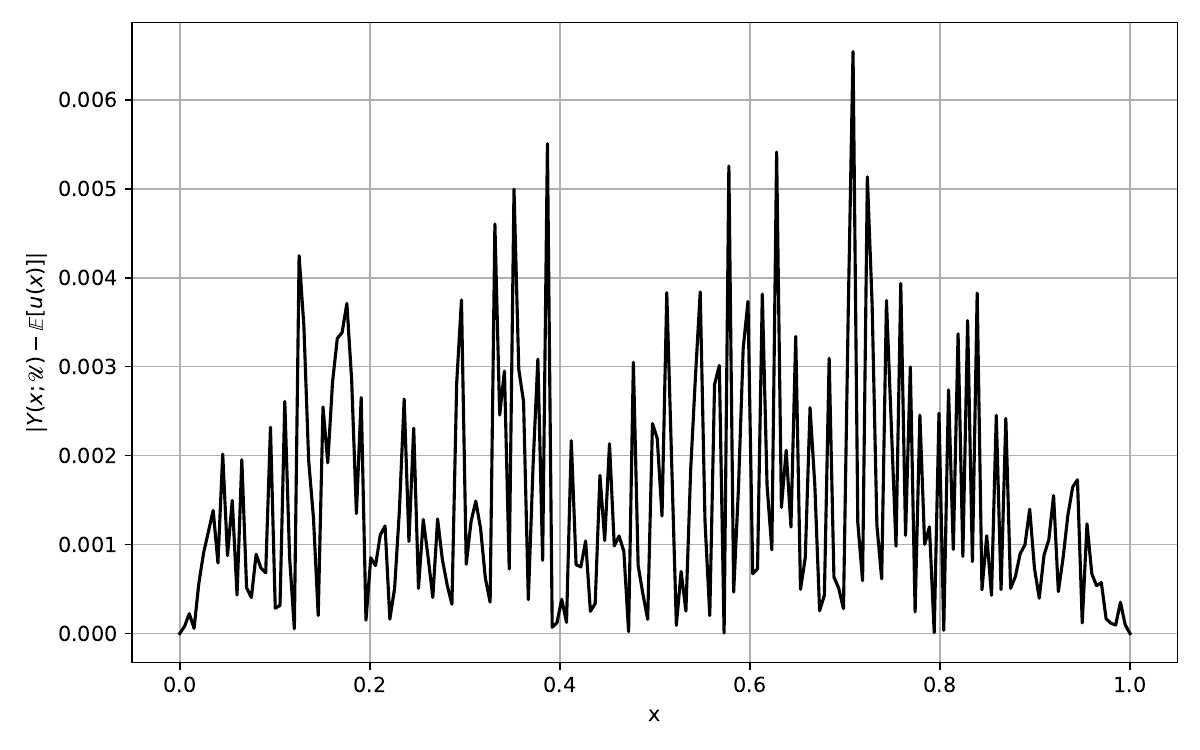}
        \caption{Pointwise error for $Y(\mathscr{U})$, $T=10$.}
        \label{fig:Inverse_exp_est_err_u_10}
    \end{subfigure}
    \hfill
    \begin{subfigure}[t]{0.39\textwidth}
        \centering
        \includegraphics[width=\textwidth]{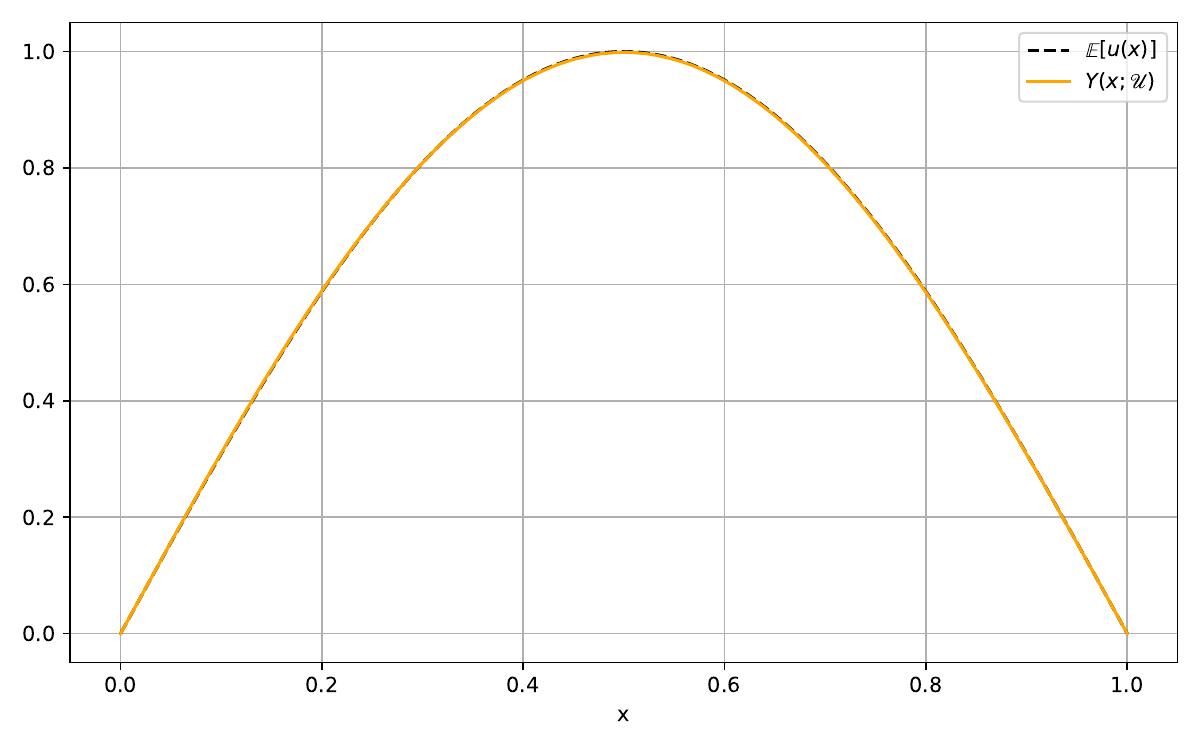}
        \caption{$Y(\mathscr{U})$ vs.\ $\mathbb{E}[u]$, $T=10^4$.}
        \label{fig:Inverse_exp_est_comp_u_10000}
    \end{subfigure}
    \hfill
    \begin{subfigure}[t]{0.39\textwidth}
        \centering
        \includegraphics[width=\textwidth]{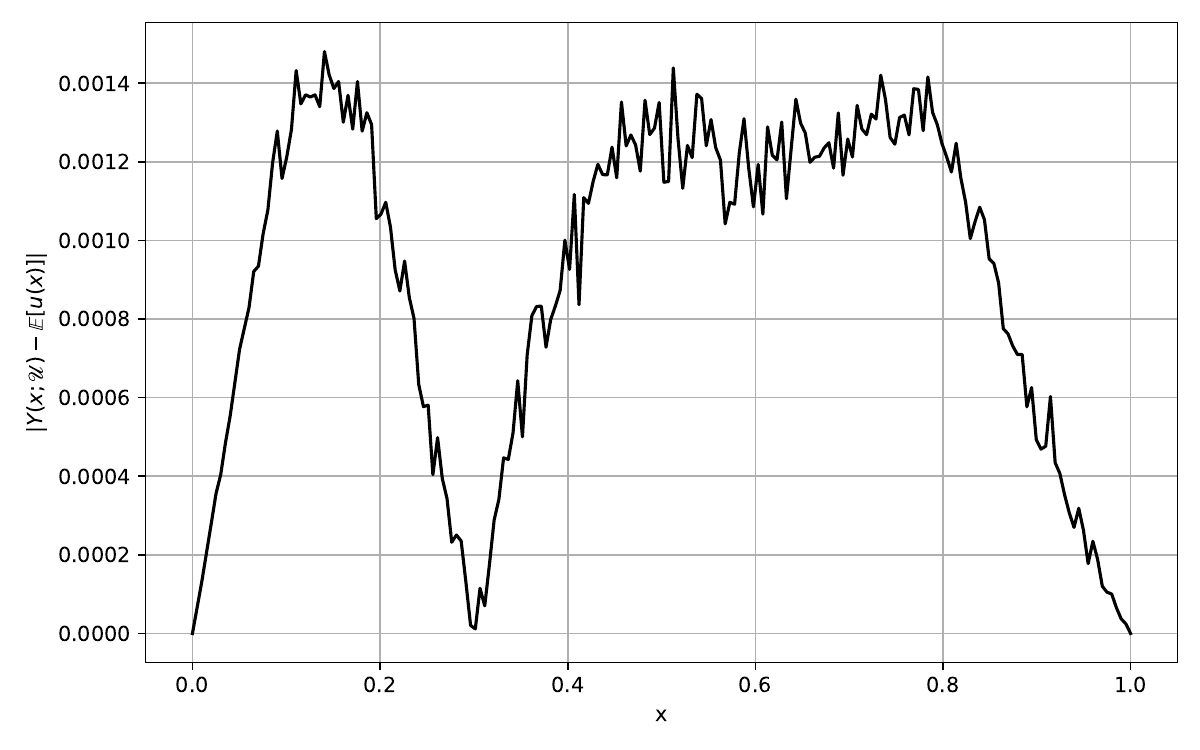}
        \caption{Pointwise error for $Y(\mathscr{U})$, $T=10^4$.}
        \label{fig:Inverse_exp_est_err_u_10000}
    \end{subfigure}
    \hfill
    \begin{subfigure}[t]{0.39\textwidth}
        \centering
        \includegraphics[width=\textwidth]{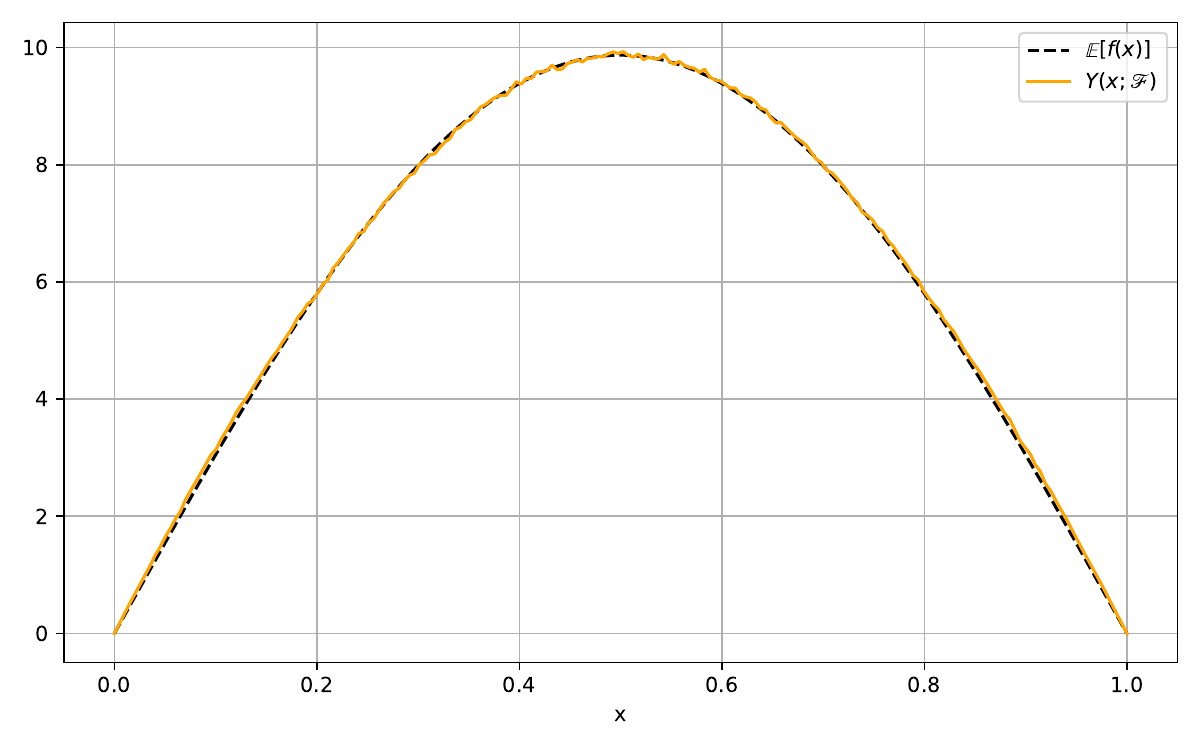}
        \caption{$Y(\mathscr{F})$ vs.\ $\mathbb{E}[f]$, $T=10$.}
        \label{fig:Inverse_exp_est_comp_f_10}
    \end{subfigure}
    \hfill
    \begin{subfigure}[t]{0.39\textwidth}
        \centering
        \includegraphics[width=\textwidth]{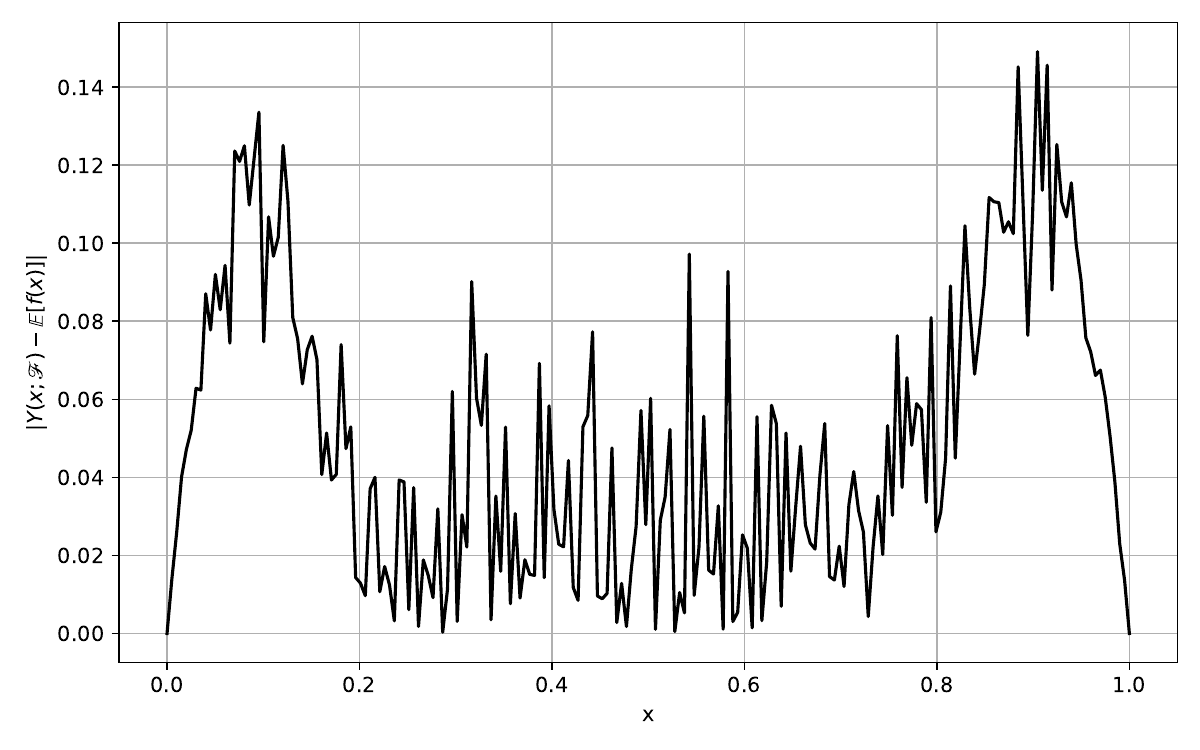}
        \caption{Pointwise error for $Y(\mathscr{F})$, $T=10$.}
        \label{fig:Inverse_exp_est_err_f_10}
    \end{subfigure}
    \hfill
    \begin{subfigure}[t]{0.39\textwidth}
        \centering
        \includegraphics[width=\textwidth]{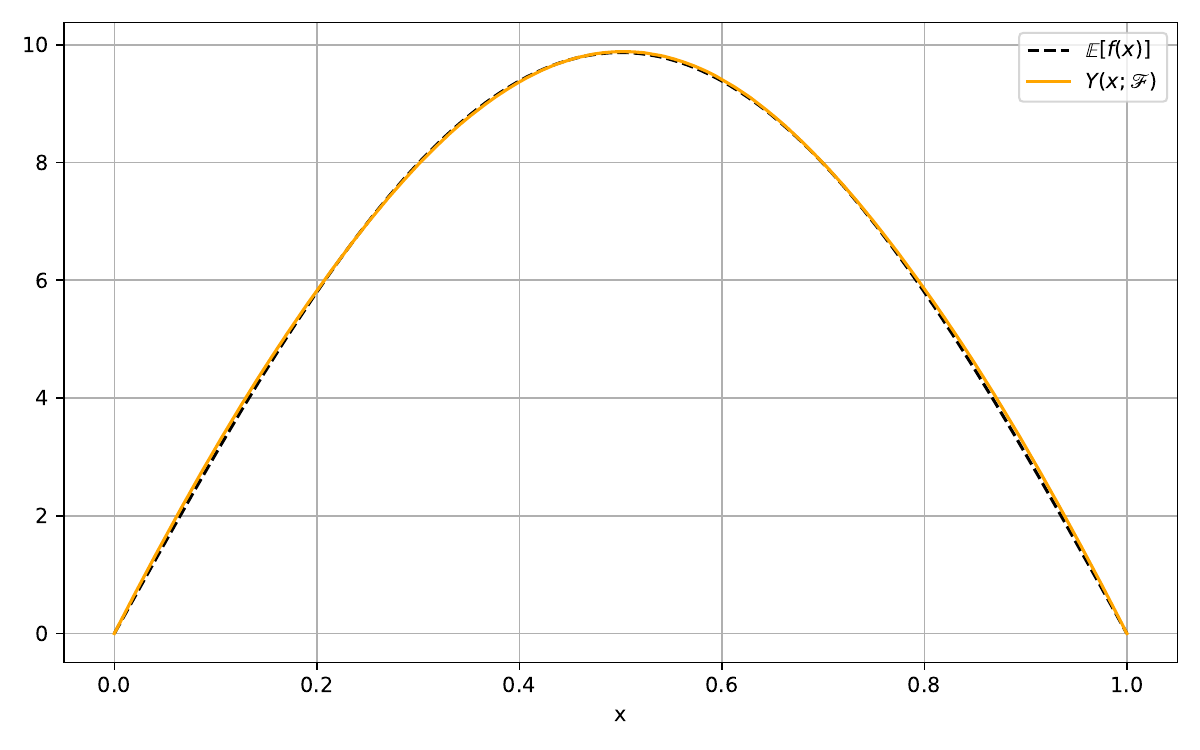}
        \caption{$Y(\mathscr{F})$ vs.\ $\mathbb{E}[f]$, $T=10^4$.}
        \label{fig:Inverse_exp_est_comp_f_10000}
    \end{subfigure}
    \hfill
    \begin{subfigure}[t]{0.39\textwidth}
        \centering
        \includegraphics[width=\textwidth]{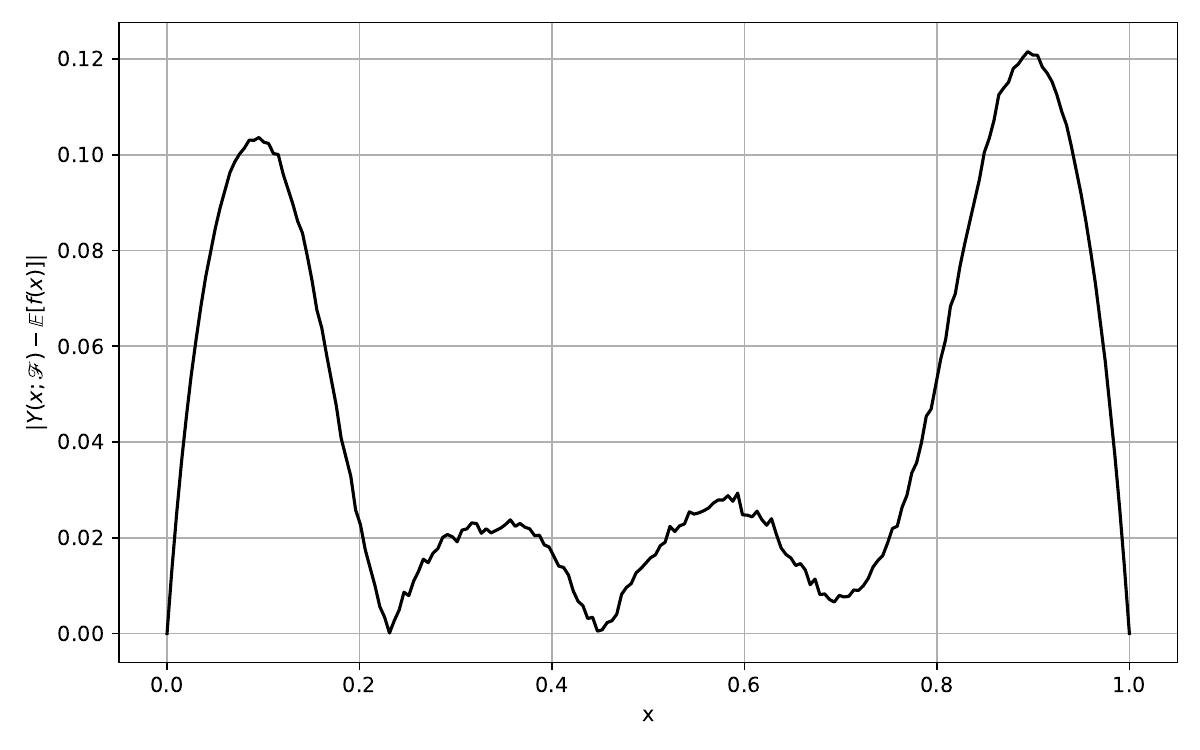}
        \caption{Pointwise error for $Y(\mathscr{F})$, $T=10^4$.}
        \label{fig:Inverse_exp_est_err_f_10000}
    \end{subfigure}

    \caption{Single-fidelity single-level Monte Carlo estimation of predictive means for the
    inverse benchmark. The top two rows report $Y(x)$ for $\mathscr{U}$ and
    compares it to $\mathbb{E}[u(x)]$ from \eqref{eq:numerics:inverse_ground_truth_moments};
    the bottom two rows report $Y(x)$ for $\mathscr{F}$ and compares it to
    $\mathbb{E}[f(x)]$. For each output we show a low-fidelity estimate ($T=10$)
    and a high-fidelity estimate ($T=10^4$), together with the corresponding
    pointwise errors. In contrast to the variance estimates in Figure \ref{fig:Inverse_var_est}, the predictive means converge to the analytical expectations as $T$ increases, indicating that the epistemic error is negligible compared to the Monte Carlo noise.}
    \label{fig:Inverse_exp_est}
\end{figure}

\subsubsection{Single-fidelity behaviour: predictive variances}\label{subsubsec:numerics:inverse_single_fidelity_var}
We next assess single-fidelity estimation of the dropout-induced predictive
variances via $V(x)$. We report standard deviations
$\sqrt{V(x)}$ and compare to the analytical standard deviations obtained
from \eqref{eq:numerics:inverse_ground_truth_vars}. Results for $T=10$ and
$T=10^4$ are shown in Figure~\ref{fig:Inverse_var_est}.

\begin{figure}[bhtp]
    \centering
    \begin{subfigure}[t]{0.385\textwidth}
        \includegraphics[width=\textwidth]{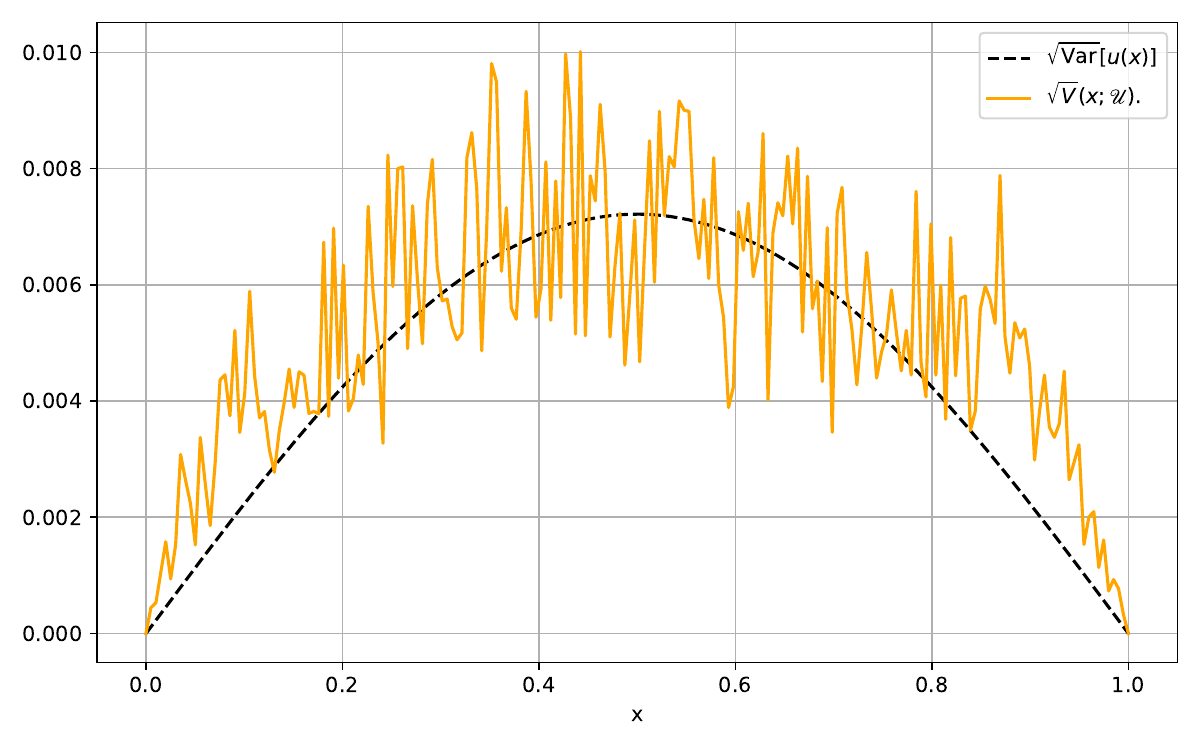}
        \caption{$\sqrt{V(\mathscr{U})}$ vs.\ $\sqrt{\mathrm{Var}[u]}$, $T=10$.}
        \label{fig:Inverse_var_est_comp_u_10}
    \end{subfigure}
    \hfill
    \begin{subfigure}[t]{0.385\textwidth}
        \includegraphics[width=\textwidth]{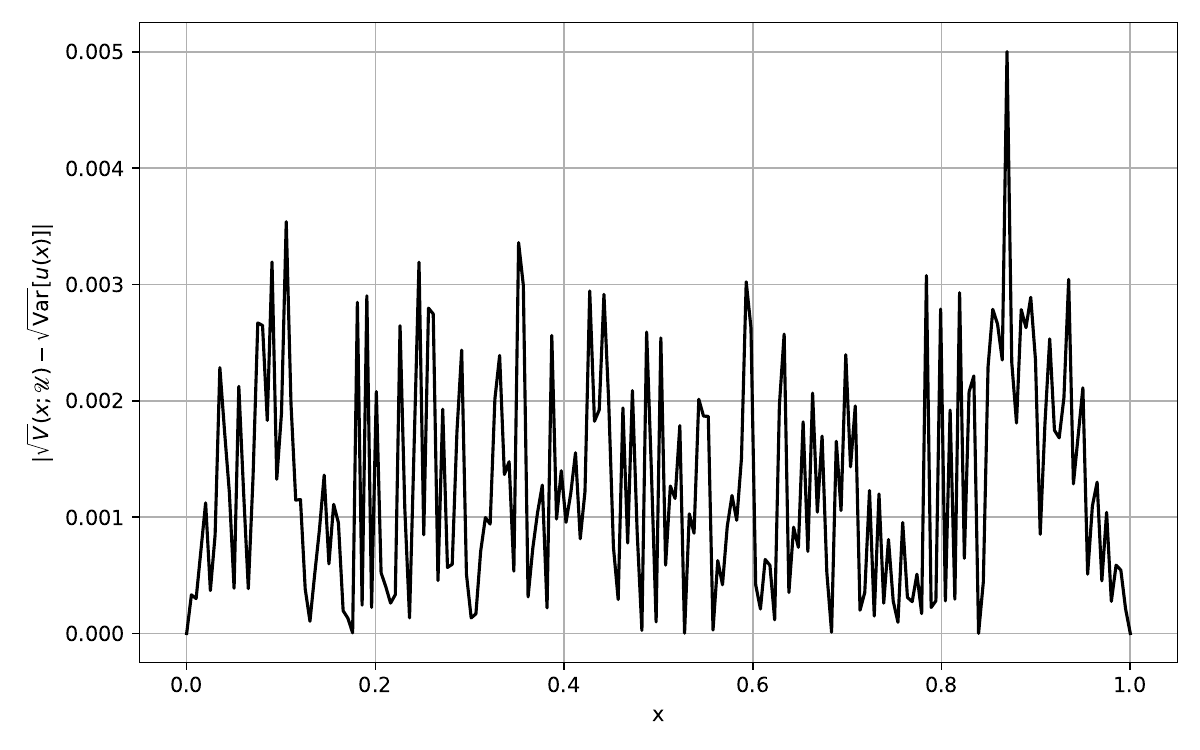}
        \caption{Pointwise error for $\sqrt{V(\mathscr{U})}$, $T=10$.}
        \label{fig:Inverse_var_est_err_u_10}
    \end{subfigure}
    \hfill
    \begin{subfigure}[t]{0.385\textwidth}
        \includegraphics[width=\textwidth]{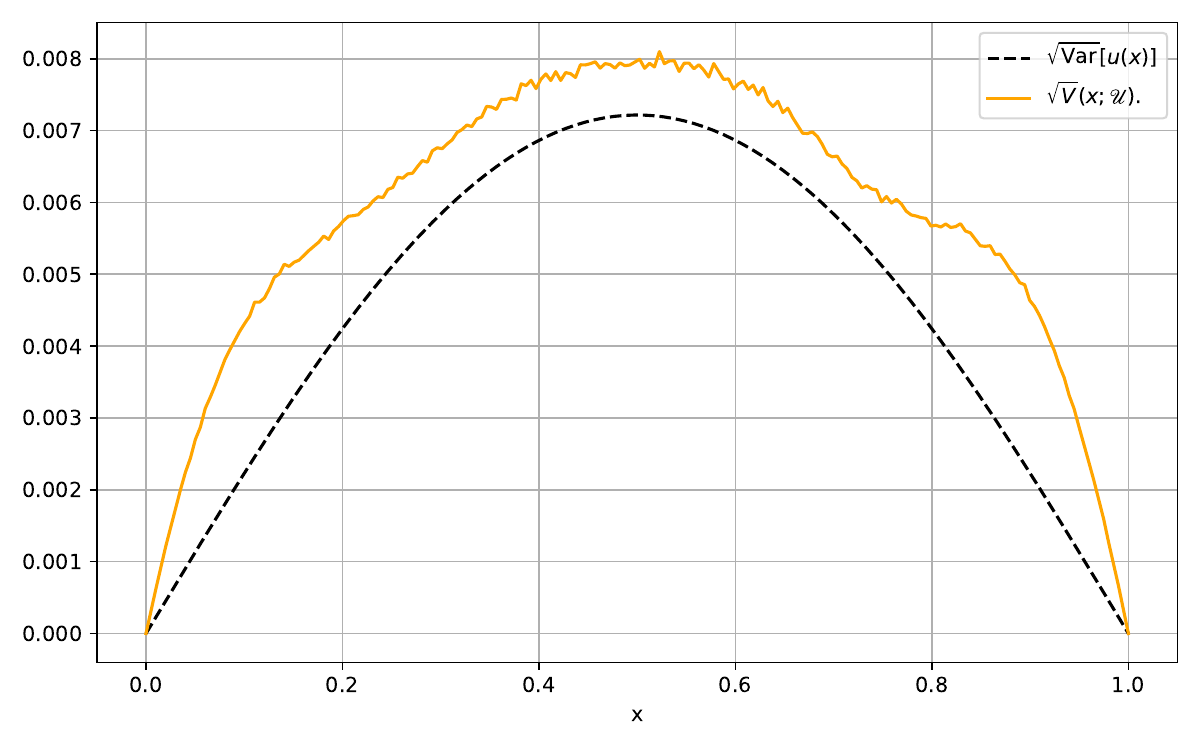}
        \caption{$\sqrt{V(\mathscr{U})}$ vs.\ $\sqrt{\mathrm{Var}[u]}$, $T=10^4$.}
        \label{fig:Inverse_var_est_comp_u_10000}
    \end{subfigure}
    \hfill
    \begin{subfigure}[t]{0.385\textwidth}
        \includegraphics[width=\textwidth]{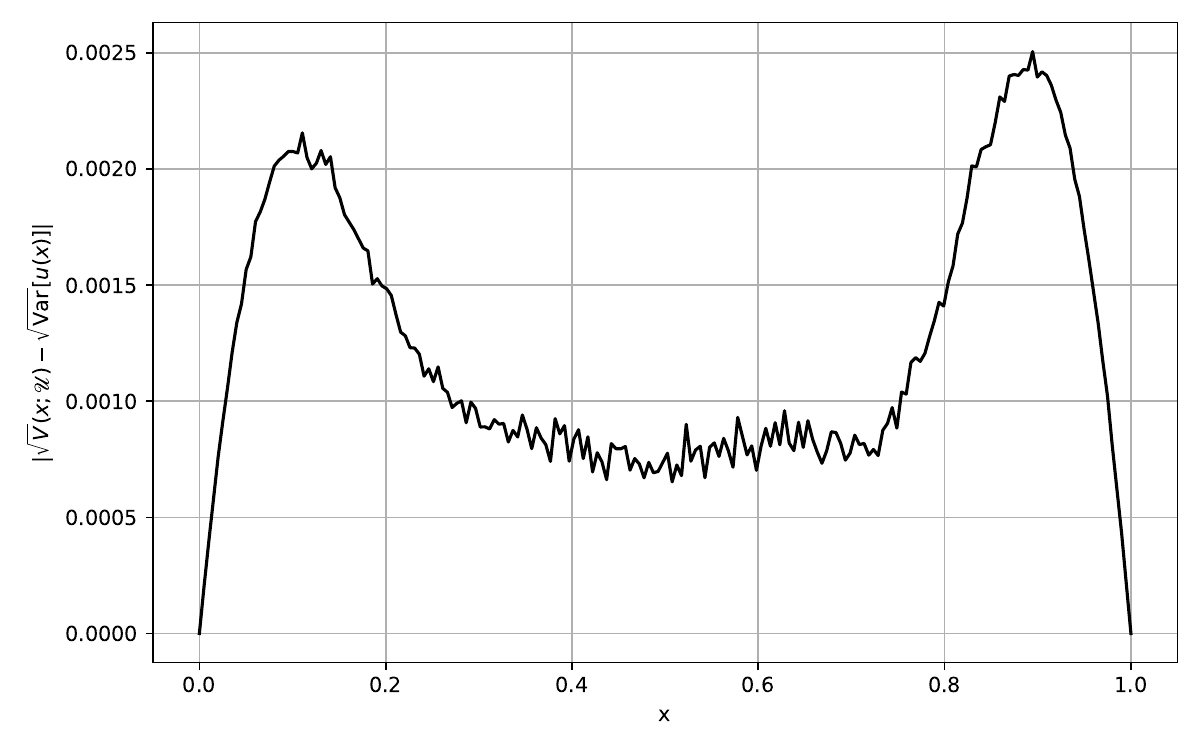}
        \caption{Pointwise error for $\sqrt{V(\mathscr{U})}$, $T=10^4$.}
        \label{fig:Inverse_var_est_err_u_10000}
    \end{subfigure}
    \hfill
    \begin{subfigure}[t]{0.385\textwidth}
        \includegraphics[width=\textwidth]{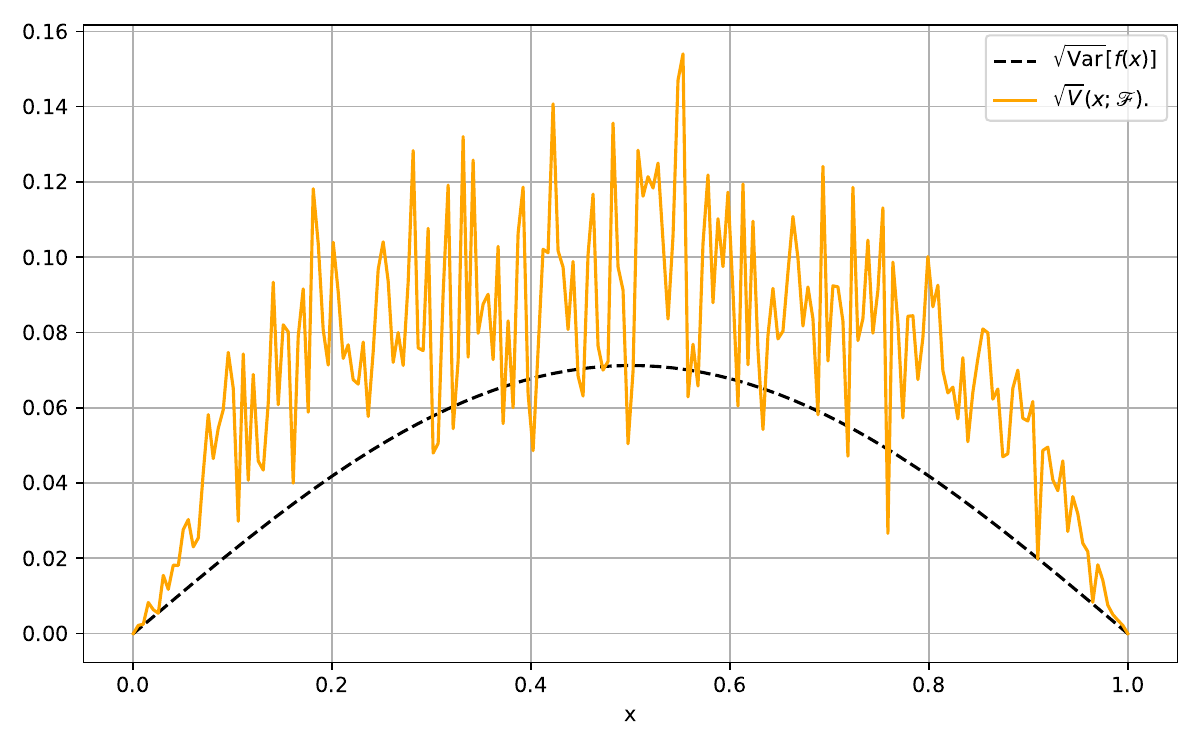}
        \caption{$\sqrt{V(\mathscr{F})}$ vs.\ $\sqrt{\mathrm{Var}[f]}$, $T=10$.}
        \label{fig:Inverse_var_est_comp_f_10}
    \end{subfigure}
    \hfill
    \begin{subfigure}[t]{0.385\textwidth}
        \includegraphics[width=\textwidth]{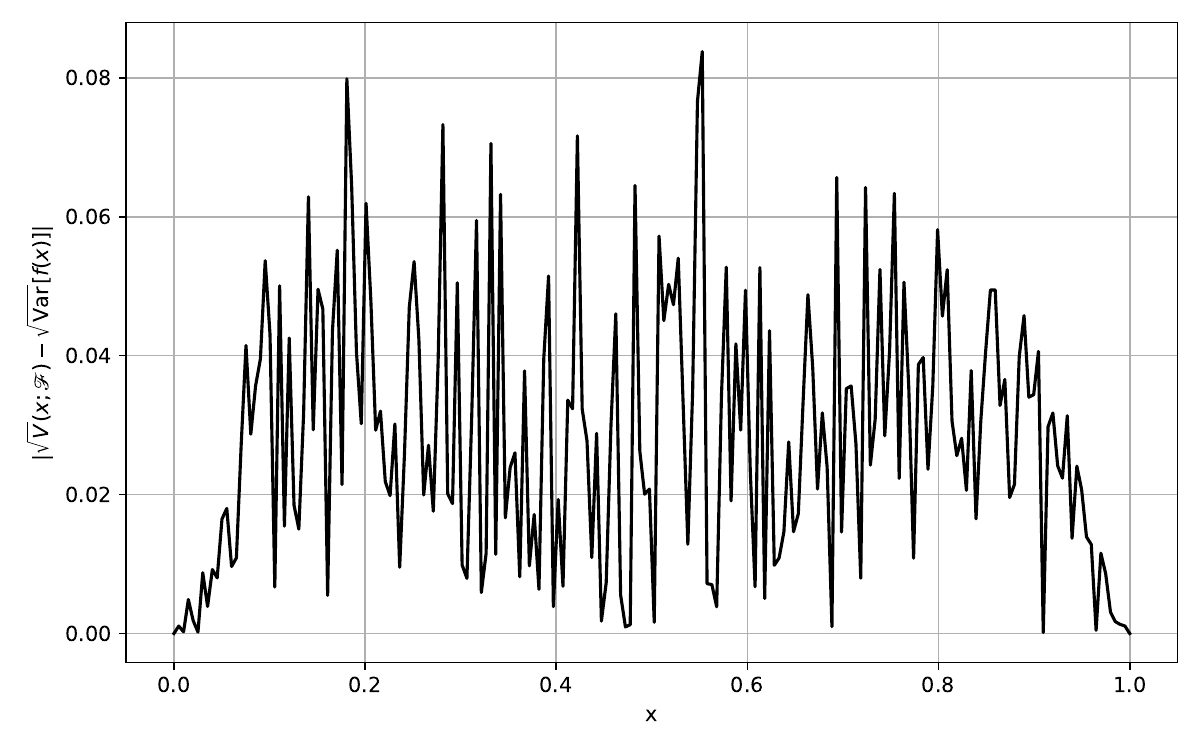}
        \caption{Pointwise error for $\sqrt{V(\mathscr{F})}$, $T=10$.}
        \label{fig:Inverse_var_est_err_f_10}
    \end{subfigure}
    \hfill
    \begin{subfigure}[t]{0.385\textwidth}
        \includegraphics[width=\textwidth]{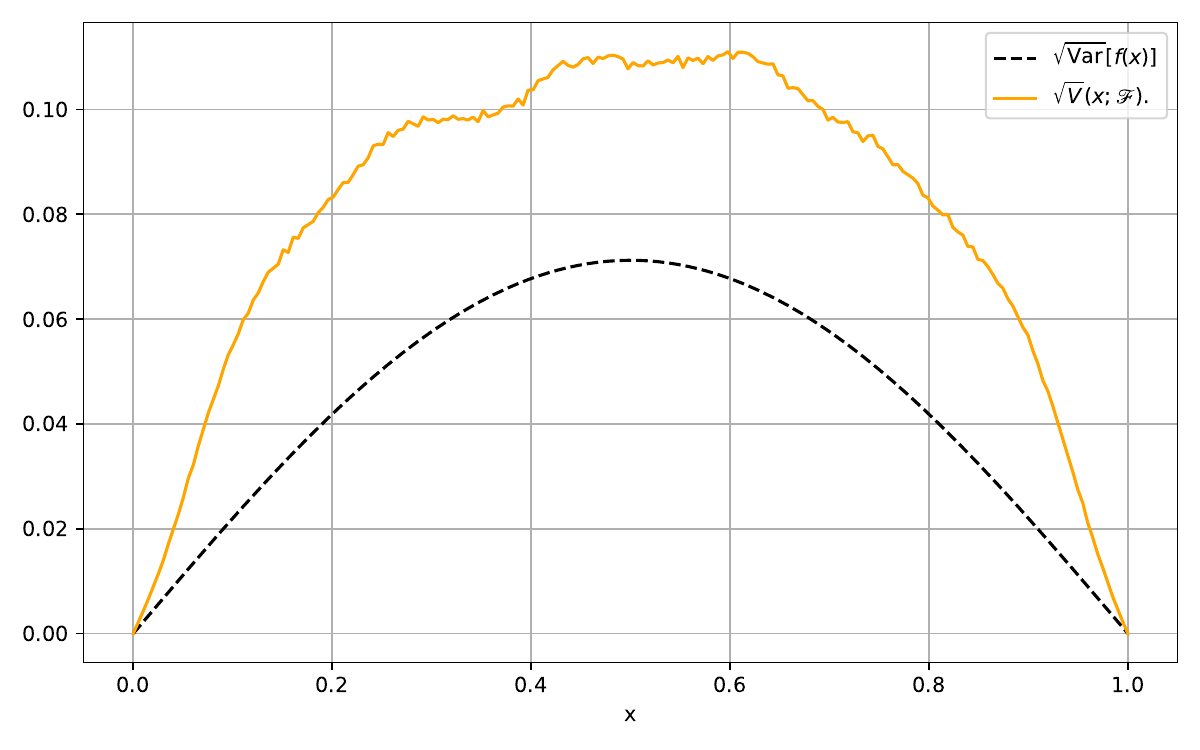}
        \caption{$\sqrt{V(\mathscr{F})}$ vs.\ $\sqrt{\mathrm{Var}[f]}$, $T=10^4$.}
        \label{fig:Inverse_var_est_comp_f_10000}
    \end{subfigure}
    \hfill
    \begin{subfigure}[t]{0.385\textwidth}
        \includegraphics[width=\textwidth]{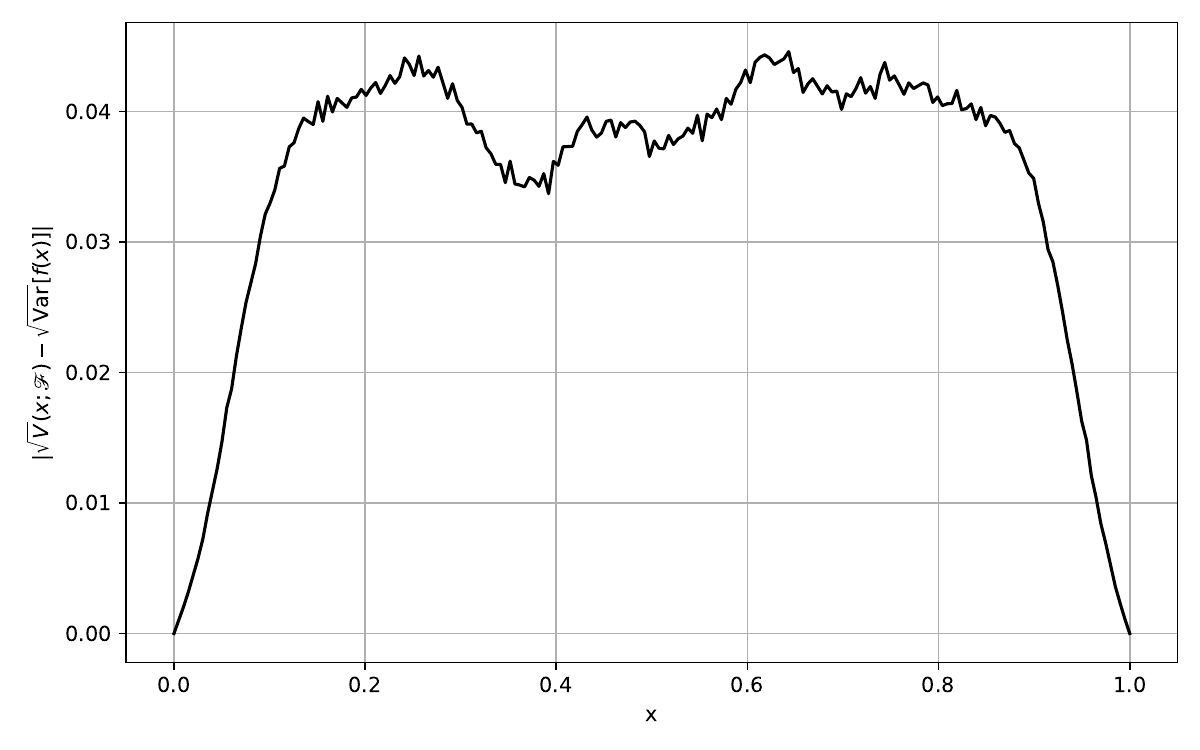}
        \caption{Pointwise error for $\sqrt{V(\mathscr{F})}$, $T=10^4$.}
        \label{fig:Inverse_var_est_err_f_10000}
    \end{subfigure}

    \caption{Single-fidelity single-level Monte Carlo estimation of predictive standard
    deviations for the inverse benchmark. The top row reports
    $\sqrt{V(x)}$ for $\mathscr{U}$ and compares it to
    $\sqrt{\mathrm{Var}[u(x)]}$ from \eqref{eq:numerics:inverse_ground_truth_vars};
    the bottom row reports $\sqrt{V(x)}$ for $\mathscr{F}$ and compares
    it to $\sqrt{\mathrm{Var}[f(x)]}$. Low- and high-fidelity results
    ($T=10$ and $T=10^4$) are shown together with pointwise errors. We observe that increasing $T$ reduces the sampling noise in the estimator of $V$, the remaining discrepancy at high $T$ reflects surrogate-induced epistemic mismatch rather than Monte Carlo error, indicating that the dropout posterior does not fully capture the aleatoric variability of the target mappings.}
    \label{fig:Inverse_var_est}
\end{figure}

\subsubsection{Empirical verification of sampling-variance rates}\label{subsubsec:numerics:inverse_variance_rates}
We now verify the predicted $\mathcal{O}(T^{-1})$ decay of the sampling variances
of the single-fidelity estimators (Lemma~\ref{lem:mlmc:single_fidelity_moments}).
For each fidelity $T$ we compute $M_0$ independent outer replicates and form the
empirical outer-variance estimators $\mathcal{S}^2_Y(x;M_0,T)$ and
$\mathcal{S}^2_V(x;M_0,T)$ for both outputs, then summarise their magnitudes
using discrete $L^1$ norms over the evaluation grid as described in
Subsubsection~\ref{subsubsec:numerics:estimators_norms}.
Figure~\ref{fig:inverse_var_v_fid} reports the resulting log--log fits and
Table~\ref{tab:inverse_log_grads_vs_fid} lists the estimated slopes and their
$99\%$ confidence intervals.

\begin{table}[htbp]
    \centering
    \begin{tabular}{|c|c|c|c|}
        \hline \hline
        & lower $99\%$ CI & est.\ gradient & upper $99\%$ CI \\
        \hline
        $\mathcal S_Y^2(\mathscr{U})$ & -1.0159 & -1.0033 & -0.9907  \\
        $\mathcal S_Y^2(\mathscr{F})$ & -1.0094 & -0.9976 & -0.9857  \\
        $\mathcal S_V^2(\mathscr{U})$ & -1.0000 & -0.9852 & -0.9704 \\
        $\mathcal S_V^2(\mathscr{F})$ & -0.9963 & -0.9842 & -0.9722 \\
        \hline
    \end{tabular}
    \caption{Estimated log--log gradients of the single-fidelity sampling
    variances in Figure~\ref{fig:inverse_var_v_fid} with respect to $T$, together
    with $99\%$ confidence intervals. We observe that the the $\leb1$-norms of the estimator variances exhibit approximately $\mathcal{O}(T^{-1})$ decay.}
    \label{tab:inverse_log_grads_vs_fid}
\end{table}

\begin{figure}[htbp]
\centering
\begin{subfigure}{0.43\textwidth}
    \includegraphics[width=\textwidth]{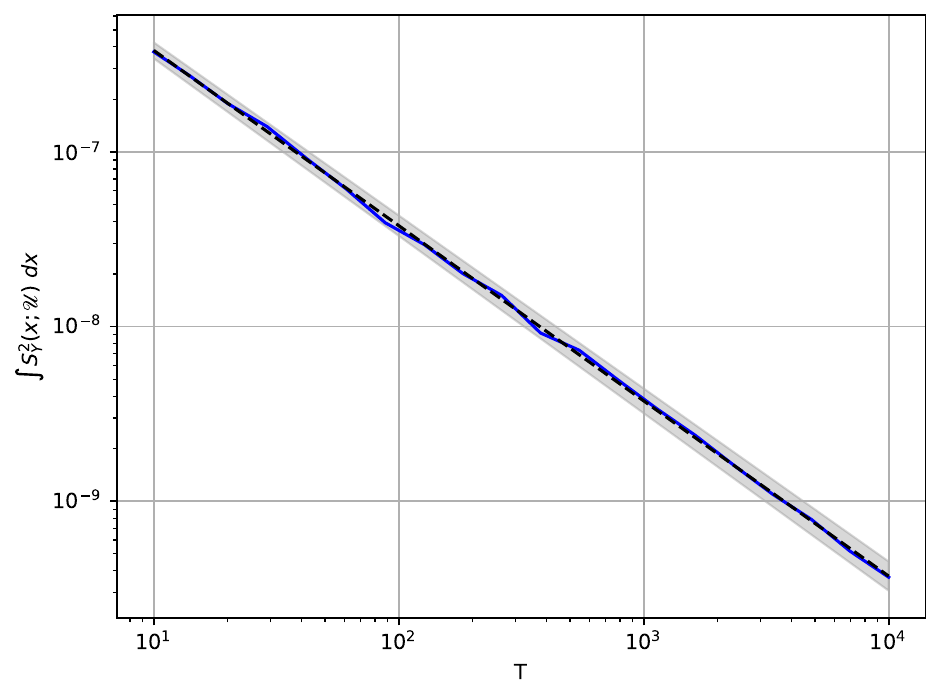}
    \caption{$\|\mathcal S_Y^2(\mathscr{U})\|_{L^1}$ vs.\ $T$.}
    \label{fig:inverse_exp_estimator_var_v_fid_u}
\end{subfigure}
\hfill
\begin{subfigure}{0.43\textwidth}
    \includegraphics[width=\textwidth]{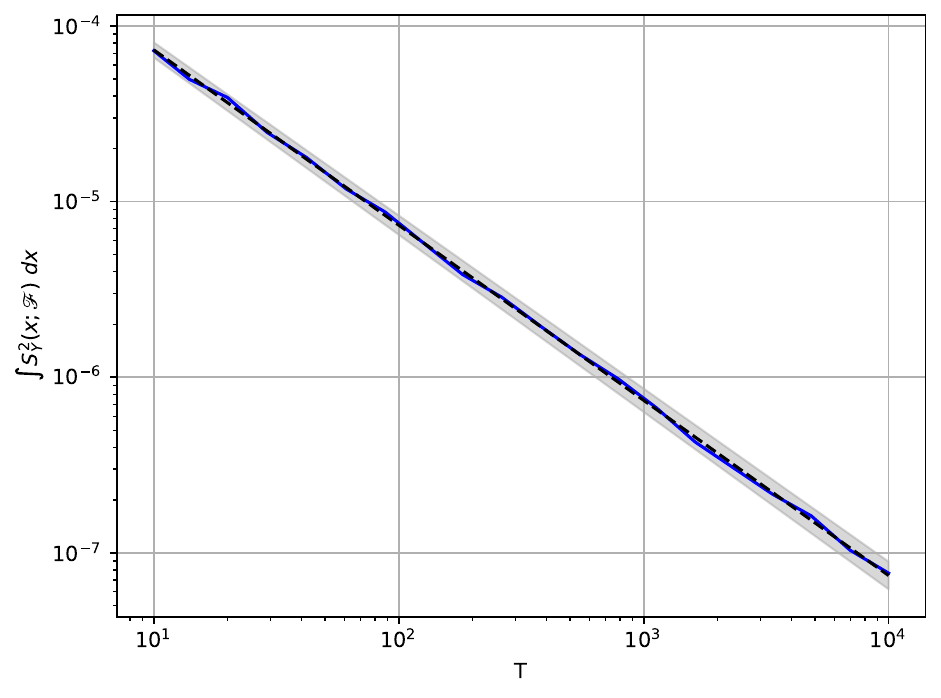}
    \caption{$\|\mathcal S_Y^2(\mathscr{F})\|_{L^1}$ vs.\ $T$.}
    \label{fig:inverse_exp_estimator_var_v_fid_f}
\end{subfigure}
\hfill
\begin{subfigure}{0.43\textwidth}
    \includegraphics[width=\textwidth]{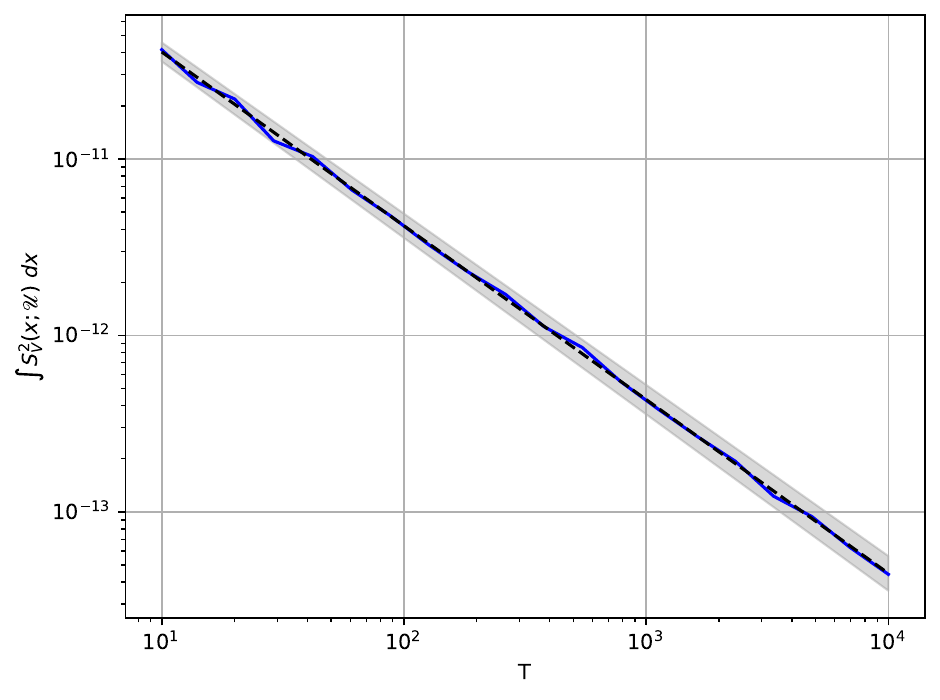}
    \caption{$\|\mathcal S_V^2(\mathscr{U})\|_{L^1}$ vs.\ $T$.}
    \label{fig:inverse_var_estimator_var_v_fid_u}
\end{subfigure}
\hfill
\begin{subfigure}{0.43\textwidth}
    \includegraphics[width=\textwidth]{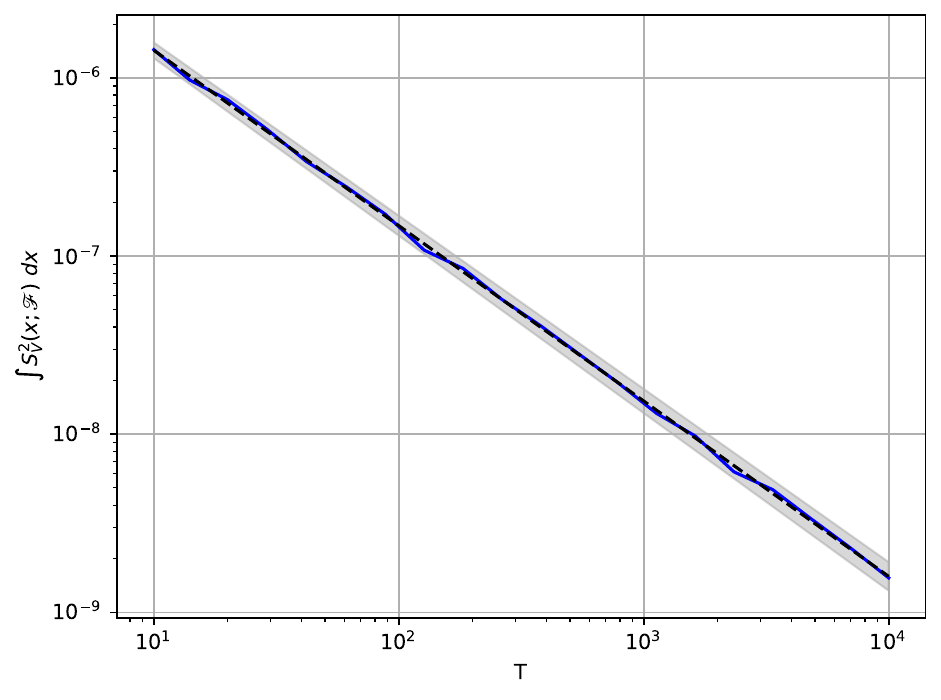}
    \caption{$\|\mathcal S_V^2(\mathscr{F})\|_{L^1}$ vs.\ $T$.}
    \label{fig:inverse_var_estimator_var_v_fid_f}
\end{subfigure}
\caption{Empirical sampling-variance decay for the inverse benchmark. Each panel
reports the discrete $L^1$ norm over the evaluation grid of the empirical
sampling variance of a single-fidelity estimator as a function of the inner
fidelity $T$. The dashed line is a linear fit on log--log axes and the shaded
region indicates the $99\%$ confidence interval for the fitted slope. The
corresponding gradients are listed in Table~\ref{tab:inverse_log_grads_vs_fid}.}
\label{fig:inverse_var_v_fid}
\end{figure}

\subsubsection{Fixed-cost multilevel allocation test}\label{subsubsec:numerics:inverse_fixed_cost_allocation}
We conclude with a fixed-cost allocation experiment for the multilevel
estimators.  We use a three-level ladder $\vec
T=(T_0,T_1,T_2)=(4,8,16)$ (so $L=2$) and a fixed coupled evaluation
budget $c_{\mathrm{cpl}}=1000$
(subsubsection~\ref{subsubsec:numerics:cost_accounting}). We enumerate
feasible integer allocations $\vec M=(M_0,M_1,M_2)$ satisfying
$T_0M_0+(T_1-T_0)M_1+(T_2-T_1)M_2=c_{\mathrm{cpl}}$ and the constraint $M_\ell\ge
2$. We enumerate feasible integer allocations $\vec M=(M_0,M_1,M_2)$
satisfying the budget and the constraint $M_\ell\ge 2$ (so that sample
variances are defined), and we evaluate the empirical MLMC variance
estimators $S_Y^2$ and $S_V^2$ for both outputs. For visualisation we
plot reciprocal surfaces $1/\|S_{\cdot}^2\|_{L^1}$ as functions of
$(M_1,M_2)$, noting that $M_0$ is then fixed by the cost constraint.

The black marker indicates an empirical minimiser over feasible
integer allocations. The red marker indicates the continuous optimum
obtained from the allocation formulas in
\S~\ref{subsubsec:mlmc:allocation_mean} for the mean estimator and
\S~\ref{subsubsec:mlmc:allocation_variance} for the variance estimator
(in the latter case under the zero-excess-kurtosis closure used for
the simplified allocation).

The theoretical continuous optima for $\vec T=(4,8,16)$ and $c_{\mathrm{cpl}}=1000$
are the same as the previous problem, as the optimal $\vec M$ depends solely on the fidelity ladder $\vec T$,
\begin{equation}\label{eq:inverse_problem_optimal_levels_v_theo}
\vec M^\star_{Y}=\left(\frac{1000}{12}, \frac{1000}{24}, \frac{1000}{48}\right),
\qquad
\vec M^\star_{V}\approx(82.64, 44.17, 19.76),
\end{equation}
where $\vec M^\star_Y$ is the mean-estimator allocation and $\vec M^\star_V$ is
the variance-estimator allocation under the zero-excess-kurtosis closure.

\begin{figure}[htbp]
\centering
\begin{subfigure}{0.43\textwidth}
    \includegraphics[width=\textwidth]{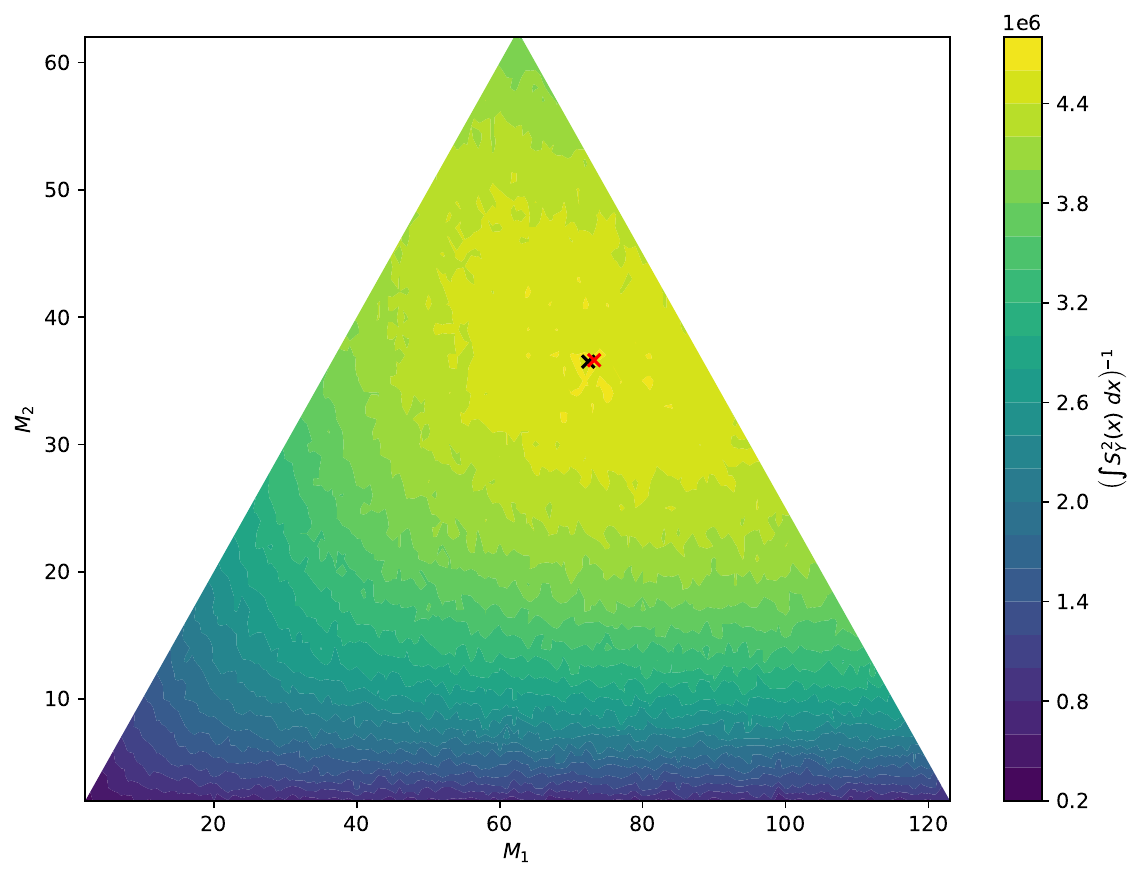}
    \caption{$1/\|S_Y^2(\mathscr{U})\|_{L^1}$ as a function of $(M_1,M_2)$.}
    \label{fig:inverse_u_var_exp_estimator_fixed_cost}
\end{subfigure}
\hfill
\begin{subfigure}{0.43\textwidth}
    \includegraphics[width=\textwidth]{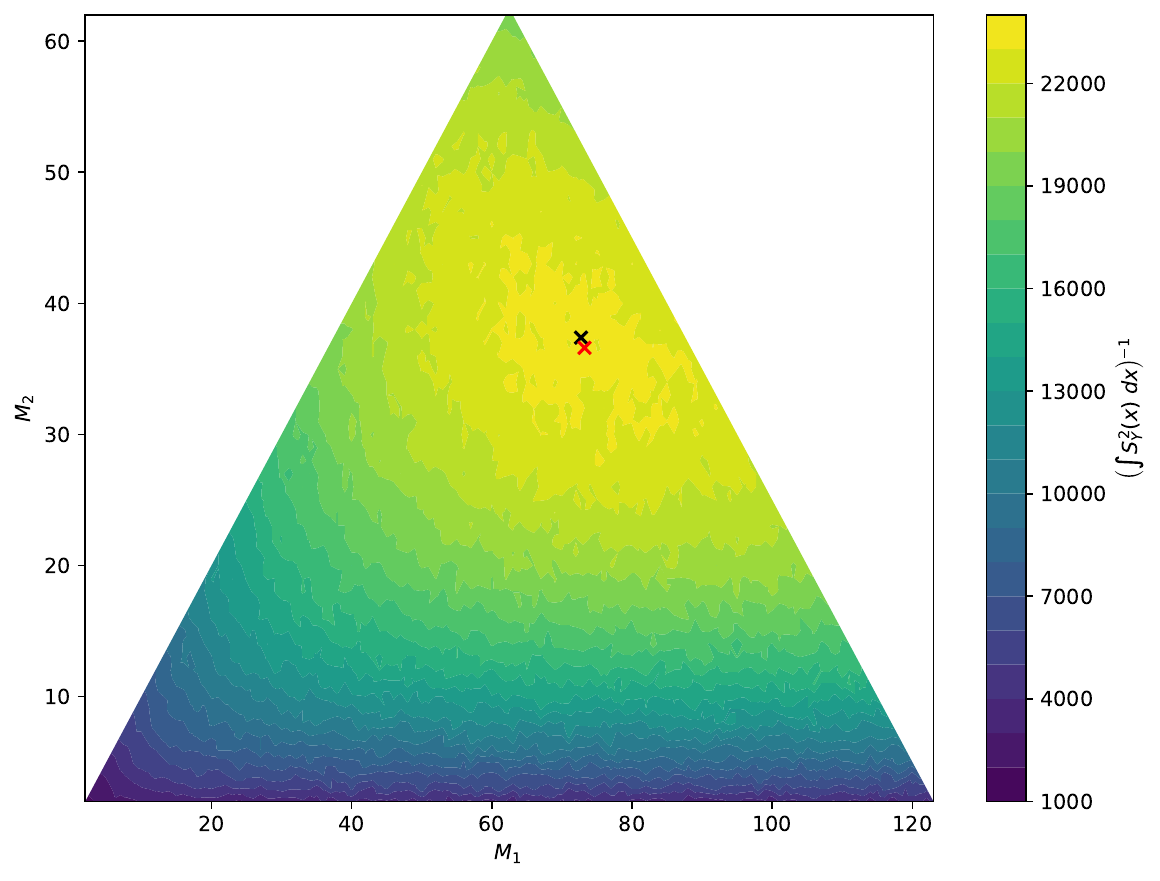}
    \caption{$1/\|S_Y^2(\mathscr{F})\|_{L^1}$ as a function of $(M_1,M_2)$.}
    \label{fig:inverse_f_var_exp_estimator_fixed_cost}
\end{subfigure}
\hfill
\begin{subfigure}{0.43\textwidth}
    \includegraphics[width=\textwidth]{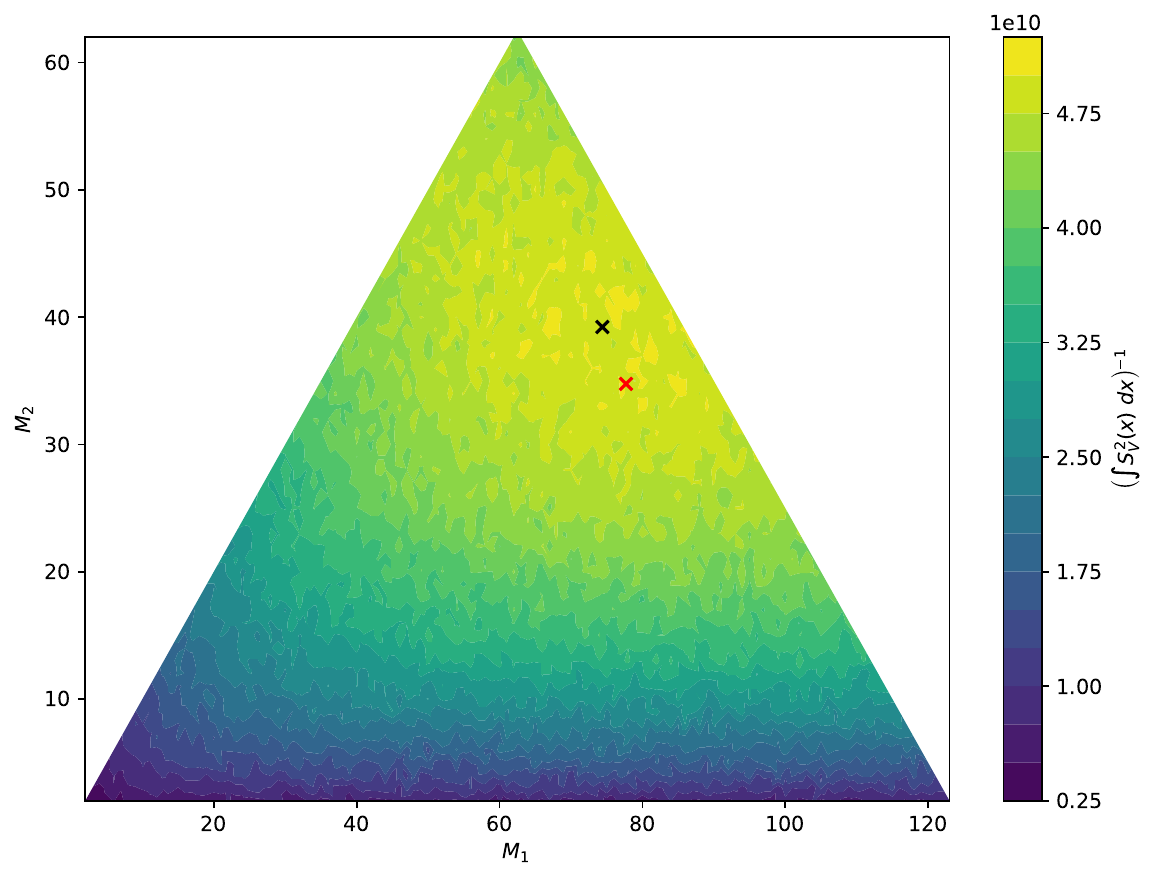}
    \caption{$1/\|S_V^2(\mathscr{U})\|_{L^1}$ as a function of $(M_1,M_2)$.}
    \label{fig:inverse_u_var_var_estimator_fixed_cost}
\end{subfigure}
\hfill
\begin{subfigure}{0.43\textwidth}
    \includegraphics[width=\textwidth]{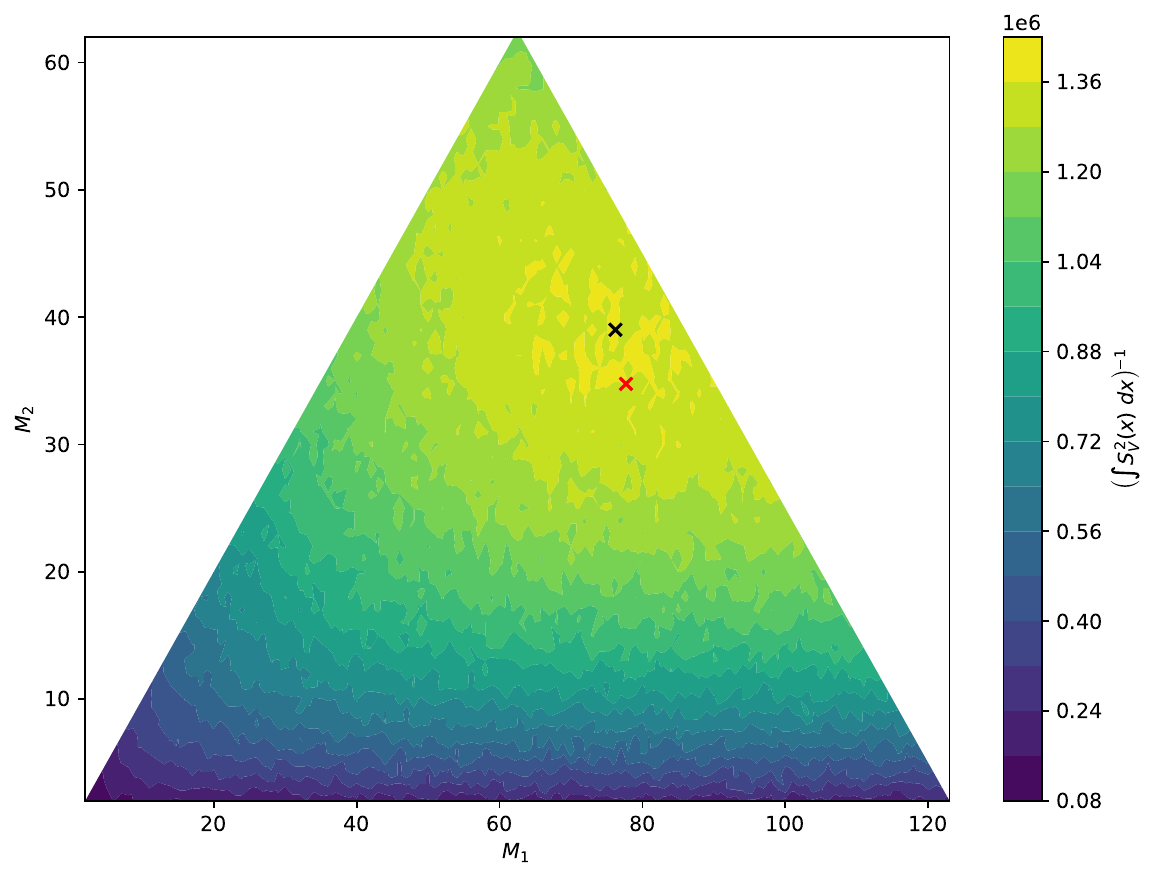}
    \caption{$1/\|S_V^2(\mathscr{F})\|_{L^1}$ as a function of $(M_1,M_2)$.}
    \label{fig:inverse_f_var_var_estimator_fixed_cost}
\end{subfigure}
\caption{Fixed-cost allocation test for the inverse benchmark with
$\vec T=(4,8,16)$ and $c_{\mathrm{cpl}}=1000$. Each panel shows the reciprocal
variance surface of an MLMC estimator as a function of the integer allocations
$(M_1,M_2)$, with $M_0$ determined by the coupled cost constraint. The black
marker indicates an empirical minimiser over feasible integer allocations and
the red marker indicates the continuous optimum given by the allocation
formulas in \S~\ref{subsubsec:mlmc:allocation_mean} (top row) and
\S~\ref{subsubsec:mlmc:allocation_variance} (bottom row, using the
zero-excess-kurtosis closure for the simplified allocation).}
\label{fig:inverse_var_estimator_fixed_cost}
\end{figure}

\section{Code discussion}\label{sec:code_discussion}

\subsection{Reproducibility}\label{subsec:code:reproducibility}
All experiments are reproducible up to (i) the inherent stochasticity of
dropout-based inference and (ii) non-determinism arising from GPU arithmetic and
library implementations. In our code, pseudo-random number generators are
seeded at the start of training using values specified in \texttt{CONFIG.py},
and the same seeding mechanism is available for evaluation to make figures
repeatable when desired. Each training run creates a timestamped output
directory in which the exact configuration file is copied and the trained model
weights are saved, so that runs can be reloaded without ambiguity. While these
steps make results repeatable on the same software stack and hardware, bit-for-bit
determinism is not guaranteed across different GPU architectures or different
PyTorch/CUDA versions.

The code repository for this paper is available at
\url{https://github.com/aaronpim/Multilevel_Monte_Carlo_Dropout_PINNs.git}.

\subsection{Computational cost and implementation notes}\label{subsec:code:cost_notes}
Throughout, the operative notion of computational cost is the number of
stochastic forward passes of the trained network, consistent with the cost
models in Subsection~\ref{subsec:mlmc:cost_model}. In particular, the coupled
sampling strategy reuses dropout masks across fidelities and therefore reduces
the number of additional network evaluations required to form multilevel
increments compared with an uncoupled baseline.

In practice, the additional arithmetic required to aggregate samples into
$Y(x,T)$, $V(x,T)$, and the multilevel estimators is negligible compared with
the network evaluation cost for the architectures considered here. For this
reason we do not attempt a hardware-specific FLOP accounting and instead report
budgets and allocations in terms of forward-pass counts.

\section{Conclusions}
\label{sec:conclusion}

We have developed a multilevel Monte Carlo framework for MC-dropout in
which epistemic uncertainty is induced by sampling dropout masks and
fidelity is controlled by the number of stochastic forward
passes. Coupling coarse and fine inner estimators through mask reuse,
the resulting multilevel estimators for predictive means and variances
remain unbiased for the corresponding dropout-induced quantities while
substantially reducing sampling variance at a given evaluation
budget. We derived explicit variance–cost relations, effective coupled
cost models and principled sample allocation rules, and we verified
these predictions on forward and inverse PINNs--Uzawa benchmarks,
where the empirical variance rates and fixed-cost allocation
experiments align with the theory up to the expected effects of
integer feasibility and moment-closure assumptions.

\appendix

\section{Supplementary derivations}\label{app:supp}
\subsection{Proofs of expectation and variance identities}
\begin{proof}[Proof of Lemma~\ref{lem:mlmc:single_fidelity_moments}]
The identities for $\mathbb{E}[Y]$ and $\mathrm{Var}[Y]$ are standard.
The unbiasedness of the sample variance implies $\mathbb{E}[V]=\mu_2$,
and the expression for $\mathrm{Var}[V]$ is a classical formula for the
sampling variance of the unbiased variance estimator, stated for
example in \cite[equation 10]{benhamou2018few}.
\end{proof}

\begin{proof}[Proof of Lemma~\ref{lem:mlmc:mean_moments}]
By linearity of expectation and \eqref{eq:mlmc:mean_telescoping},
\[
\mathbb{E}\big[\mathcal{Y}(x;\vec M,\vec T)\big]
=
\mathbb{E}\big[Y(x;T_0)\big]
+
\sum_{\ell=1}^{L}\mathbb{E}\big[\Delta Y(x;T_\ell)\big]
=
\mathbb{E}\big[Y(x;T_L)\big]
=
\mu(x),
\]
where we used Lemma \ref{lem:mlmc:single_fidelity_moments}. Consider the the expectation of the empirical sampling
variance of the expectation estimator, the linearlity of expectation implies that
\begin{equation}
    \mathbb{E}[S^2_Y(x;\vec M,\vec T)]
    =
    \frac{1}{M_0} \mathbb{E}[\mathcal{S}^2_{Y}(x;M_0,T_0)]
    +
    \sum_{\ell = 1}^{L}\frac{1}{M_\ell} \mathbb{E}[\mathcal{S}^2_{\Delta Y}(x;M_\ell,T_\ell)],
\end{equation}
The definitions of $\mathcal{S}^2_{Y}$ and $\mathcal{S}^2_{\Delta Y}$ are that they are the unbiased variance estimators of $Y$ and $\Delta Y$ respectively. Under the assumption that the outer samples are i.i.d., we have that expectation of the unbiased variance estimator is the variance, $\mathbb{E}[\mathcal{S}^2_{Z}] = \mathrm{Var}[Z]$, and therefore that
\begin{equation}
    \mathbb{E}[S^2_Y(x;\vec M,\vec T)]
    =
    \frac{1}{M_0}\mathrm{Var}\big[Y(x;T_0)\big]
    +
    \sum_{\ell=1}^L \frac{1}{M_\ell}\mathrm{Var}\big[\Delta Y(x;T_\ell)\big].
\end{equation}
Under the coupling \eqref{eq:mlmc:mean_recurrence},
$\Delta Y(x;T_\ell)=\frac{1}{T_\ell}\sum_{t=T_{\ell-1}+1}^{T_\ell}\mathscr{D}(x;\theta_t)
-\frac{T_\ell-T_{\ell-1}}{T_\ell}Y(x;T_{\ell-1})$, and a direct
computation using independence yields
$\mathrm{Var}[\Delta Y(x;T_\ell)]=\mu_2(x)\left(\frac{1}{T_{\ell-1}}-\frac{1}{T_\ell}\right)$,
while Lemma \ref{lem:mlmc:single_fidelity_moments} gives
$\mathrm{Var}[Y(x;T_0)]=\mu_2(x)/T_0$.
\end{proof}

\begin{proof}[Proof of Lemma~\ref{lem:cov_overlap_var}]
Let $D_1,\dots,D_{T_{\ell}}$ be i.i.d, 
and let $\mu_k:=\mathbb{E}\bigl[(X-\mu)^k\bigr]$ denote
the $k$th central moment.
Introduce the mean of the \emph{new} block of size $T_{\ell}-T_{\ell-1}$,
\begin{equation}\label{eq:cov_overlap:def_Yd}
    \overline{D}_{T_{\ell-1}:T_{\ell}}:=\frac{1}{T_{\ell}-T_{\ell-1}}\sum_{j=T_{\ell-1}+1}^{T_{\ell}} D_j.
\end{equation}
A standard update identity due to \cite{ChanGolubLeveque1979} states that
\begin{equation}\label{eq:cov_overlap:chan_update}
\begin{aligned}
    (T_{\ell}-1) V(T_{\ell})
    =&
    (T_{\ell-1}-1) V(T_{\ell-1})
    +\sum_{i=T_{\ell-1}+1}^{T_{\ell}}\bigl(D_i-\overline{D}_{T_{\ell-1}:T_{\ell}}\bigr)^2
    \\ &+\frac{(T_{\ell}-T_{\ell-1})T_{\ell-1}}{T_{\ell}}\bigl(Y(T_{\ell-1})-\overline{D}_{T_{\ell-1}:T_{\ell}}\bigr)^2.
\end{aligned}
\end{equation}
The middle term on the right-hand side depends only on the new samples
$D_{T_{\ell-1}+1},\dots,D_{T_{\ell}}$ and is independent of $V(T_{\ell-1})$, hence its covariance
with $V(T_{\ell-1})$ is zero. Taking covariance with $V(T_{\ell-1})$ in
\eqref{eq:cov_overlap:chan_update} and using bilinearity gives
\begin{equation}\label{eq:cov_overlap:cov_reduce}
\begin{aligned}
    (T_{\ell}-1) \mathrm{Cov}\bigl[V(T_{\ell}),V(T_{\ell-1})\bigr]
    =&
    (T_{\ell-1}-1) \mathrm{Var}\bigl[V(T_{\ell-1})\bigr]
    \\ &+\frac{(T_{\ell}-T_{\ell-1})T_{\ell-1}}{T_{\ell}} \mathrm{Cov}\Bigl(\bigl(Y(T_{\ell-1})-\overline{D}_{T_{\ell-1}:T_{\ell}}\bigr)^2,\ V(T_{\ell-1})\Bigr).
\end{aligned}
\end{equation}
Since $\overline{D}_{T_{\ell-1}:T_{\ell}}$ is independent of $\{D_i\}_{i=1}^{T_{\ell-1}}$ and
$\mathbb{E}[\overline{D}_{T_{\ell-1}:T_{\ell}}]=\mu$, a short expansion plus independence, and expressing the covariance as the product of their deviations yields
\begin{equation}\label{eq:cov_overlap:drop_new_block}
    \begin{aligned}
    \mathrm{Cov}\Bigl(\bigl(Y(T_{\ell-1})-\overline{D}_{T_{\ell-1}:T_{\ell}}\bigr)^2,\ V(T_{\ell-1})\Bigr)
    =&
    \mathbb{E}\Bigl[\bigl(Y(T_{\ell-1})-\mu\bigr)^2\bigl(V(T_{\ell-1})-\mu_2\bigr)\Bigr].    
    \end{aligned}
\end{equation}
It remains to compute the expectation on the right-hand side.
Write $\xi_i:=D_i-\mu$, then $Y(T)-\mu=\frac{1}{T}\sum_{i=1}^{T}\xi_i$ and
\begin{equation}\label{eq:cov_overlap:var_decomp}
    (T_{\ell-1}-1) V(T_{\ell-1})=\sum_{i=1}^{T_{\ell-1}}\xi_i^2-T_{\ell-1}\bigl(Y(T_{\ell-1})-\mu\bigr)^2.
\end{equation}
Using \eqref{eq:cov_overlap:var_decomp},
\begin{equation}\label{eq:cov_overlap:EY2v}
\begin{aligned}
    \mathbb{E}\bigl[(Y(T_{\ell-1})-\mu)^2 V(T_{\ell-1})\bigr]
    &=
    \frac{1}{T_{\ell-1}-1}\mathbb{E}\Bigl[(Y(T_{\ell-1})-\mu)^2\sum_{i=1}^{T_{\ell-1}}\xi_i^2\Bigr]
    -\frac{T_{\ell-1}}{T_{\ell-1}-1}\mathbb{E}\bigl[(Y(T_{\ell-1})-\mu)^4\bigr].
\end{aligned}
\end{equation}
A direct index-counting argument gives
\begin{equation}\label{eq:cov_overlap:moment_id_1}
\mathbb{E}\Bigl[(Y(T_{\ell-1})-\mu)^2\sum_{i=1}^{T_{\ell-1}}\xi_i^2\Bigr]
=
\frac{1}{T_{\ell-1}^2} \mathbb{E}\Bigl[\Bigl(\sum_{i=1}^{T_{\ell-1}}\xi_i\Bigr)^2\Bigl(\sum_{k=1}^{T_{\ell-1}}\xi_k^2\Bigr)\Bigr]
=
\frac{\mu_4}{T_{\ell-1}}+\frac{T_{\ell-1}-1}{T_{\ell-1}}\mu_2^2,
\end{equation}
and similarly
\begin{equation}\label{eq:cov_overlap:moment_id_2}
    \mathbb{E}\bigl[(Y(T_{\ell-1})-\mu)^4\bigr]
    =
    \frac{\mu_4}{T_{\ell-1}^3}+\frac{3(T_{\ell-1}-1)}{T_{\ell-1}^3}\mu_2^2.
\end{equation}
Substituting \eqref{eq:cov_overlap:moment_id_1}--\eqref{eq:cov_overlap:moment_id_2}
into \eqref{eq:cov_overlap:EY2v} yields
\begin{equation}\label{eq:cov_overlap:EY2v_simplified}
    \mathbb{E}\bigl[(Y(T_{\ell-1})-\mu)^2 V(T_{\ell-1})\bigr]
    =
    \frac{\mu_4}{T_{\ell-1}^2}+\frac{T_{\ell-1}-3}{T_{\ell-1}^2}\mu_2^2.
\end{equation}
Since $\mathbb{E}\bigl[(Y(T_{\ell-1})-\mu)^2\bigr]=\mu_2/T_{\ell-1}$, we obtain
\begin{equation}\label{eq:cov_overlap:key_cov_term}
\mathbb{E}\Bigl[\bigl(Y(T_{\ell-1})-\mu\bigr)^2\bigl(V(T_{\ell-1})-\mu_2\bigr)\Bigr]
=
\Bigl(\frac{\mu_4}{T_{\ell-1}^2}+\frac{T_{\ell-1}-3}{T_{\ell-1}^2}\mu_2^2\Bigr)-\frac{\mu_2}{T_{\ell-1}}\mu_2
=
\frac{1}{T_{\ell-1}^2}\bigl(\mu_4-3\mu_2^2\bigr).
\end{equation}
Combining \eqref{eq:cov_overlap:cov_reduce}, \eqref{eq:cov_overlap:drop_new_block}
and \eqref{eq:cov_overlap:key_cov_term} gives
\[
    (T_{\ell}-1) \mathrm{Cov}\bigl[V(T_{\ell}),V(T_{\ell-1})\bigr]
    =
    (T_{\ell-1}-1) \mathrm{Var}\bigl[V(T_{\ell-1})\bigr]
    +\frac{T_{\ell}-T_{\ell-1}}{T_{\ell-1}T_{\ell}}\bigl(\mu_4-3\mu_2^2\bigr).
\]
Dividing by $T_{\ell}-1$ yields the first line of \eqref{eq:cov_overlap:cov_formula}.
The $\mathrm{Var}\bigl[V(T_{\ell-1})\bigr]$ term may be directly evaluated using the closed form for
$\mathrm{Var}[V(T_{\ell-1})]$, which is
$\frac{1}{T_{\ell-1}}\bigl(\mu_4-\frac{T_{\ell-1}-3}{T_{\ell-1}-1}\mu_2^2\bigr)$.
\end{proof}

\subsection{Proofs of optimal allocation results}\label{app:allocation_proofs}

\begin{proof}[Proof of Lemma~\ref{lem:mlmc:allocation_mean}]
Fix $x$ and $\vec T=(T_0,\dots,T_L)$ with $1\le T_0<\cdots<T_L$. Recall the
additive coupled-cost representation
\begin{equation}\label{eq:app:alloc:cost}
    c = T_0 M_0 + \sum_{\ell=1}^{L}(T_\ell-T_{\ell-1})M_\ell
    = \sum_{\ell=0}^{L} a_\ell M_\ell,
    \qquad
    a_0:=T_0,\quad a_\ell:=T_\ell-T_{\ell-1}\ (\ell\ge 1),
\end{equation}
and the variance decomposition
\begin{equation}\label{eq:app:alloc:mean_obj}
    \mathbb{E}[S^2_Y(x;\vec M,\vec T)]
    =
    \frac{v_0(x)}{M_0} + \sum_{\ell=1}^{L}\frac{v_\ell(x)}{M_\ell},
\end{equation}
where
\begin{equation}\label{eq:app:alloc:mean_level_vars}
    v_0(x)=\frac{\mu_2(x)}{T_0},
    \qquad
    v_\ell(x)=\mu_2(x)\left(\frac{1}{T_{\ell-1}}-\frac{1}{T_\ell}\right)
    =\mu_2(x) \frac{T_\ell-T_{\ell-1}}{T_\ell T_{\ell-1}},
    \quad \ell\ge 1.
\end{equation}
On the open set $\{M_\ell>0\}_{\ell=0}^L$, the objective
$\sum_{\ell=0}^L v_\ell(x)/M_\ell$ is strictly convex and the constraint
$\sum_{\ell=0}^L a_\ell M_\ell=c$ is affine, so any stationary point is the
unique global minimiser.

Introduce the Lagrangian
\begin{equation}\label{eq:app:alloc:lagrangian_mean}
    \mathcal{L}(\vec M,\lambda)
    :=
    \sum_{\ell=0}^{L}\frac{v_\ell(x)}{M_\ell}
    +\lambda\left(\sum_{\ell=0}^{L}a_\ell M_\ell-c\right).
\end{equation}
The first-order optimality conditions are, for each $\ell$,
\begin{equation}\label{eq:app:alloc:foc_mean}
    \frac{\partial\mathcal{L}}{\partial M_\ell}
    =
    -\frac{v_\ell(x)}{M_\ell^{2}}+\lambda a_\ell
    =0,
    \qquad\Longrightarrow\qquad
    M_\ell=\sqrt{\frac{v_\ell(x)}{\lambda a_\ell}}.
\end{equation}
Using \eqref{eq:app:alloc:mean_level_vars}, we obtain
\begin{equation}\label{eq:app:alloc:mean_ratios}
    \frac{v_0(x)}{a_0}=\frac{\mu_2(x)}{T_0^2},
    \qquad
    \frac{v_\ell(x)}{a_\ell}=\frac{\mu_2(x)}{T_\ell T_{\ell-1}}\quad(\ell\ge 1),
\end{equation}
hence \eqref{eq:app:alloc:foc_mean} implies
\begin{equation}\label{eq:app:alloc:mean_M_form}
    M_0=\frac{\sqrt{\mu_2(x)}}{\sqrt{\lambda}}\frac{1}{T_0},
    \qquad
    M_\ell=\frac{\sqrt{\mu_2(x)}}{\sqrt{\lambda}}\frac{1}{\sqrt{T_\ell T_{\ell-1}}},
    \quad \ell=1,\dots,L.
\end{equation}
Enforcing the budget constraint \eqref{eq:app:alloc:cost} gives
\begin{equation}\label{eq:app:alloc:mean_lambda}
\begin{aligned}
    c
    &=
    T_0 M_0+\sum_{\ell=1}^{L}(T_\ell-T_{\ell-1})M_\ell
    \\
    &=
    \frac{\sqrt{\mu_2(x)}}{\sqrt{\lambda}}
    \left(
        1+\sum_{\ell=1}^{L}\frac{T_\ell-T_{\ell-1}}{\sqrt{T_\ell T_{\ell-1}}}
    \right),
\end{aligned}
\end{equation}
so
\begin{equation}\label{eq:app:alloc:mean_prefactor}
    \frac{\sqrt{\mu_2(x)}}{\sqrt{\lambda}}
    =
    c\left(
        1+\sum_{\ell=1}^{L}\frac{T_\ell-T_{\ell-1}}{\sqrt{T_\ell T_{\ell-1}}}
    \right)^{-1}.
\end{equation}
Substituting \eqref{eq:app:alloc:mean_prefactor} into \eqref{eq:app:alloc:mean_M_form}
yields \eqref{eq:mlmc:allocation_mean_solution}. The factor $\mu_2(x)$ cancels,
so the optimal allocation depends only on $\vec T$ and $c$.
\end{proof}
\begin{proof}[Proof of Lemma~\ref{lem:mlmc:allocation_variance_general}]
We seek to minimise the expectation of the empirical sampling variance of the variance estimator
\begin{equation}\label{eq:app:alloc:var_obj}
    \mathbb{E}[S^2_V(x;\vec M,\vec T)]
    =
    \frac{w_0(x)}{M_0}+\sum_{\ell=1}^{L}\frac{w_\ell(x)}{M_\ell}
\end{equation}
subject to the same coupled-cost constraint
\begin{equation}\label{eq:app:alloc:var_cost}
    c=\sum_{\ell=0}^{L}a_\ell M_\ell,
    \qquad
    a_0:=T_0,\quad a_\ell:=T_\ell-T_{\ell-1}\ (\ell\ge 1).
\end{equation}
As in the proof of Lemma~\ref{lem:mlmc:allocation_mean}, strict convexity of
$\sum w_\ell/M_\ell$ for $M_\ell>0$ implies that the KKT stationary point is the
unique global minimiser. Introducing the Lagrangian
\[
    \mathcal{L}(\vec M,\lambda)
    :=
    \sum_{\ell=0}^{L}\frac{w_\ell(x)}{M_\ell}
    +\lambda\left(\sum_{\ell=0}^{L}a_\ell M_\ell-c\right),
\]
the first-order conditions give
\begin{equation}\label{eq:app:alloc:foc_var}
    -\frac{w_\ell(x)}{(M_\ell^{\star})^2}+\lambda a_\ell=0,
    \qquad\Longrightarrow\qquad
    M_\ell^{\star}=\sqrt{\frac{w_\ell(x)}{\lambda a_\ell}}.
\end{equation}
Enforcing the budget constraint \eqref{eq:app:alloc:var_cost} gives
\begin{equation}
    c=\sum_{k=0}^{L} \sqrt{\frac{w_k(x) a_k}{\lambda}},
    \qquad
     \sqrt{\lambda}=\dfrac{1}{c}\sum_{k=0}^{L} \sqrt{w_k(x) a_k}, 
     \qquad 
     M_\ell^{\star}
    =
    c \frac{\sqrt{w_\ell(x)/a_\ell}}{\sum_{k=0}^{L}\sqrt{w_k(x)a_k}}.
\end{equation}
\end{proof}

\printbibliography

@phdthesis{gal2016uncertaintyPHD,
    author = {Gal, Yarin},
    title = {Uncertainty in deep learning},
    school = {University of Cambridge},
    year = {2016}
}

@inproceedings{guo2017calibration,
  title={On calibration of modern neural networks},
  author={Guo, Chuan and Pleiss, Geoff and Sun, Yu and Weinberger, Kilian Q},
  booktitle={International conference on machine learning},
  pages={1321--1330},
  year={2017},
  organization={PMLR}
}

@article{benhamou2018few,
  title={A few properties of sample variance},
  author={Benhamou, Eric},
  journal={arXiv preprint arXiv:1809.03774},
  year={2018}
}

@inproceedings{cox2024bayesian,
	author = {Cox, Alexander MG and Hattam, Laura and Kyprianou, Andreas E and Pryer, Tristan},
	booktitle = {Proceedings A},
	number = {2301},
	organization = {The Royal Society},
	pages = {20230836},
	title = {A Bayesian inverse approach to proton therapy dose delivery verification},
	volume = {480},
	year = {2024}}

@article{pim2025surrogate,
	author = {Pim, Aaron and Pryer, Tristan},
	journal = {arXiv preprint arXiv:2509.18155},
	title = {Surrogate Modelling of Proton Dose with Monte Carlo Dropout Uncertainty Quantification},
	year = {2025}}

@article{yang2021b,
	author = {Yang, Liu and Meng, Xuhui and Karniadakis, George Em},
	journal = {Journal of Computational Physics},
	pages = {109913},
	publisher = {Elsevier},
	title = {B-PINNs: Bayesian physics-informed neural networks for forward and inverse PDE problems with noisy data},
	volume = {425},
	year = {2021}}

@article{srivastava2014dropout,
	author = {Srivastava, Nitish and Hinton, Geoffrey and Krizhevsky, Alex and Sutskever, Ilya and Salakhutdinov, Ruslan},
	journal = {The journal of machine learning research},
	number = {1},
	pages = {1929--1958},
	publisher = {JMLR. org},
	title = {Dropout: a simple way to prevent neural networks from overfitting},
	volume = {15},
	year = {2014}}

@article{raissi2019physics,
	author = {Raissi, Maziar and Perdikaris, Paris and Karniadakis, George E},
	journal = {Journal of Computational physics},
	pages = {686--707},
	publisher = {Elsevier},
	title = {Physics-informed neural networks: A deep learning framework for solving forward and inverse problems involving nonlinear partial differential equations},
	volume = {378},
	year = {2019}}

@article{meng2021multi,
	author = {Meng, Xuhui and Babaee, Hessam and Karniadakis, George Em},
	journal = {Journal of Computational Physics},
	pages = {110361},
	publisher = {Elsevier},
	title = {Multi-fidelity Bayesian neural networks: Algorithms and applications},
	volume = {438},
	year = {2021}}

@article{lakshminarayanan2017simple,
	author = {Lakshminarayanan, Balaji and Pritzel, Alexander and Blundell, Charles},
	journal = {Advances in neural information processing systems},
	title = {Simple and scalable predictive uncertainty estimation using deep ensembles},
	volume = {30},
	year = {2017}}

@article{kendall2017uncertainties,
	author = {Kendall, Alex and Gal, Yarin},
	journal = {Advances in neural information processing systems},
	title = {What uncertainties do we need in bayesian deep learning for computer vision?},
	volume = {30},
	year = {2017}}

@techreport{ChanGolubLeveque1979,
	author = {Chan, Tony F. and Golub, Gene H. and LeVeque, Randall J.},
	institution = {Stanford University, Department of Computer Science},
	month = {11},
	number = {STAN-CS-79-773},
	title = {Updating formulae and a pairwise algorithm for computing sample variances},
	year = {1979}}

@article{makridakis2024deep,
	author = {Makridakis, Charalambos G and Pim, Aaron and Pryer, Tristan},
	journal = {arXiv preprint arXiv:2411.08702},
	title = {A Deep Uzawa-Lagrange Multiplier Approach for Boundary Conditions in PINNs and Deep Ritz Methods},
	year = {2024}}

@article{hullermeier2021aleatoric,
	author = {H{\"u}llermeier, Eyke and Waegeman, Willem},
	journal = {Machine learning},
	number = {3},
	pages = {457--506},
	publisher = {Springer},
	title = {Aleatoric and epistemic uncertainty in machine learning: An introduction to concepts and methods},
	volume = {110},
	year = {2021}}

@article{hikida2025multilevel,
	author = {Hikida, Yuga and Bharti, Ayush and Jeffrey, Niall and Briol, Fran{\c{c}}ois-Xavier},
	journal = {arXiv preprint arXiv:2506.06087},
	title = {Multilevel neural simulation-based inference},
	year = {2025}}

@inproceedings{heinrich2001multilevel,
	author = {Heinrich, Stefan},
	booktitle = {International Conference on Large-Scale Scientific Computing},
	organization = {Springer},
	pages = {58--67},
	title = {Multilevel monte carlo methods},
	year = {2001}}

@article{giles2015multilevel,
	author = {Giles, Michael B},
	journal = {Acta numerica},
	pages = {259--328},
	publisher = {Cambridge University Press},
	title = {Multilevel monte carlo methods},
	volume = {24},
	year = {2015}}

@article{giles2008multilevel,
	author = {Giles, Michael B},
	journal = {Operations research},
	number = {3},
	pages = {607--617},
	publisher = {INFORMS},
	title = {Multilevel monte carlo path simulation},
	volume = {56},
	year = {2008}}

@article{gerstner2021multilevel,
	author = {Gerstner, Thomas and Harrach, Bastian and Roth, Daniel and Simon, Martin},
	journal = {arXiv preprint arXiv:2102.08734},
	title = {Multilevel Monte Carlo learning},
	year = {2021}}

@article{gawlikowski2023survey,
	author = {Gawlikowski, Jakob and Tassi, Cedrique Rovile Njieutcheu and Ali, Mohsin and Lee, Jongseok and Humt, Matthias and Feng, Jianxiang and Kruspe, Anna and Triebel, Rudolph and Jung, Peter and Roscher, Ribana and others},
	journal = {Artificial Intelligence Review},
	number = {Suppl 1},
	pages = {1513--1589},
	publisher = {Springer},
	title = {A survey of uncertainty in deep neural networks},
	volume = {56},
	year = {2023}}

@inproceedings{gal2016dropout,
	author = {Gal, Yarin and Ghahramani, Zoubin},
	booktitle = {international conference on machine learning},
	organization = {PMLR},
	pages = {1050--1059},
	title = {Dropout as a bayesian approximation: Representing model uncertainty in deep learning},
	year = {2016}}

@article{fujisawa2021multilevel,
	author = {Fujisawa, Masahiro and Sato, Issei},
	journal = {Journal of Machine Learning Research},
	number = {278},
	pages = {1--44},
	title = {Multilevel monte carlo variational inference},
	volume = {22},
	year = {2021}}

@article{cliffe2011multilevel,
	author = {Cliffe, K Andrew and Giles, Mike B and Scheichl, Robert and Teckentrup, Aretha L},
	journal = {Computing and Visualization in Science},
	number = {1},
	pages = {3},
	publisher = {Springer},
	title = {Multilevel Monte Carlo methods and applications to elliptic PDEs with random coefficients},
	volume = {14},
	year = {2011}}

@article{blanchet2023dropout,
	author = {Blanchet, Jos{\'e} and Kang, Yang and Olea, Jos{\'e} Luis Montiel and Nguyen, Viet Anh and Zhang, Xuhui},
	journal = {Journal of Machine Learning Research},
	number = {180},
	pages = {1--60},
	title = {Dropout training is distributionally robust optimal},
	volume = {24},
	year = {2023}}

\end{document}